\documentclass[11pt, a4paper]{article}

\usepackage{amsmath, amssymb, amsthm}
\usepackage{mathrsfs}
\usepackage{geometry}
\usepackage{hyperref}
\usepackage{xcolor}
\usepackage[breakable]{tcolorbox}
\usepackage{enumitem}
\usepackage{booktabs}
\usepackage{tikz}
\usepackage{pifont}

\hypersetup{
    colorlinks=true,
    linkcolor=blue!70!black,
    citecolor=green!50!black,
    urlcolor=blue!60!black
}

\newtheorem{theorem}{Theorem}[section]
\newtheorem{proposition}[theorem]{Proposition}

\theoremstyle{definition}
\newtheorem{definition}[theorem]{Definition}
\newtheorem{example}[theorem]{Example}
\theoremstyle{remark}
\newtheorem{remark}[theorem]{Remark}

\newtcolorbox{intuition}{
    colback=blue!5!white,
    colframe=blue!50!black,
    title=\textbf{Intuition},
    fonttitle=\bfseries,
    breakable
}

\newtcolorbox{keypoint}{
    colback=red!5!white,
    colframe=red!60!black,
    title=\textbf{Key Insight},
    fonttitle=\bfseries,
    breakable
}

\newtcolorbox{summary}{
    colback=green!5!white,
    colframe=green!50!black,
    title=\textbf{Summary},
    fonttitle=\bfseries,
    breakable
}

\newcommand{\R}{\mathbb{R}}
\newcommand{\E}{\mathbb{E}}
\newcommand{\Prob}{\mathcal{P}}
\newcommand{\KL}{\mathrm{KL}}
\newcommand{\grad}{\mathrm{grad}}
\newcommand{\id}{\mathrm{id}}
\newcommand{\tr}{\mathrm{tr}}

\title{\LARGE\textbf{The Geometry Behind Diffusion and Flow Matching}\\[6pt]
\large Gradient Flows and Geodesics in Wasserstein Space}
\author{%
  \large\textbf{Yian Yao}\\[2pt]
  \normalsize from the mortal world
  \and
  \large\textbf{Weiwei Zhang}\\[2pt]
  \normalsize from heaven
}
\date{}

\begin{document}
\maketitle

\begin{abstract}
The space $\Prob_2(\R^d)$ of probability measures with finite second moment carries a natural geometry: the quadratic Wasserstein distance $W_2$ makes it a complete metric space and, following Otto, a (formal) Riemannian manifold whose geodesics are the optimal-transport interpolations.

\medskip
On this manifold, the gradient flow of the free energy $\mathcal{F}(\rho)=\KL(\rho\|\pi)$ is exactly the Fokker--Planck equation, and its implicit-Euler discretization is the JKO scheme. This is the geometry underlying diffusion models: the forward process descends the free energy, and each denoising step realizes one JKO step, which recovers DDPM, DDIM, NCSN/SMLD, and Energy Matching---one scheme, not separate theories.

\medskip
The same manifold supports a second variational principle. Its geodesics---the minimum-action curves of the Benamou--Brenier formula---are precisely the optimal-transport paths that Flow Matching learns. Fixing both endpoints and following the geodesic, generation becomes a deterministic ODE along a straight line---hence far fewer sampling steps.

\medskip
Placing both families of models on one manifold makes their relationship exact: diffusion follows a free-energy gradient flow, an initial-value problem; optimal-transport Flow Matching follows a Wasserstein geodesic, a boundary-value problem. The two reach the same endpoints along different paths.

\begin{center}
\fbox{\parbox{0.92\textwidth}{\raggedright
Wasserstein distance $\xrightarrow{\text{endows geometry}}$ probability space becomes a Riemannian manifold,\\
carrying \textbf{two complementary variational principles}:\\[8pt]
\textbf{(A) Diffusion} --- free-energy \emph{gradient flow}\\
(\emph{initial-value} problem: given $\rho_k$, descend $\mathcal{F}$ while staying close in $W_2$)\\[3pt]
free energy $\to$ Fokker-Planck $\xrightarrow{\text{time discretization}}$ JKO scheme\\[8pt]
\textbf{(B) Flow Matching} --- Benamou--Brenier \emph{minimum action}\\
(\emph{boundary-value} problem: given both $\rho_0$ and $\rho_1$, minimize kinetic energy)\\[3pt]
minimum action $\to$ Wasserstein geodesics $\to$ OT transport paths\\[8pt]
\rule{0.7\textwidth}{0.4pt}\\[4pt]
\textbf{Essential difference:} \emph{same endpoint $\rho_1=p_{\text{data}}$, different path.}
}}
\end{center}

\newpage
\begin{center}
{\LARGE\bfseries Outline}
\end{center}
\vspace{1em}
{\large
\begin{itemize}[leftmargin=1.8em, itemsep=8pt]
\item \textbf{Prologue} starts from the continuity and Fokker--Planck equations as they appear in diffusion models and Flow Matching, and isolates the connections those treatments leave implicit.
\item \textbf{Section~1} introduces the Wasserstein distance---the transport cost between two distributions.
\item \textbf{Section~2} derives the continuity equation---the local form of mass conservation.
\item \textbf{Section~3} establishes the Riemannian structure: the Benamou--Brenier formula identifies $W_2$ as a geodesic distance, and its geodesics as the optimal-transport paths learned by Flow Matching.
\item \textbf{Section~4} derives the Fokker--Planck equation from a stochastic differential equation and recasts it as a continuity equation.
\item \textbf{Section~5} proves the central identity: the Fokker--Planck equation is the Wasserstein gradient flow of the free energy.
\item \textbf{Section~6} discretizes the flow into the JKO scheme, proves convergence, and recovers DDPM, DDIM, NCSN/SMLD, Flow Matching, and Energy Matching from a single variational template.
\end{itemize}
}

\medskip
\textbf{Prerequisites.} We assume multivariable calculus, linear algebra, and elementary probability; measure theory, differential geometry, and convex analysis are developed in self-contained appendices.
\end{abstract}

\newpage
\tableofcontents
\newpage

\section*{Notation and Conventions}\label{sec:notation}
\addcontentsline{toc}{section}{Notation and Conventions}

The following notation is used frequently throughout this article and is collected here for easy reference.

\subsection*{Sets and Spaces}

\begin{itemize}[leftmargin=2em]
\item $\R^d$: $d$-dimensional Euclidean space.
\item \textbf{Borel sets}: Starting from all open sets in $\R^d$ and closing under countable unions, intersections, and complements, one obtains the \textbf{Borel $\sigma$-algebra}, denoted $\mathcal{B}(\R^d)$. Its elements are called Borel sets.
Intuitively, all ``reasonable'' geometric sets (open, closed, countable unions/intersections) are Borel sets; virtually every set you encounter in practice is a Borel set.

\item $\Prob_2(\R^d)$: the set of all probability measures on $\R^d$ with finite second moment, i.e., probability measures $\mu$ satisfying $\int_{\R^d}|x|^2\,d\mu(x)<\infty$.

\item $L^2(\R^d)$ (or $L^2(\rho)$): \textbf{space of square-integrable functions}. $L^2(\R^d)$ consists of all functions $f$ with $\int_{\R^d}|f(x)|^2\,dx<\infty$.
$L^2(\rho)$ is the weighted version with weight $\rho$: $\int|f|^2\rho\,dx<\infty$.
The inner product is $\langle f,g\rangle_{L^2(\rho)}=\int f(x)g(x)\rho(x)\,dx$.
This is a \textbf{Hilbert space} (a complete inner product space).

\item $C_c^\infty(\R^d)$: \textbf{space of compactly supported smooth functions}. It contains all infinitely differentiable functions that vanish outside some bounded region. Here:
\begin{itemize}
    \item \textbf{Smooth} ($C^\infty$) = derivatives of all orders exist and are continuous
    \item \textbf{Compactly supported} ($c$ for compact support) = there exists a bounded closed set $K$ such that $f(x)=0$ for all $x\notin K$. Intuitively, the function has ``finite range of activity'' and is zero far away.
\end{itemize}
\end{itemize}

\subsection*{Operators and Symbols}

\begin{itemize}[leftmargin=2em]
\item $\inf A$ (\textbf{infimum}): the greatest lower bound of a set $A$. For example, $\inf\{1/n : n\in\mathbb{N}\} = 0$.
Similar to $\min$, but allows the minimum not to be attained (limit case).

\item $\sup A$ (\textbf{supremum}): the least upper bound of a set $A$, a generalization of $\max$.

\item $\arg\min_x f(x)$: the value of $x$ at which $f(x)$ attains its minimum (not the minimum value itself, but the ``location'' where the minimum is achieved).

\item $\nabla f$: gradient (the vector of first partial derivatives of $f$ with respect to spatial variables).
\item $\nabla\cdot v$: divergence (the sum of partial derivatives of the components of a vector field $v$).
\item $\Delta f = \nabla\cdot(\nabla f)$: Laplacian.
\end{itemize}

\subsection*{Key Terminology}

\begin{itemize}[leftmargin=2em]
\item \textbf{a.e.} (almost everywhere): a property ``holds a.e.'' means it may fail at some points, but the set of failure points has measure zero (``occupies zero volume'').
For example, functions $f$ and $g$ are equal a.e.\ under Lebesgue measure means $\{x: f(x)\neq g(x)\}$ has zero volume.

\item \textbf{Test function}: a function in $C_c^\infty(\R^d)$. The name comes from the fact that in the weak formulation of PDEs, we use them to ``probe'' whether the equation holds---multiply both sides of the PDE by a test function and integrate, converting the PDE into an integral identity that must hold for all test functions. This avoids requiring smoothness of the solution itself.
\end{itemize}

\medskip
\textbf{Note:} For a detailed review of foundational concepts such as measures, absolute continuity, pushforward, and couplings, see Appendix~\ref{app:measures}.
For basic concepts in convex analysis, see Appendix~\ref{app:convex}.

\newpage

\section*{Prologue: Things You May Already Know}\label{sec:prologue}
\addcontentsline{toc}{section}{Prologue: Things You May Already Know}

For anyone who has worked on generative modeling, the continuity equation and the Fokker--Planck equation are familiar objects---one works with them constantly, often without naming them as such.

\subsection*{Scenario 1: How you encounter Fokker-Planck in Diffusion Models}
\addcontentsline{toc}{subsection}{Scenario 1: How you encounter Fokker-Planck in Diffusion Models}

Open Song et al.\ (2021) \textit{``Score-Based Generative Modeling through SDEs''}, and you will see the following setup.

\subsubsection*{Discrete version: DDPM and SMLD}

\textbf{Before} the unified framework of Song et al.\ (2021), DDPM and SMLD each defined forward processes using discrete Markov chains:

\textbf{DDPM forward process} (Ho et al., 2020): given noise schedule $\beta_1,\ldots,\beta_N$:
\[
\mathbf{x}_i = \sqrt{1-\beta_i}\,\mathbf{x}_{i-1} + \sqrt{\beta_i}\,\mathbf{z}_{i-1}, \quad \mathbf{z}_{i-1}\sim\mathcal{N}(0,I), \quad i=1,\ldots,N
\]
Each step: shrink the previous $\mathbf{x}_{i-1}$ slightly (multiply by $\sqrt{1-\beta_i}<1$), then add a bit of noise ($\sqrt{\beta_i}\,\mathbf{z}$).
After $N$ steps, $\mathbf{x}_N\approx\mathcal{N}(0,I)$. Closed-form expression:
\[
\mathbf{x}_i = \sqrt{\bar\alpha_i}\,\mathbf{x}_0 + \sqrt{1-\bar\alpha_i}\,\boldsymbol{\epsilon}, \quad \boldsymbol{\epsilon}\sim\mathcal{N}(0,I), \quad \bar\alpha_i := \prod_{j=1}^i(1-\beta_j)
\]

\textbf{SMLD forward process} (Song \& Ermon, 2019): given noise scales $\sigma_1<\sigma_2<\cdots<\sigma_N$:
\[
\mathbf{x}_i = \mathbf{x}_{i-1} + \sqrt{\sigma_i^2 - \sigma_{i-1}^2}\,\mathbf{z}_{i-1}, \quad i=1,\ldots,N
\]
Each step only adds noise without rescaling. The variance increases progressively to $\sigma_N^2$.

\subsubsection*{Continuous version: when $N\to\infty$}

The key insight of Song et al.\ (2021): as the number of steps $N\to\infty$, the discrete chain becomes a continuous SDE.

\textbf{DDPM $\to$ VP-SDE}: let $\beta_i = \beta(i/N)/N$, take $N\to\infty$, the discrete update
\[
\mathbf{x}_i = \sqrt{1-\beta_i}\,\mathbf{x}_{i-1} + \sqrt{\beta_i}\,\mathbf{z}_{i-1}
\approx \mathbf{x}_{i-1} - \frac{\beta_i}{2}\mathbf{x}_{i-1} + \sqrt{\beta_i}\,\mathbf{z}_{i-1}
\]
becomes an SDE ($dt = 1/N$, $\mathbf{z}\sqrt{dt} = d\mathbf{w}$):
\begin{equation}\label{eq:vpsde_prologue}
d\mathbf{x} = -\frac{\beta(t)}{2}\mathbf{x}\,dt + \sqrt{\beta(t)}\,d\mathbf{w} \qquad\text{(VP-SDE)}
\end{equation}
This is an OU process---the spring force $-\frac{\beta}{2}\mathbf{x}$ pulls particles back to the origin, while noise $\sqrt{\beta}\,d\mathbf{w}$ disperses them.

\textbf{SMLD $\to$ VE-SDE}: similarly, $\mathbf{x}_i = \mathbf{x}_{i-1} + \sqrt{\sigma_i^2-\sigma_{i-1}^2}\,\mathbf{z}_{i-1}$ in the continuous limit becomes:
\begin{equation}\label{eq:vesde_prologue}
d\mathbf{x} = \sqrt{\frac{d[\sigma^2(t)]}{dt}}\,d\mathbf{w} \qquad\text{(VE-SDE)}
\end{equation}
Pure diffusion, no restoring force. The variance grows monotonically (``explodes''), ultimately yielding a Gaussian with very large variance.

\medskip
Unified notation: $d\mathbf{x} = \mathbf{f}(\mathbf{x},t)\,dt + g(t)\,d\mathbf{w}$, where
\begin{center}
\renewcommand{\arraystretch}{1.3}
\begin{tabular}{@{}lll@{}}
\toprule
& \textbf{VP-SDE (DDPM)} & \textbf{VE-SDE (SMLD)} \\
\midrule
drift $\mathbf{f}(\mathbf{x},t)$ & $-\frac{\beta(t)}{2}\mathbf{x}$ & $0$ \\
diffusion $g(t)$ & $\sqrt{\beta(t)}$ & $\sqrt{\dot\sigma^2(t)}$ \\
Steady state & $\mathcal{N}(0,I)$ & $\mathcal{N}(0,\sigma_{\max}^2 I)$ \\
Potential $V(x)$ & $\frac{1}{2}|x|^2$ & $0$ (no potential) \\
\bottomrule
\end{tabular}
\end{center}

\subsubsection*{From SDE to the Fokker-Planck equation}

The probability density $\rho_t(x)$ induced by these SDEs satisfies the \textbf{Fokker-Planck equation} (derivation in Section~4):
\begin{equation}\label{eq:fp_prologue}
\partial_t\rho_t = -\nabla\cdot(\mathbf{f}\,\rho_t) + \frac{g^2}{2}\Delta\rho_t
\end{equation}
Substituting $\mathbf{f} = -\frac{\beta}{2}\mathbf{x}$, $g^2 = \beta$ for VP-SDE:
\[
\partial_t\rho_t = \frac{\beta}{2}\nabla\cdot(\mathbf{x}\,\rho_t) + \frac{\beta}{2}\Delta\rho_t
\]

\subsubsection*{How Fokker-Planck is used in the paper: deriving the Reverse-time SDE}

In the Diffusion Model paper, the \textbf{sole use} of Fokker-Planck is to derive the \textbf{reverse-time SDE}. Here is the complete derivation.

\textbf{Goal}: the forward process $d\mathbf{x} = \mathbf{f}\,dt + g\,d\mathbf{w}$ runs from $t=0$ (data) to $t=T$ (noise).
We want to find an SDE running from $T$ to $0$ whose marginal distributions are exactly $\rho_t$ in reversed time.

\textbf{Suppose} the reverse SDE (let $\bar t = T-t$ be the reverse time) takes the form:
\[
d\mathbf{x} = \bar{\mathbf{f}}\,d\bar t + g\,d\bar{\mathbf{w}}
\]
Its marginal distribution $\bar\rho_{\bar t} := \rho_{T-\bar t}$ should satisfy the \textbf{reverse Fokker-Planck}:
\begin{equation}\label{eq:reverse_fp}
\partial_{\bar t}\bar\rho_{\bar t} = -\nabla\cdot(\bar{\mathbf{f}}\,\bar\rho_{\bar t}) + \frac{g^2}{2}\Delta\bar\rho_{\bar t}
\end{equation}

\textbf{Key step}: from $\bar\rho_{\bar t} = \rho_{T-\bar t}$, we get $\partial_{\bar t}\bar\rho_{\bar t} = -\partial_t\rho_t$.

Substituting the forward Fokker-Planck~\eqref{eq:fp_prologue}:
\[
-\partial_t\rho_t = -\left[-\nabla\cdot(\mathbf{f}\rho_t) + \tfrac{g^2}{2}\Delta\rho_t\right]
= \nabla\cdot(\mathbf{f}\rho_t) - \tfrac{g^2}{2}\Delta\rho_t
\]

Setting this equal to the right-hand side of \eqref{eq:reverse_fp} (both expressed in terms of $\rho_t$):
\[
\nabla\cdot(\mathbf{f}\rho_t) - \tfrac{g^2}{2}\Delta\rho_t = -\nabla\cdot(\bar{\mathbf{f}}\,\rho_t) + \tfrac{g^2}{2}\Delta\rho_t
\]

Rearranging:
\[
\nabla\cdot(\mathbf{f}\rho_t) = -\nabla\cdot(\bar{\mathbf{f}}\,\rho_t) + g^2\Delta\rho_t
\]

Using the identity $\Delta\rho = \nabla\cdot(\nabla\rho) = \nabla\cdot(\rho\,\nabla\log\rho)$:
\[
\nabla\cdot(\mathbf{f}\rho_t) = \nabla\cdot\!\left[-\bar{\mathbf{f}}\,\rho_t + g^2\rho_t\,\nabla\log\rho_t\right]
\]

Comparing inside the $\nabla\cdot$, solving for $\bar{\mathbf{f}}$:
\[
\boxed{\bar{\mathbf{f}} = -\mathbf{f} + g^2\,\nabla\log\rho_t}
\]

Therefore, the reverse SDE (returning to the original time parameter $t$, running backward from $T$ to $0$) is:
\begin{equation}\label{eq:reverse_sde_prologue}
d\mathbf{x} = \left[\mathbf{f}(\mathbf{x},t) - g(t)^2\,\nabla\log\rho_t(\mathbf{x})\right]dt + g(t)\,d\bar{\mathbf{w}}
\end{equation}

This is the Anderson (1982) theorem. Note that $\nabla\log\rho_t$---the score---is the only unknown quantity that needs to be learned.

\subsubsection*{Discrete counterpart: DDPM denoising steps}

Apply \textbf{Euler-Maruyama discretization} to the reverse SDE~\eqref{eq:reverse_sde_prologue}.
Recall: for a general SDE $dX = a(X,t)\,dt + b(t)\,dW$, the Euler-Maruyama method simply ``replaces $dt$ with $\Delta t$ and $dW$ with $\sqrt{\Delta t}\,\mathbf{z}$'' ($\mathbf{z}\sim\mathcal{N}(0,I)$):
\[
X_{n+1} = X_n + a(X_n, t_n)\,\Delta t + b(t_n)\,\sqrt{\Delta t}\,\mathbf{z}_n
\]
This is the simplest numerical scheme for SDEs, analogous to the Euler method for ODEs, but with the additional random term $\sqrt{\Delta t}\,\mathbf{z}$.

For the reverse SDE ($a = \mathbf{f} - g^2\nabla\log\rho_t$, $b=g$), taking $\Delta t = 1/N$, we obtain the DDPM denoising formula:
\[
\mathbf{x}_{i-1} = \frac{1}{\sqrt{1-\beta_i}}\left(\mathbf{x}_i + \beta_i\,s_\theta(\mathbf{x}_i,i)\right) + \sqrt{\beta_i}\,\mathbf{z}, \quad \mathbf{z}\sim\mathcal{N}(0,I)
\]
where $s_\theta \approx \nabla\log\rho_i$ is the score learned by the neural network. Training objective (denoising score matching; Hyv\"arinen, 2005; Vincent, 2011):
\[
\mathcal{L} = \mathbb{E}_{t,\mathbf{x}_0,\boldsymbol\epsilon}\left[\|s_\theta(\sqrt{\bar\alpha_t}\,\mathbf{x}_0 + \sqrt{1-\bar\alpha_t}\,\boldsymbol\epsilon,\; t) - \left(-\frac{\boldsymbol\epsilon}{\sqrt{1-\bar\alpha_t}}\right)\|^2\right]
\]
Note $\nabla_{\mathbf{x}}\log p_{t|0}(\mathbf{x}|\mathbf{x}_0) = -\frac{\mathbf{x}-\sqrt{\bar\alpha_t}\mathbf{x}_0}{1-\bar\alpha_t} = -\frac{\boldsymbol\epsilon}{\sqrt{1-\bar\alpha_t}}$,
so score matching is essentially training the network to predict the direction of the noise $\boldsymbol\epsilon$.

\subsubsection*{Connection to this article}

In the diffusion paper, the derivation chain is:
\[
\text{SDE} \xrightarrow{\text{Fokker-Planck}} \text{density evolution}\rho_t \xrightarrow{\text{time reversal}} \text{reverse SDE with score} \xrightarrow{\text{learn score}} \text{generation}
\]
Here Fokker--Planck appears only as an intermediate step in deriving the score from the SDE, and is set aside once the derivation is complete.

\medskip
This article takes the equation more seriously. Rewriting Fokker-Planck in the form of a continuity equation $\partial_t\rho = -\nabla\cdot(\rho\,v)$ (using $\Delta\rho = \nabla\cdot(\rho\nabla\log\rho)$):
\[
\partial_t\rho_t = -\nabla\cdot\!\left(\rho_t\underbrace{\left[\mathbf{f} - \tfrac{g^2}{2}\nabla\log\rho_t\right]}_{=\,v_t}\right)
\]
so the velocity field is $v_t = \mathbf{f} - \frac{g^2}{2}\nabla\log\rho_t$.
For VP-SDE ($\mathbf{f}=-\frac{\beta}{2}\mathbf{x}$, $g^2=\beta$), this gives $v_t = -\frac{\beta}{2}\mathbf{x} - \frac{\beta}{2}\nabla\log\rho_t = -\frac{\beta}{2}\nabla\left(\frac{|x|^2}{2} + \log\rho_t\right)$.

Comparing with the Wasserstein gradient $\grad_W\mathcal{F} = \nabla V + \nabla\log\rho$ derived in the main text (taking $V=\frac{1}{2}|x|^2$), we see $v_t = -\frac{\beta}{2}\,\grad_W\mathcal{F}$---the velocity field is the \emph{negative} gradient, up to the positive time-scaling factor $\beta/2$.

\textbf{Section~5 will prove}: $v_t$ is precisely the \textbf{negative gradient} of $\KL(\rho_t\|\pi)$ ($\pi=\mathcal{N}(0,I)\propto e^{-|x|^2/2}$) under the Wasserstein metric (up to the positive factor $\beta/2$, which is merely a time reparametrization).
That is, the forward process of VP-SDE $=$ the \textbf{gradient flow} of free energy in probability space.

\textbf{Section~6 will prove}: each DDPM denoising step $=$ one step of the JKO scheme---implicit gradient descent in probability space.

\medskip
In short, the derivation above uses Fokker--Planck only to extract the score. The equation carries considerably more geometric structure, and recovering it is the purpose of this article.

\subsection*{Scenario 2: How you encounter the continuity equation in Flow Matching}
\addcontentsline{toc}{subsection}{Scenario 2: How you encounter the continuity equation in Flow Matching}

\begin{remark}[Time-convention warning]
In Scenario~1 (diffusion), time runs from \textbf{data} ($t=0$) to \textbf{noise} ($t=T$). In Flow Matching below, the convention is \emph{reversed}: $t=0$ is \textbf{noise} and $t=1$ is \textbf{data}. Both conventions follow their respective original papers; keep this flip in mind when comparing the two scenarios (e.g., the ``endpoint'' $\delta_{x_1}$ here is data, whereas the endpoint of the forward diffusion is noise).
\end{remark}

Open Lipman et al.\ (2023) \textit{``Flow Matching for Generative Modeling''}, and the story is entirely different---no SDEs, no noise, everything is deterministic.

The \textbf{core objects} are a probability density path $p_t$ and a velocity field $v_t$, satisfying:
\begin{equation}\label{eq:ce_prologue}
\partial_t p_t + \nabla\cdot(p_t\,v_t) = 0 \qquad\text{(continuity equation)}
\end{equation}

\textbf{Where does this equation come from?} If particles follow the ODE $\frac{dX_t}{dt} = v_t(X_t)$ with initial condition $X_0\sim p_0$,
what equation does the density $p_t$ of $X_t$ satisfy? For any test function $\varphi$:
\[
\frac{d}{dt}\int\varphi\,p_t\,dx = \frac{d}{dt}\E[\varphi(X_t)] = \E[\nabla\varphi(X_t)\cdot v_t(X_t)]
= \int\nabla\varphi\cdot v_t\,p_t\,dx
\]
Integrating the right side by parts: $= -\int\varphi\,\nabla\cdot(v_t\,p_t)\,dx$. Comparing with the left side $= \int\varphi\,\partial_t p_t\,dx$, we get $\partial_t p_t = -\nabla\cdot(p_t v_t)$.

\textbf{Intuitive meaning}: the change in density $=$ the negative of the net outflux. If $v_t$ ``carries mass away'' from $x$ ($\nabla\cdot(p_t v_t) > 0$), then the density at $x$ decreases.

\medskip
The \textbf{training objective} is to fit a known target velocity field $u_t(x|x_1)$ with a neural network $v_\theta(x,t)$:
\[
\mathcal{L}_{\text{CFM}}(\theta) = \mathbb{E}_{t,\,x_1\sim q,\,x\sim p_t(x|x_1)}\|v_\theta(x,t) - u_t(x|x_1)\|^2
\]

For the OT path (the simplest choice), the conditional probability path and conditional velocity field are:
\[
p_t(x|x_1) = \mathcal{N}\bigl(x \;\big|\; t\,x_1,\;(1-t)^2 I\bigr), \qquad u_t(x|x_1) = \frac{x_1 - x}{1-t}
\]
This is a ``straight line'' from $\mathcal{N}(0,I)$ to $\delta_{x_1}$. Each particle travels at constant velocity $\frac{x_1-x_0}{1}$ from its starting point $x_0$ to its endpoint $x_1$.

For generation, one only needs to solve an ODE: $\frac{dx}{dt} = v_\theta(x,t)$, integrating from $x_0\sim\mathcal{N}(0,I)$ to $t=1$.

\textbf{Summary: in the context of Flow Matching, the continuity equation is the constraint linking ``velocity field $\to$ density path.''}
The paper uses it to prove: if $v_\theta$ accurately learns the conditional velocity field, then the marginal density path $p_t$ it generates is the desired one.
One question is left unaddressed, however: how does this equation relate to Fokker--Planck? Are the two different expressions of the same object?

\subsection*{Scenario 3: Probability Flow ODE---where the two equations meet}
\addcontentsline{toc}{subsection}{Scenario 3: Probability Flow ODE---where the two equations meet}

A central observation of Song et al.\ (2021) connects the two pictures. They proved:

\begin{keypoint}
For \textbf{any} forward SDE $d\mathbf{x} = \mathbf{f}\,dt + g\,d\mathbf{w}$, there exists a \textbf{deterministic ODE} (probability flow ODE):
\begin{equation}\label{eq:prob_flow_prologue}
\frac{d\mathbf{x}}{dt} = \underbrace{\mathbf{f}(\mathbf{x},t)}_{\text{drift}} - \frac{1}{2}g(t)^2\underbrace{\nabla\log\rho_t(\mathbf{x})}_{\text{score}}
\end{equation}
such that the marginal distribution $\rho_t$ of particles is \textbf{exactly the same as that of the SDE}.
\end{keypoint}

Rather than state the result, let us derive it.

\medskip
\textbf{Derivation: reading off the probability flow ODE from Fokker--Planck}

Starting point: we know that $\rho_t$ satisfies the Fokker-Planck equation
\[
\partial_t\rho_t = -\nabla\cdot(\mathbf{f}\,\rho_t) + \frac{g^2}{2}\Delta\rho_t
\]
Now ask: \textbf{can we find a velocity field $v_t$} such that $\rho_t$ also satisfies the continuity equation $\partial_t\rho_t = -\nabla\cdot(\rho_t\,v_t)$?

If so, there exists a deterministic ODE $\dot x = v_t(x)$ that produces the same density evolution---no noise needed.

Key trick: ``disguise'' the diffusion term of Fokker-Planck as a transport term. Note the identity:
\[
\Delta\rho = \nabla\cdot(\nabla\rho) = \nabla\cdot\!\left(\rho\cdot\frac{\nabla\rho}{\rho}\right) = \nabla\cdot(\rho\,\nabla\log\rho)
\]
Substituting into Fokker-Planck:
\begin{align*}
\partial_t\rho_t &= -\nabla\cdot(\mathbf{f}\,\rho_t) + \frac{g^2}{2}\nabla\cdot(\rho_t\,\nabla\log\rho_t)\\
&= -\nabla\cdot\!\left(\mathbf{f}\,\rho_t - \frac{g^2}{2}\rho_t\,\nabla\log\rho_t\right)\\
&= -\nabla\cdot\!\left(\rho_t\underbrace{\left[\mathbf{f} - \frac{g^2}{2}\nabla\log\rho_t\right]}_{=:\,v_t}\right)
\end{align*}

The velocity field can be read off directly: $v_t = \mathbf{f} - \frac{g^2}{2}\nabla\log\rho_t$ ensures $\partial_t\rho_t + \nabla\cdot(\rho_t v_t) = 0$.
The corresponding ODE is \eqref{eq:prob_flow_prologue}.

\medskip
\textbf{Verification (from the other direction):} substitute $v_t = \mathbf{f} - \frac{g^2}{2}\nabla\log\rho_t$ into the continuity equation and expand:
\begin{align*}
-\nabla\cdot(\rho_t v_t) &= -\nabla\cdot\!\left(\rho_t\mathbf{f} - \frac{g^2}{2}\rho_t\nabla\log\rho_t\right)\\
&= -\nabla\cdot(\rho_t\mathbf{f}) + \frac{g^2}{2}\nabla\cdot\!\left(\rho_t\cdot\frac{\nabla\rho_t}{\rho_t}\right)\\
&= -\nabla\cdot(\rho_t\mathbf{f}) + \frac{g^2}{2}\nabla\cdot(\nabla\rho_t)\\
&= -\nabla\cdot(\rho_t\mathbf{f}) + \frac{g^2}{2}\Delta\rho_t \quad\checkmark
\end{align*}
This is exactly Fokker-Planck. The two equations describe \textbf{the same density trajectory} $\rho_t$.

The point deserves emphasis:
\begin{itemize}[leftmargin=2em]
\item Fokker-Planck is an equation ``with diffusion''---it has $\Delta\rho$ (second-order term)
\item The continuity equation is a ``pure transport'' equation---only $\nabla\cdot(\rho v)$ (first-order term)
\item Yet they describe \textbf{the same density trajectory} $\rho_t$
\end{itemize}

What distinguishes them is the velocity field.
The effective velocity $v = \mathbf{f} - \frac{g^2}{2}\nabla\log\rho$ of the transport form absorbs the diffusion into the velocity: the term $-\frac{g^2}{2}\nabla\log\rho = -\frac{g^2}{2}\frac{\nabla\rho}{\rho}$ converts $\frac{g^2}{2}\Delta\rho$ exactly into $-\nabla\cdot(\rho\cdot(-\frac{g^2}{2}\nabla\log\rho))$.

\textbf{In Flow Matching language:} the probability flow ODE of a diffusion model is itself a flow, with velocity field $v_t = f - \frac{g^2}{2}\,\text{score}$. Diffusion and Flow Matching are thus two perspectives on a single object.

\medskip
\textbf{A note on EDM.} Karras et al.\ (2022), in \textit{``Elucidating the Design Space of Diffusion-Based Generative Models,''} push this ODE-centric view furthest. They write the forward process simply as the data convolved with Gaussian noise of scale $\sigma$---marginals $p(x;\sigma)$---and center generation on the probability flow ODE
\[
\frac{dx}{dt} = -\dot\sigma(t)\,\sigma(t)\,\nabla_x\log p\bigl(x;\sigma(t)\bigr).
\]
VP- and VE-SDE then become particular choices among a small set of \emph{orthogonal} design axes---the noise schedule $\sigma(t)$, an overall scaling $s(t)$, the network preconditioning, the training loss weighting, and the ODE solver (they adopt a second-order Heun integrator)---each of which can be tuned in isolation. In the language of this article, these axes are reparametrizations of \emph{time and scale} along one and the same density trajectory: the underlying geometric object---the Fokker--Planck flow, equivalently the Wasserstein gradient flow of the free energy---is invariant under them. EDM is thus practical evidence for the article's thesis that what matters is the trajectory and its score, not the particular SDE used to present it.

\medskip
\textbf{A note on energy-based formulations.} A complementary line of recent work makes the \emph{equilibrium} side of this geometry explicit. \emph{Energy Matching} (Balcerak et al., 2025) replaces the time-dependent score or velocity network by a single \emph{time-independent} scalar potential $V_\theta$: far from the data it transports samples by $-\nabla V_\theta$ (an optimal-transport flow), and near the data it relaxes to the Boltzmann equilibrium $\rho\propto e^{-V_\theta/\varepsilon}$. \emph{Equilibrium Matching} (Wang et al., 2025) likewise discards time conditioning, learning a single time-invariant gradient field of an implicit energy landscape whose stationary points are the data; generation then becomes gradient descent on that landscape, with step size and compute chosen at inference. In the language of this article both are direct parametrizations of the \emph{free-energy / equilibrium} structure: the equilibrium $\rho\propto e^{-V}$ is the minimizer of the free energy, and the relaxation toward it is the Wasserstein gradient flow (Section~6 derives Energy Matching from the JKO scheme). Where EDM emphasizes the transport, probability-flow-ODE face of the geometry, these emphasize its gradient-flow, Gibbs face---two faces of one object.

\subsection*{How deep do these connections run?}
\addcontentsline{toc}{subsection}{How deep do these connections run?}

Most expositions stop here: Fokker--Planck and the continuity equation are ``equivalent,'' and score and velocity differ by a drift. A few further questions, however, repay attention.

\begin{enumerate}[leftmargin=2em]
\item \textbf{Why can Fokker-Planck be written in the form of a continuity equation?}

This is not merely an algebraic coincidence. $\partial_t\rho + \nabla\cdot(\rho v) = 0$ means the evolution of density can be understood as
probability mass ``flowing'' under a \textbf{velocity field}. But in what ``space'' does this flow take place?
The answer: it takes place in the space of probability distributions $\Prob_2(\R^d)$, and the ``velocity field'' $v$ is precisely a \textbf{tangent vector} in this space.

\item \textbf{Why does VP-SDE ultimately converge to $\mathcal{N}(0,I)$, and not some other distribution?}

For VP-SDE, $\mathbf{f} = -\frac{\beta}{2}\mathbf{x}$ implies potential $V(x) = \frac{1}{2}|x|^2$,
so the Gibbs distribution is $\pi \propto e^{-V} = e^{-|x|^2/2} = \mathcal{N}(0,I)$.
The Fokker-Planck equation drives $\rho_t$ to converge to $\pi$---but this is no accident: it is \textbf{minimizing} $\KL(\rho_t\|\pi)$.

\textbf{Proof:} compute $\frac{d}{dt}\KL(\rho_t\|\pi)$. Let $\mathcal{F}(\rho) = \KL(\rho\|\pi) = \int\rho\log\frac{\rho}{\pi}\,dx$.
\begin{align*}
\frac{d}{dt}\mathcal{F}(\rho_t) &= \int\left(\log\frac{\rho_t}{\pi}+1\right)\partial_t\rho_t\,dx
= \int\log\frac{\rho_t}{\pi}\,\partial_t\rho_t\,dx \quad\text{(since $\int\partial_t\rho_t = 0$)}
\end{align*}
Substituting $\partial_t\rho_t = \nabla\cdot(\rho_t\nabla\log\frac{\rho_t}{\pi})$ (Fokker-Planck rewritten with velocity field $v = -\nabla\log\frac{\rho}{\pi}$):
\begin{align*}
&= \int\log\frac{\rho_t}{\pi}\cdot\nabla\cdot\!\left(\rho_t\nabla\log\frac{\rho_t}{\pi}\right)dx\\
&= -\int\rho_t\left|\nabla\log\frac{\rho_t}{\pi}\right|^2 dx \quad\text{(integration by parts)}\\
&\leq 0
\end{align*}
Equality holds if and only if $\nabla\log\frac{\rho_t}{\pi} = 0$, i.e., $\rho_t = \pi$. Therefore $\KL(\rho_t\|\pi)$ is \textbf{strictly monotonically decreasing} until $\rho_t = \pi$.

This is the hallmark of ``gradient descent'': the objective function decreases monotonically along the trajectory, at a rate proportional to the ``squared norm of the gradient'' $\int\rho|\nabla\log\frac{\rho}{\pi}|^2$ (this is called the relative Fisher information).

\item \textbf{Why is the OT path in Flow Matching ``optimal''?}

Lipman et al.\ found that models trained with the conditional OT path $x_t = (1-t)x_0 + t\,x_1$ (straight lines)
perform better and sample faster than models trained with diffusion paths (curved, roundabout).

Why should straight lines be preferable? A short computation settles it. For the conditional OT path, particles move from $x_0$ to $x_1$ at constant speed:
\[
X_t = (1-t)\,x_0 + t\,x_1, \qquad \dot X_t = x_1 - x_0 \quad\text{(constant velocity)}
\]
``Kinetic energy integral'' (measuring how ``costly'' the path is):
\[
\int_0^1|x_1-x_0|^2\,dt = |x_1-x_0|^2
\]
Averaging over all particles: $\E[|x_1-x_0|^2] = W_2^2(p_0,p_1)$ (for the optimal pairing).

For diffusion paths (such as VP-SDE), particle trajectories are curved spirals (first diffusing outward, then pulled back).
The path length for the same pair $(x_0,x_1)$ is longer, and the kinetic energy integral is larger:
\[
\int_0^1|\dot X_t^{\text{diffusion}}|^2\,dt > |x_1-x_0|^2
\]

The Benamou-Brenier formula states:
\[
W_2^2(\rho_0,\rho_1) = \inf_{\substack{(\rho_t,v_t):\\\partial_t\rho_t+\nabla\cdot(\rho_t v_t)=0\\\rho|_{t=0}=\rho_0,\;\rho|_{t=1}=\rho_1}}\int_0^1\!\!\int|v_t(x)|^2\rho_t(x)\,dx\,dt
\]
The right side is the ``total kinetic energy'' over all paths satisfying the continuity equation---the OT path (straight-line motion) achieves the minimum $W_2^2$.
Diffusion paths are not minimal $\to$ velocity fields are larger and more complex $\to$ harder for networks to learn.

\textbf{Conclusion: OT paths are geodesics (``straight lines'') in Wasserstein space, while diffusion paths ``take detours.''} This is the mathematical root of why Flow Matching with OT paths is more efficient than Diffusion.

\item \textbf{What is each DDPM denoising step ``optimizing''?}

DDPM starts from $x_T\sim\mathcal{N}(0,I)$ and progressively denoises to obtain $x_0\sim p_{\text{data}}$. Each step takes the form:
\[
x_{t-1} = \frac{1}{\sqrt{\alpha_t}}\left(x_t - \frac{1-\alpha_t}{\sqrt{1-\bar\alpha_t}}\epsilon_\theta(x_t,t)\right) + \sigma_t\,z
\]
This formula appears to simply ``subtract a bit of predicted noise.'' But what is its deeper meaning?

Recall our gradient flow: the Fokker-Planck velocity field is $v = -\nabla\log\rho - \nabla V$.
The DDPM prediction $\epsilon_\theta$ satisfies $\epsilon_\theta \approx -\sqrt{1-\bar\alpha_t}\,\nabla\log\rho_t$ (i.e., the negatively normalized score).
Substituting ``subtracting $\epsilon$'' amounts to ``taking a step in the score direction''---a discretized gradient flow.

More precisely: in probability space, this step is equivalent to solving the \textbf{JKO scheme} (implicit Euler):
\[
\rho_{k+1} = \arg\min_\rho\left\{\underbrace{\mathcal{F}(\rho)}_{\text{free energy (wants to decrease)}} + \underbrace{\frac{1}{2\tau}W_2^2(\rho,\rho_k)}_{\text{don't stray too far from the previous step}}\right\}
\]
where $\mathcal{F}(\rho) = \int\rho\log\rho + \int V\rho$ is the free energy.

The \textbf{JKO optimality condition} (derived in detail in Section~6) states that the minimizer satisfies:
\[
\nabla\log\rho_{k+1} + \nabla V + \frac{x - T_k(x)}{\tau} = 0
\]
where $T_k$ is the optimal transport map. Rearranging gives $\frac{x-T_k(x)}{\tau} = -\nabla\log\rho_{k+1} - \nabla V$, whose right-hand side is precisely the negative velocity field of Fokker--Planck.

So each DDPM denoising step is one step of implicit gradient descent in probability space---which accounts for its stability (implicit methods are unconditionally stable) and its convergence to the data distribution.

\item \textbf{Why can Energy Matching (2025) use a single scalar field for both transport and equilibrium?}

Balcerak et al.\ directly use the JKO scheme as a generative framework:
\[
\rho_{t+\Delta t} = \arg\min_\rho\left\{\frac{W_2^2(\rho,\rho_t)}{2\Delta t} + \int V_\theta\,d\rho + \varepsilon(t)\int\rho\log\rho\,dx\right\}
\]
The first-order optimality condition yields $\dot x = -\nabla V_\theta$ (transport) when $\varepsilon=0$,
and $\rho \propto e^{-V_\theta/\varepsilon}$ (Boltzmann distribution) at equilibrium.
A single scalar field $V_\theta$ simultaneously encodes the ``path'' and the ``endpoint''---because the JKO scheme naturally unifies both.
\end{enumerate}

\subsection*{This article's mission}
\addcontentsline{toc}{subsection}{This article's mission}

The answers to all the questions above point to the same mathematical structure:

\begin{center}
\fbox{\parbox{0.88\textwidth}{\centering
The space of probability distributions $\Prob_2(\R^d)$ forms a (formal) \textbf{Riemannian manifold} under the \textbf{Wasserstein distance}.\\[4pt]
The \textbf{continuity equation} is the ``equation of motion'' on this manifold (how tangent vectors induce density changes),\\[2pt]
The \textbf{Fokker-Planck equation} is the \textbf{gradient flow} of the free energy $\KL(\rho\|\pi)$ on this manifold,\\[2pt]
The \textbf{JKO scheme} is the implicit Euler discretization of this gradient flow.
}}
\end{center}

\medskip
Translated into language you are familiar with:

\begin{center}
\renewcommand{\arraystretch}{1.4}
\begin{tabular}{@{}p{5.8cm}p{7.5cm}@{}}
\toprule
\textbf{What you already know (ML language)} & \textbf{What this article will tell you (math language)} \\
\midrule
Continuity equation in Flow Matching & \textbf{Equation of motion} on the Wasserstein manifold \\
$v_t$ is a velocity field & $v_t$ is a \textbf{tangent vector} in probability space \\
$\int|v_t|^2 p_t\,dx$ is part of the training loss & This is the \textbf{Riemannian metric} of the Wasserstein manifold \\
OT paths are straighter than diffusion paths & OT paths are \textbf{geodesics} in Wasserstein space \\
Fokker-Planck equation & \textbf{Wasserstein gradient flow} of the free energy $\KL(\rho\|\pi)$ \\
Score $\nabla\log\rho_t$ & \textbf{Wasserstein gradient} of negative entropy \\
Each DDPM denoising step & One step of \textbf{JKO (implicit Euler)} in probability space \\
Energy Matching & Direct implementation of the JKO scheme \\
\bottomrule
\end{tabular}
\end{center}

\medskip

These are not analogies---they are precise mathematical theorems. The following six sections will build up this structure brick by brick:
\begin{itemize}[leftmargin=2em, itemsep=1pt]
\item Section~1 defines the Wasserstein distance (``the metric of the space'')
\item Section~2 derives the continuity equation (``the equation of motion'')
\item Section~3 proves $W_2 = $ geodesic distance (Benamou-Brenier, ``why OT paths are optimal'')
\item Section~4 derives Fokker-Planck (``from SDE to PDE'')
\item Section~5 proves Fokker-Planck $=$ gradient flow (``what is being optimized'')
\item Section~6 derives the JKO scheme (``how to discretize,'' ``each denoising step $=$ one optimization step'')
\end{itemize}

\medskip
\textbf{Prerequisites:} If you have read either Song et al.\ (2021) or Lipman et al.\ (2023), you already have sufficient background.
We assume familiarity with multivariable calculus, linear algebra, and basic probability theory; prerequisites in measure theory, differential geometry, etc.\ are developed from scratch in the appendices.

\subsection*{Related perspectives}
\addcontentsline{toc}{subsection}{Related perspectives}

This article gives a unified, end-to-end account of diffusion and flow matching from a single geometric object: the Wasserstein manifold of probability measures. On it the two are not the same thing but complementary principles---diffusion is the gradient flow of the free energy (an initial-value problem), flow matching the geodesic, i.e.\ the minimal-action path (a boundary-value problem). Where most treatments proceed through stochastic processes and score functions, here one geometric idea is carried, without gaps, from the definition of the Wasserstein metric to the algorithms used in practice. Several lines of work reach this same picture from other directions; we place the article among them below.

\textbf{Unifying frameworks for flows and diffusions.} The cleanest statement that diffusion and flow-based models are facets of one object is the \emph{stochastic interpolants} framework of Albergo, Boffi \& Vanden-Eijnden (2025): a single interpolant between two densities induces both a transport (continuity) equation and a family of forward/backward Fokker--Planck equations with tunable noise, hence both deterministic (ODE) and stochastic (SDE) generators. Lipman et al.\ (2024), in the \emph{Flow Matching Guide and Code}, give a comprehensive exposition presenting diffusion paths as special cases of Flow Matching. Both unify the two through the SDE/ODE and velocity-field language; this article instead routes the unification through Wasserstein geometry---gradient flow versus geodesic.

\textbf{Diffusion as a Wasserstein gradient flow.} Closest to the central identity of Section~5 is the work of Vuong, McCann, Santos \& Lin (2025), who argue that diffusion training is better read as flow matching to the velocity field of a Wasserstein gradient flow than as score learning, and give numerical evidence that the learned field is generally not conservative (hence not a genuine score).

\textbf{JKO as a generative algorithm.} A line of work builds generative models directly on the JKO scheme of Section~6: Mokrov et al.\ (2021) scale Wasserstein gradient flows using input-convex networks, and Xu, Cheng \& Xie (2023, JKO-iFlow) stack residual blocks that each realize one JKO step. Energy Matching (Balcerak et al., 2025) and Equilibrium Matching (Wang et al., 2025) belong to the same family.

\textbf{A geometric caveat.} The two principles should not be over-identified. Lavenant \& Santambrogio (2022) prove that the flow map of the Fokker--Planck equation is \emph{not} the optimal-transport map: the diffusion trajectory is in general not the $W_2$ geodesic between its endpoints. This is precisely the distinction drawn here---diffusion descends the free energy (an initial-value gradient flow, along a curved path), whereas OT Flow Matching follows the geodesic (a boundary-value problem)---and it is why the two reach the same endpoints by different paths.

\textbf{Complementary expositions.} For tutorials that unify Flow Matching and diffusion through the SDE/score lens rather than optimal transport, see the lecture notes of Holderrieth \& Erives (2025, MIT 6.S184).

\newpage

\section{Wasserstein Distance: The Cost of Moving Dirt}

\subsection{Why do we need a new distance?}

Suppose you have two piles of sand (probability distributions) and want to measure ``how different'' they are.

Consider two probability distributions $\mu$ and $\nu$ on $\R^d$. Classical divergences like KL or total variation have a fundamental limitation: they ignore the \emph{geometry} of the underlying space. Let us first recall their definitions:

\begin{definition}[KL divergence]
Let $\mu$ and $\nu$ have density functions $p(x)$ and $q(x)$, respectively. The \textbf{KL divergence} (Kullback--Leibler divergence; Kullback \& Leibler, 1951) is defined as:
\begin{equation}
    \KL(\mu\|\nu) = \int_{\R^d} p(x)\log\frac{p(x)}{q(x)}\,dx
\end{equation}
If there exists $x$ such that $p(x)>0$ but $q(x)=0$ (i.e., the support of $\mu$ is not contained in the support of $\nu$), then $\KL(\mu\|\nu)=+\infty$.

Intuition: KL measures ``how much information is lost when approximating $\mu$ by $\nu$''. It is not symmetric: $\KL(\mu\|\nu)\neq\KL(\nu\|\mu)$.
\end{definition}

\begin{definition}[Total Variation distance]
The \textbf{total variation distance} is defined as:
\begin{equation}
    \mathrm{TV}(\mu,\nu) = \sup_{A\subseteq\R^d}\left|\mu(A) - \nu(A)\right| = \frac{1}{2}\int_{\R^d}|p(x)-q(x)|\,dx
\end{equation}
The first expression is the most general definition (taking the supremum over all measurable sets $A$); the second is an equivalent form when both measures have densities.

Intuition: TV finds the ``event $A$ that most effectively distinguishes $\mu$ from $\nu$''---for this event, the difference in probability assigned by the two measures is maximized. TV$\in[0,1]$.
\end{definition}

\textbf{The essential limitation of both metrics: each is blind to the geometry of the underlying space.}

\begin{example}
Let $\mu = \delta_x$ and $\nu = \delta_y$ be two Dirac masses. Then:
\begin{itemize}
    \item $\KL(\mu \| \nu) = +\infty$ for any $x \neq y$ (they have disjoint support).
    \item $\mathrm{TV}(\mu, \nu) = 1$ for any $x \neq y$, whether $|x-y| = 0.001$ or $|x-y| = 1000$.
\end{itemize}
Neither captures that $\delta_x$ and $\delta_y$ should be ``close'' when $x \approx y$.

\textbf{In contrast, for Wasserstein:} $W_2(\delta_x,\delta_y) = |x-y|$. As $x$ approaches $y$, the distance tends to zero---this is exactly the behavior we expect.
\end{example}

\begin{intuition}
Imagine you are a laborer with a pile of sand (shaped like $\mu$), and you need to reshape it into a different configuration ($\nu$).
The Wasserstein distance measures precisely this: \textbf{the minimum total transportation cost required to accomplish this task}.
Sand that is closer costs less to move, so this metric naturally respects the geometry of the underlying space.
\end{intuition}

\subsection{The Kantorovich formulation}

The relaxation below is due to Kantorovich (1942); the resulting metric is named after Vaserstein (1969).

\begin{definition}[2-Wasserstein distance]
Let $\Prob_2(\R^d)$ denote the space of probability measures with finite second moment. For $\mu, \nu \in \Prob_2(\R^d)$:
\begin{equation}\label{eq:W2}
    W_2(\mu, \nu) := \left( \inf_{\gamma \in \Gamma(\mu, \nu)} \int_{\R^d \times \R^d} |x - y|^2 \, d\gamma(x, y) \right)^{1/2}
\end{equation}
where $\Gamma(\mu, \nu)$ is the set of \textbf{couplings}:
\[
    \Gamma(\mu, \nu) := \left\{ \gamma \in \Prob(\R^d \times \R^d) \;\middle|\; 
    \begin{aligned}
        &\gamma(A \times \R^d) = \mu(A) \;\;\forall A\\
        &\gamma(\R^d \times B) = \nu(B) \;\;\forall B
    \end{aligned}
    \right\}
\]
\end{definition}

\begin{intuition}
\textbf{Unpacking the formula term by term:}

$W_2(\mu,\nu) = \Big(\inf_\gamma \int|x-y|^2\,d\gamma(x,y)\Big)^{1/2}$

\begin{itemize}
    \item $\gamma(x,y)$: a transport plan. The ``density'' of $\gamma$ near the point $(x,y)$ represents ``how much mass is moved from location $x$ to location $y$''.
    \item $|x-y|^2$: the unit cost of moving mass from $x$ to $y$ (the squared distance). Why squared? (1) It is mathematically more convenient (related to inner products and energy); (2) it penalizes long-distance transport more heavily (moving 10 meters costs 100 times as much as moving 1 meter, not 10 times).
    \item $\int|x-y|^2\,d\gamma$: the \textbf{total transport cost} under plan $\gamma$ = $\sum$(each unit of mass $\times$ its squared transport distance).
    \item $\inf_\gamma$: find the plan with the \textbf{minimum cost} among all feasible plans.
    \item $(\cdot)^{1/2}$: taking the square root gives $W_2$ the dimension of a ``distance'' (same order as $|x-y|$).
\end{itemize}

\textbf{Intuition for the marginal conditions:}
\begin{itemize}
    \item $\gamma(A\times\R^d) = \mu(A)$: the total mass transported out of region $A$ equals the total mass of $\mu$ on $A$. In other words, all the dirt in $A$ must be moved out---none may be left behind.
    \item $\gamma(\R^d\times B) = \nu(B)$: the total mass arriving at region $B$ equals the total mass of $\nu$ on $B$. In other words, the demand at $B$ must be met exactly---no more, no less.
\end{itemize}

\textbf{An assignment analogy:} think of a dispatch problem---$\mu$ is the distribution of available vehicles, $\nu$ the distribution of requested destinations, $\gamma$ the assignment ``which vehicle serves which destination,'' and $|x-y|^2$ the cost of a trip. Then $W_2$ is the minimum total cost of the optimal assignment (after taking the square root).
\end{intuition}

\subsection{The Monge formulation and Brenier's theorem}

\textbf{Historical background:} In 1781, the French mathematician Gaspard Monge posed a practical problem:
how to transport a pile of earth (excavated from a quarry) to designated construction sites at minimum cost.
His idea was very natural: \textbf{assign each grain of earth a definite destination}.

\begin{definition}[Monge's problem]
Find a map $T:\R^d\to\R^d$ (a ``transport map'') that moves the mass at each location $x$ to location $T(x)$,
subject to:
\begin{enumerate}
    \item \textbf{Pushforward constraint:} $T_\#\mu = \nu$, meaning $\nu(B) = \mu(T^{-1}(B))$ for all Borel $B$.
    \item \textbf{Minimize total cost:}
    \begin{equation}\label{eq:monge}
        \inf_{T:\,T_\#\mu=\nu}\int_{\R^d}|x - T(x)|^2\,d\mu(x)
    \end{equation}
\end{enumerate}
\end{definition}

\begin{intuition}
\textbf{Unpacking the Monge problem term by term:}

\textbf{Meaning of the map $T$:} $T$ is an ``instruction table''---for each location $x$ in the source distribution $\mu$, $T(x)$ tells you ``where the dirt at $x$ should be sent''.

\textbf{Meaning of the constraint $T_\#\mu = \nu$:} After the transport is complete, the new distribution of dirt is exactly $\nu$.
\begin{itemize}
    \item In plain language: if the dirt at every location $x$ is moved according to $T$, the resulting shape must be $\nu$.
    \item In formulas: $\nu(B) = \mu(T^{-1}(B)) = \mu(\{x:T(x)\in B\})$ = ``the total amount of dirt transported into region $B$ equals the amount that $\nu$ requires in $B$''.
\end{itemize}

\textbf{Meaning of the cost $\int|x-T(x)|^2\,d\mu(x)$:}
\begin{itemize}
    \item $|x-T(x)|^2$: the dirt at location $x$ is moved to $T(x)$; this is the squared transport distance.
    \item $d\mu(x)$: how much dirt is at location $x$ (weighted by $\mu$).
    \item The integral: total transport cost of all dirt = $\sum$(amount of each portion of dirt $\times$ squared transport distance).
\end{itemize}

\textbf{One-dimensional example:} Let $\mu = \text{Uniform}[0,1]$ and $\nu = \text{Uniform}[1,2]$.
The optimal map is clearly $T(x) = x+1$ (shift every grain of sand one unit to the right). Cost = $\int_0^1|x-(x+1)|^2\,dx = \int_0^1 1\,dx = 1$.

\textbf{Another example:} $\mu = \text{Uniform}[0,2]$, $\nu = \text{Uniform}[0,1]$ (compressing a wide distribution into a narrow one).
The optimal map is $T(x) = x/2$. Cost = $\int_0^2|x - x/2|^2\cdot\frac{1}{2}\,dx = \int_0^2\frac{x^2}{4}\cdot\frac{1}{2}\,dx = \frac{1}{8}\cdot\frac{x^3}{3}\Big|_0^2 = \frac{1}{3}$.
\end{intuition}

\textbf{The principal difficulty with Monge's formulation:}

The map $T$ requires that \textbf{all} mass at each $x$ goes to the same place---splitting is not allowed. This leads to two problems:

\textbf{Problem 1: the solution may not exist.}
Let $\mu = \delta_0$ (all mass at the origin), $\nu = \frac{1}{2}\delta_{-1}+\frac{1}{2}\delta_1$ (half at $-1$, half at $1$).
No map $T$ satisfying $T_\#\delta_0 = \nu$ exists---since $T(0)$ can only be a single point, it cannot simultaneously be $-1$ and $1$.
But the Kantorovich coupling $\gamma = \frac{1}{2}\delta_{(0,-1)}+\frac{1}{2}\delta_{(0,1)}$ is feasible (splitting the mass at the origin in half).

\textbf{Problem 2: the optimization problem is non-convex.}
The constraint $T_\#\mu=\nu$ is highly nonlinear. If both $T_1$ and $T_2$ satisfy $T_{1\#}\mu=T_{2\#}\mu=\nu$,
their average $\frac{T_1+T_2}{2}$ typically does \textbf{not} satisfy the constraint. This makes direct optimization very difficult.

\textbf{Kantorovich's relaxation:} replace the map $T$ with a joint distribution $\gamma$, turning the non-convex constraint into a \textbf{linear constraint}
(the marginal conditions are linear in $\gamma$), and the non-convex problem into a linear program---which always has a solution and can be computed efficiently.

\medskip

\textbf{The relationship between Monge and Kantorovich:}

If the optimal map $T$ exists, then the corresponding Kantorovich coupling is $\gamma = (\id, T)_\#\mu$,
i.e., the joint distribution is concentrated on the graph of the map $T$: $\{(x,T(x)):x\in\R^d\}$.
Conversely, if the optimal coupling $\gamma$ happens to be concentrated on the graph of some map, then that map is the solution to the Monge problem.

\begin{theorem}[Brenier, 1991]\label{thm:brenier}
If $\mu$ is absolutely continuous with respect to the Lebesgue measure, then:
\begin{enumerate}
    \item The optimal transport map $T$ exists and is unique $\mu$-a.e.
    \item $T = \nabla \phi$ for some convex function $\phi: \R^d \to \R$.
\end{enumerate}
Consequently:
\begin{equation}
    W_2^2(\mu, \nu) = \int_{\R^d} |x - \nabla\phi(x)|^2 \, d\mu(x)
\end{equation}
\end{theorem}

\begin{intuition}
\textbf{What does Brenier's theorem say? Three levels of understanding:}

\textbf{Level 1: Optimal transport is deterministic.}
Although Kantorovich allows ``splitting the mass at $x$ among multiple destinations $y$'' (probabilistic transport),
when $\mu$ is absolutely continuous, the optimal plan \textbf{requires no splitting at all}---the mass at each $x$ is moved entirely to a single definite location $T(x)$.
This greatly simplifies the problem: instead of optimizing over a joint distribution $\gamma$ (an infinite-dimensional object), we optimize over a map $T$ (relatively simpler).

\textbf{Level 2: The optimal map does not cross.}
$T = \nabla\phi$ with $\phi$ convex means that $T$ is ``monotone'' (the high-dimensional analogue of monotonicity = gradient of a convex function).
Intuition: if two transport paths cross ($x_1$'s dirt goes to $y_2$, $x_2$'s dirt goes to $y_1$, and the paths cross),
then we can ``uncross'' them ($x_1\to y_1$, $x_2\to y_2$), reducing the total cost.
Therefore the optimal plan must be uncrossed, and ``uncrossed'' $\iff$ the map is the gradient of a convex function.

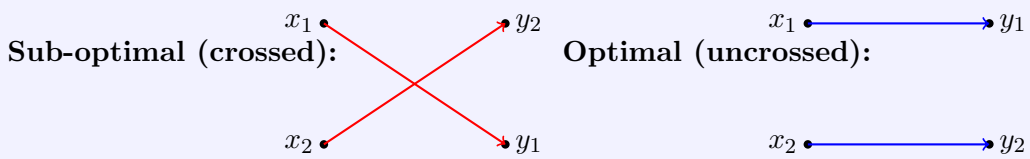
\begin{center}
\begin{tikzpicture}[scale=0.8]
    \node at (-2.5, 1.5) {\textbf{Sub-optimal (crossed):}};
    \fill (0,2) circle(2pt) node[left]{$x_1$};
    \fill (0,0) circle(2pt) node[left]{$x_2$};
    \fill (3,2) circle(2pt) node[right]{$y_2$};
    \fill (3,0) circle(2pt) node[right]{$y_1$};
    \draw[->, red, thick] (0,2) -- (3,0);
    \draw[->, red, thick] (0,0) -- (3,2);
    
    \node at (6.5, 1.5) {\textbf{Optimal (uncrossed):}};
    \fill (8,2) circle(2pt) node[left]{$x_1$};
    \fill (8,0) circle(2pt) node[left]{$x_2$};
    \fill (11,2) circle(2pt) node[right]{$y_1$};
    \fill (11,0) circle(2pt) node[right]{$y_2$};
    \draw[->, blue, thick] (8,2) -- (11,2);
    \draw[->, blue, thick] (8,0) -- (11,0);
\end{tikzpicture}
\end{center}

Cost of the crossed plan: $|x_1-y_2|^2 + |x_2-y_1|^2$.
Cost of the uncrossed plan: $|x_1-y_1|^2 + |x_2-y_2|^2$.
The difference is $2(x_1-x_2)\cdot(y_2-y_1)$ (obtained by expanding the squares). When the paths cross, $(x_1-x_2)$ and $(y_1-y_2)$ point in opposite directions, so this difference is $\geq 0$---crossing is always more expensive.

\textbf{Level 3: The geometric meaning of convexity of $\phi$.}
$\nabla\phi$ is the gradient of a convex function, so it is a ``spreading'' map---it does not send two distant points to the same neighborhood.
More precisely: the gradient of a convex function satisfies the monotonicity condition
\[
(\nabla\phi(x_1)-\nabla\phi(x_2))\cdot(x_1-x_2) \geq 0
\]
This guarantees that the map ``preserves order''---points on the left are sent to the left, points on the right to the right, with no reversals (crossings).
\end{intuition}

\textbf{A common point of confusion: Brenier's theorem fails for discrete distributions.}

Consider the following counterexample:
\[
\mu = 0.6\,\delta_a + 0.4\,\delta_b, \qquad \nu = 0.5\,\delta_c + 0.5\,\delta_d
\]

If a map $T$ satisfying $T_\#\mu = \nu$ exists, then $T$ can only send $a$ to one of $c$ or $d$, and likewise for $b$.
But no assignment works:
\begin{itemize}
    \item $T(a)=c,\,T(b)=d$: then $\nu(\{c\}) = \mu(T^{-1}(\{c\})) = \mu(\{a\}) = 0.6 \neq 0.5$~~\text{\ding{55}}
    \item $T(a)=d,\,T(b)=c$: then $\nu(\{d\}) = \mu(\{a\}) = 0.6 \neq 0.5$~~\text{\ding{55}}
    \item $T(a)=c,\,T(b)=c$: then $\nu(\{c\}) = 0.6+0.4 = 1 \neq 0.5$~~\text{\ding{55}}
    \item All other cases fail similarly.
\end{itemize}

\textbf{Conclusion: No transport map from $\mu$ to $\nu$ exists.}

An optimal \textbf{coupling}, however, does exist. The optimal plan is: split 0.5 from $a$ to $c$, send the remaining 0.1 to $d$; send all 0.4 from $b$ to $d$:
\[
\gamma = 0.5\,\delta_{(a,c)} + 0.1\,\delta_{(a,d)} + 0.4\,\delta_{(b,d)}
\]
Verification of marginals: the total mass of $\gamma$ on $a$ is $0.5+0.1 = 0.6 = \mu(\{a\})$~\checkmark;
the total mass of $\gamma$ on $c$ is $0.5 = \nu(\{c\})$~\checkmark; and so on.

\textbf{This is why Brenier's theorem requires the condition ``$\mu$ is absolutely continuous'':}
\begin{itemize}
    \item \textbf{Discrete distributions} (as above): a finite amount of mass is concentrated at a single point, and transport may \textbf{necessarily require splitting}, so a map $T$ is insufficient---one must use a coupling $\gamma$.
    \item \textbf{Absolutely continuous distributions} (with density $\rho$): each point carries only an ``infinitesimal'' mass $\rho(x)dx$ and never needs to be split---even if $\nu$ is discrete, no source point has ``a bulk of mass that must be allocated to different destinations'' (since each point of $\mu$ has zero mass). Therefore a deterministic map $T$ suffices. Note: only $\mu$ needs to be absolutely continuous; there is \textbf{no} restriction on $\nu$.
\end{itemize}

\textbf{In one sentence:} Absolutely continuous $\Rightarrow$ mass is ``continuously spread'' $\Rightarrow$ no splitting needed $\Rightarrow$ map $T$ exists $\Rightarrow$ Brenier applies.

Discrete/singular distribution $\Rightarrow$ there are ``bulk'' masses $\Rightarrow$ splitting may be necessary $\Rightarrow$ only a coupling $\gamma$ can describe the solution $\Rightarrow$ Brenier does not apply.

\subsection{A concrete example: Gaussians}

\begin{example}\label{ex:gaussian}
For $\mu = \mathcal{N}(m_1, \Sigma_1)$ and $\nu = \mathcal{N}(m_2, \Sigma_2)$:
\begin{equation}\label{eq:W2_gaussian}
    W_2^2(\mu, \nu) = |m_1 - m_2|^2 + \tr\!\left(\Sigma_1 + \Sigma_2 - 2\bigl(\Sigma_1^{1/2}\Sigma_2\Sigma_1^{1/2}\bigr)^{1/2}\right)
\end{equation}
Special case: if $\Sigma_1 = \Sigma_2 = \sigma^2 I$, then $W_2(\mu, \nu) = |m_1 - m_2|$.

\textbf{Key steps of the derivation:} For Gaussian distributions, the optimal transport map is affine: $T(x) = Ax + b$.
By Brenier's theorem, $T = \nabla\phi$ ($\phi$ convex) requires $A$ to be a \textbf{symmetric positive definite} matrix.
The pushforward condition $T_\#\mu = \nu$ gives $b = m_2 - Am_1$ and $A\Sigma_1 A^T = \Sigma_2$.
Combined with $A = A^T\succ 0$, the unique solution is $A = \Sigma_1^{-1/2}(\Sigma_1^{1/2}\Sigma_2\Sigma_1^{1/2})^{1/2}\Sigma_1^{-1/2}$.
The cost $\int|x-Tx|^2\,d\mu = |m_1-m_2|^2 + \tr\bigl((I-A)\Sigma_1(I-A)^T\bigr)$,
which simplifies to the stated formula using $A^2\Sigma_1 = \Sigma_2$ (by the symmetry of $A$).
\end{example}

\subsection{Key properties}

The space $(\Prob_2(\R^d), W_2)$ is:
\begin{enumerate}
    \item A \textbf{complete separable metric space} (Polish space).
    \item It metrizes \textbf{weak convergence + convergence of second moments}.
    \item Most profoundly: it admits a \textbf{formal Riemannian structure} (Otto, 2001).
\end{enumerate}

We now explain each of these three properties in detail.

\subsubsection*{Complete (completeness)}

\begin{definition}[Completeness]
A metric space $(X,d)$ is \textbf{complete} if every Cauchy sequence in it converges.

Cauchy sequence: $\{x_n\}$ satisfies that for any $\epsilon>0$, there exists $N$ such that $d(x_m,x_n)<\epsilon$ whenever $m,n>N$ (the points get arbitrarily close to each other).

Completeness means: as long as points get closer and closer together, they must converge to some limit within the space---they cannot ``escape outside the space''.
\end{definition}

\textbf{Examples:}
\begin{itemize}
    \item $\R$ is complete: every Cauchy sequence converges to some real number.
    \item $\mathbb{Q}$ (the rationals) is not complete: $1, 1.4, 1.41, 1.414, \ldots$ is a Cauchy sequence in $\mathbb{Q}$, but its limit $\sqrt{2}\notin\mathbb{Q}$.
    \item The open interval $(0,1)$ is not complete: $\frac{1}{n}$ is a Cauchy sequence, but its limit $0\notin(0,1)$.
\end{itemize}
\textbf{Why this matters:} Completeness guarantees that when we take limits (such as the $\tau\to 0$ limit in the JKO scheme), the limiting object is still a legitimate probability measure. In an incomplete space, ``taking limits'' may lead outside the space, making rigorous analysis much harder.

\subsubsection*{Separable (separability)}

\begin{definition}[Separability]
A metric space $(X,d)$ is \textbf{separable} if it has a \textbf{countable dense subset}---
i.e., there exists a countable set $\{x_1,x_2,\ldots\}\subset X$ such that every point in $X$ can be approximated arbitrarily well by elements of this set.
\end{definition}

\textbf{Examples:}
\begin{itemize}
    \item $\R$ is separable: the rationals $\mathbb{Q}$ are countable and dense in $\R$.
    \item $\R^d$ is separable: points with all rational coordinates $\mathbb{Q}^d$ are dense.
    \item $\Prob_2(\R^d)$ is separable: finite mixtures of Dirac masses $\sum_{i=1}^n w_i\delta_{x_i}$ ($w_i\in\mathbb{Q}$, $x_i\in\mathbb{Q}^d$) can approximate any probability measure.
\end{itemize}
\textbf{Why this matters:} Separability ensures that the space is ``not too large''---we can approximate everything using countably many ``basic elements''.
This is a necessary condition for many theorems in topology and functional analysis (e.g., Prokhorov's theorem, weak compactness).

\subsubsection*{Weak convergence}

\begin{definition}[Weak convergence of measures]
A sequence of probability measures $\mu_n$ \textbf{converges weakly} to $\mu$ (written $\mu_n\rightharpoonup\mu$) if for all bounded continuous functions $f$:
\[
\int f\,d\mu_n \to \int f\,d\mu \quad\text{as }n\to\infty
\]
\end{definition}

\textbf{Intuition:} Weak convergence means ``the shapes of the distributions become increasingly similar''---for any ``smooth observable'' $f$,
the expectation under $\mu_n$, $\E_{\mu_n}[f]$, approaches the expectation under $\mu$, $\E_\mu[f]$.

\textbf{Examples:}
\begin{itemize}
    \item $\mu_n = \mathcal{N}(0, 1/n)$ (Gaussians with vanishing variance) converges weakly to $\delta_0$. Intuition: the bell curve becomes narrower and narrower, eventually collapsing to a point.
    \item $\mu_n = \text{Uniform}[-1/n, 1/n]$ converges weakly to $\delta_0$.
    \item $\mu_n = \frac{1}{n}\sum_{i=1}^n \delta_{i/n}$ converges weakly to $\text{Uniform}[0,1]$. Intuition: the uniformly scattered points become increasingly dense, approaching the continuous uniform distribution.
\end{itemize}

\textbf{Weak convergence vs.\ Wasserstein convergence:}

$W_2(\mu_n,\mu)\to 0$ $\iff$ $\mu_n\rightharpoonup\mu$ \textbf{and} the second moments converge ($\int|x|^2\,d\mu_n\to\int|x|^2\,d\mu$).

Pure weak convergence does not control the ``tails''---for example, $\mu_n = (1-\frac{1}{n})\delta_0 + \frac{1}{n}\delta_n$ converges weakly to $\delta_0$,
but $W_2(\mu_n,\delta_0) = \sqrt{\frac{1}{n}\cdot n^2} = \sqrt{n}\to\infty$.
Wasserstein convergence is stronger because it additionally requires that mass cannot ``escape to infinity''.

\subsubsection*{Formal Riemannian structure}

\textbf{We now use the language above to describe the Wasserstein space:}

\begin{definition}[Riemannian manifold --- informal]
A \textbf{Riemannian manifold} $(M,g)$ is a smooth manifold $M$ equipped with an \textbf{inner product} $g_p$ on the tangent space $T_pM$ at each point $p$. This inner product allows us to:
\begin{itemize}
    \item Measure the ``length'' of tangent vectors: $\|v\|_p = \sqrt{g_p(v,v)}$
    \item Compute the ``angle'' between two tangent vectors: $\cos\theta = \frac{g_p(u,v)}{\|u\|_p\|v\|_p}$
    \item Define ``shortest paths'' (geodesics) and ``distance'': $d(p,q) = \inf_\gamma \int_0^1\|\dot\gamma(t)\|_{\gamma(t)}\,dt$
    \item Define the ``gradient'' of a function: $g_p(\text{grad}\,f, v) = Df_p[v]$
\end{itemize}
\end{definition}

\textbf{Simple examples:}
\begin{itemize}
    \item $\R^n$ with the standard inner product $g_p(u,v)=u\cdot v$: the simplest Riemannian manifold. The tangent space is $\R^n$ itself, and geodesics are straight lines.
    \item The sphere $S^2$ with the induced metric: the tangent space at a point is the tangent plane to the sphere there, and geodesics are great circle arcs.
\end{itemize}

\textbf{Otto's discovery: the probability space $(\Prob_2(\R^d), W_2)$ also ``looks like'' a Riemannian manifold.}

Analogy table:
\begin{center}
\renewcommand{\arraystretch}{1.3}
\begin{tabular}{@{}lll@{}}
\toprule
\textbf{Concept} & \textbf{Finite-dim.\ manifold $M$} & \textbf{Wasserstein space $\Prob_2$} \\
\midrule
Point & $p\in M$ & $\rho\in\Prob_2(\R^d)$ \\
Tangent vector & $v\in T_pM$ & Gradient field $\nabla\phi\in L^2(\rho)$ \\
Inner product & $g_p(u,v)$ & $\int\nabla\phi\cdot\nabla\psi\,\rho\,dx$ \\
Distance & $d(p,q)$ (geodesic length) & $W_2(\mu,\nu)$ (Benamou--Brenier) \\
Gradient & $\text{grad}\,f$ & $\text{grad}_W\mathcal{F}=\nabla\frac{\delta\mathcal{F}}{\delta\rho}$ \\
Gradient flow & $\dot p = -\text{grad}\,f$ & Fokker--Planck equation \\
\bottomrule
\end{tabular}
\end{center}

\textbf{Why ``formal''?}

Because $\Prob_2(\R^d)$ is an \textbf{infinite-dimensional} space and is not, strictly speaking, a finite-dimensional smooth manifold.
Otto's ``Riemannian structure'' faces several technical difficulties in a rigorous mathematical sense:
\begin{itemize}
    \item The tangent space $T_\rho\Prob_2$ is not a standard Banach/Hilbert space ($\rho$ appears as a weight in the inner product).
    \item There is no uniform smooth structure (the tangent spaces at different ``points'' $\rho$ cannot be naturally ``glued together'').
    \item Operations like the exponential map are only well-defined in special cases.
\end{itemize}
However, this ``formal'' Riemannian perspective is \textbf{computationally} entirely correct---the
PDEs, gradient formulas, and the JKO scheme derived from it are all rigorously valid.
Later, Ambrosio--Gigli--Savar\'e (2008) made these conclusions rigorous using the theory of ``gradient flows in metric spaces'',
without relying on manifold structure. Thus, although Otto calculus is ``formal'', the conclusions it yields can all be rigorously proved.

\begin{keypoint}
$(\Prob_2(\R^d), W_2)$ is not merely a metric space---it is (formally) an infinite-dimensional Riemannian manifold.
This means we can speak of ``tangent vectors'', ``inner products'', ``gradients'', and ``gradient flows'' in the space of probability measures.
This is the foundation of the entire story.
\end{keypoint}

\section{The Continuity Equation: Conservation of Mass}

\textbf{Where we are:} Section~1 defined the Wasserstein distance---a ``static'' concept (comparing the distance between two fixed distributions). But to discuss gradient flows, we need to describe how distributions \textbf{evolve over time}. The continuity equation is precisely the tool for this: it tells us ``if mass flows according to a velocity field $v$, how does the density evolve?'' The Benamou--Brenier formula, Fokker--Planck equation, and JKO scheme that follow are all built upon this equation.

\textbf{Core idea:} the continuity equation is the \textbf{local differential form of mass conservation}: the density inside any infinitesimal region changes only through the net flux across its boundary.

\subsection{Physical picture}

Imagine ink diffusing in a glass of water. At each instant $t$, the ink concentration is $\rho_t(x)$,
and the average velocity of ink particles at position $x$ is $v_t(x)$.
The most fundamental physical law is: \textbf{mass is neither created nor destroyed}.

\subsection{Derivation from first principles}

Take any fixed region $\Omega \subset \R^d$. The total mass inside is:
\[
    M(t) = \int_\Omega \rho_t(x) \, dx
\]

The rate of change of mass equals the net flux through the boundary $\partial\Omega$:
\begin{equation}\label{eq:mass_balance}
    \frac{d}{dt}\int_\Omega \rho_t(x)\,dx = -\int_{\partial\Omega} \rho_t(x)\,v_t(x)\cdot \hat{n}(x)\,dS(x)
\end{equation}
where $\hat{n}$ is the outward unit normal (the minus sign: outflow \emph{decreases} mass inside).

\begin{intuition}
\textbf{Where does the integral $\int_{\partial\Omega}\rho\,v\cdot\hat{n}\,dS$ come from?}

Consider a small surface element $dS$ on the boundary $\partial\Omega$, with outward unit normal $\hat{n}$. Near this surface element:
\begin{enumerate}
    \item The \textbf{mass density} is $\rho(x)$ (how much mass per unit volume).
    \item The \textbf{velocity} is $v(x)$ (the direction and speed of mass at that point).
    \item The \textbf{normal component} $v\cdot\hat{n}$: the projection of velocity onto the normal direction of the surface element.
    \begin{itemize}
        \item $v\cdot\hat{n} > 0$: velocity has an outward component---mass is \textbf{flowing out}.
        \item $v\cdot\hat{n} < 0$: velocity has an inward component---mass is \textbf{flowing in}.
        \item $v\cdot\hat{n} = 0$: velocity is parallel to the surface element---mass \textbf{moves along the boundary} without entering or exiting.
    \end{itemize}
\end{enumerate}

\textbf{Mass passing through this surface element per unit time} = density $\times$ normal velocity $\times$ area:
\[
d\Phi = \rho(x)\cdot(v(x)\cdot\hat{n}(x))\cdot dS
\]

\textbf{Why ``density $\times$ normal velocity''?} Picture a checkpoint of area $dS$ across a road (the surface element).
The number of vehicles crossing per unit time equals the traffic density times the velocity component normal to the checkpoint.
When traffic approaches at an angle, only the normal component $v\cdot\hat n$ contributes to throughput; the
tangential component moves vehicles along the checkpoint without crossing it.

Integrating over the entire boundary gives the \textbf{total net outflux}:
\[
\text{Net outflux} = \int_{\partial\Omega}\rho\,v\cdot\hat{n}\,dS
\]
By conservation of mass: the rate of decrease of mass inside $\Omega$ = net outflux, hence the minus sign.
\end{intuition}

\textbf{Step-by-step derivation:}

\textbf{Step 1.} Apply the divergence theorem to the right side:
\[
    -\int_{\partial\Omega}\rho_t v_t \cdot \hat{n}\,dS = -\int_\Omega \nabla\cdot(\rho_t v_t)\,dx
\]

\textbf{Step 2.} Move $\frac{d}{dt}$ inside the integral on the left (justified since $\Omega$ is fixed):
\[
    \int_\Omega \partial_t\rho_t(x)\,dx = -\int_\Omega \nabla\cdot(\rho_t v_t)(x)\,dx
\]

\textbf{Step 3.} Since this holds for all $\Omega$ and the integrand is continuous, it must vanish pointwise (du Bois-Reymond lemma, Appendix~\ref{app:analysis}):

\begin{equation}\label{eq:continuity}
    \boxed{\partial_t \rho_t + \nabla\cdot(\rho_t\,v_t) = 0}
\end{equation}

This is the \textbf{continuity equation}.

\begin{intuition}
Expanding the divergence operator: $\nabla\cdot(\rho v) = v\cdot\nabla\rho + \rho\,\nabla\cdot v$. So the continuity equation can also be written as:
\[
\partial_t\rho + v\cdot\nabla\rho = -\rho\,\nabla\cdot v
\]
The left-hand side $\frac{D\rho}{Dt} = \partial_t\rho + v\cdot\nabla\rho$ is the rate of change of density following the fluid (the material derivative),
and the right-hand side $-\rho\nabla\cdot v$ represents density changes due to compression/expansion of the fluid.
If the fluid is incompressible ($\nabla\cdot v = 0$), then the density is constant along streamlines.
\end{intuition}

\subsection{Weak formulation}

For non-smooth $\rho_t$ (or when $\rho_t$ is a measure, not a function), we use the weak form. Multiply \eqref{eq:continuity} by a test function $\varphi \in C_c^\infty(\R^d\times[0,T])$ and integrate by parts:

\begin{equation}\label{eq:continuity_weak}
    \int_0^T\!\int_{\R^d}\left[\partial_t\varphi(x,t) + v_t(x)\cdot\nabla_x\varphi(x,t)\right]d\mu_t(x)\,dt = 0
\end{equation}

\textbf{Derivation details:} Starting from $\partial_t\rho + \nabla\cdot(\rho v) = 0$, multiply both sides by $\varphi$ and integrate:
\[
\int_0^T\!\int \varphi\,\partial_t\rho\,dx\,dt + \int_0^T\!\int\varphi\,\nabla\cdot(\rho v)\,dx\,dt = 0
\]
Integration by parts in $t$ for the first term: $\int\varphi\,\partial_t\rho\,dt = -\int\partial_t\varphi\,\rho\,dt$ (boundary terms vanish since $\varphi$ has compact support).\\
Integration by parts in $x$ for the second term: $\int\varphi\,\nabla\cdot(\rho v)\,dx = -\int\nabla\varphi\cdot(\rho v)\,dx$.\\
Combining: $-\int_0^T\!\int[\partial_t\varphi + v\cdot\nabla\varphi]\rho\,dx\,dt = 0$; negating both sides gives the result.

\subsection{Connection to particle ODEs}

\begin{proposition}\label{prop:ode_continuity}
If particles follow the ODE $\dot{X}_t = v_t(X_t)$ with initial distribution $X_0 \sim \mu_0$, then $\mu_t := (X_t)_\#\mu_0$ satisfies the continuity equation with velocity field $v_t$.
\end{proposition}

\begin{proof}
For any $\varphi \in C_c^\infty(\R^d)$:
\begin{align}
    \frac{d}{dt}\int\varphi\,d\mu_t 
    &= \frac{d}{dt}\int\varphi(X_t(x))\,d\mu_0(x) \label{eq:step1}\\
    &= \int\nabla\varphi(X_t(x))\cdot\dot{X}_t(x)\,d\mu_0(x) \label{eq:step2}\\
    &= \int\nabla\varphi(X_t(x))\cdot v_t(X_t(x))\,d\mu_0(x) \label{eq:step3}\\
    &= \int\nabla\varphi(y)\cdot v_t(y)\,d\mu_t(y) \label{eq:step4}
\end{align}
Step \eqref{eq:step1}: definition of pushforward. Step \eqref{eq:step2}: chain rule. Step \eqref{eq:step3}: substitute ODE. Step \eqref{eq:step4}: change of variables $y = X_t(x)$, back to pushforward.

This is exactly the weak form of $\partial_t\mu_t + \nabla\cdot(\mu_t v_t) = 0$.
\end{proof}

\section{The Riemannian Structure of Wasserstein Space}

\textbf{Where we are:} Section~1 gave us the Wasserstein distance, and Section~2 gave us the continuity equation. Now we unify the two within a single geometric framework: \textbf{the Wasserstein distance = the geodesic distance in probability space}, and the velocity field described by the continuity equation is precisely the ``tangent vector''. This Riemannian structure is the foundation for everything that follows (gradients, gradient flows, the JKO scheme).

\medskip
The prerequisite differential geometry (manifolds, tangent spaces, cotangent spaces, dual spaces, musical isomorphisms, Riemannian metrics, gradients, gradient flows)
can be found in \textbf{Appendix~\ref{app:diffgeom}}. If you are already familiar with these concepts, you may proceed directly.

\subsection{The Benamou--Brenier formula}

So far, the Wasserstein distance has been a ``static'' optimization problem (finding an optimal coupling).
Benamou and Brenier discovered an equivalent ``dynamic'' formulation---this bridge directly connects optimal transport with the continuity equation.

\begin{theorem}[Benamou--Brenier, 2000]\label{thm:BB}
\begin{equation}\label{eq:BB}
    W_2^2(\mu_0, \mu_1) = \inf_{(\rho_t, v_t)} \left\{\int_0^1\!\int_{\R^d} |v_t(x)|^2\,\rho_t(x)\,dx\,dt\right\}
\end{equation}
where the infimum is over all pairs $(\rho_t, v_t)_{t\in[0,1]}$ with $v_t\in L^2(\rho_t;\R^d)$, satisfying:
\begin{itemize}
    \item Continuity equation (in the weak sense): $\partial_t\rho_t + \nabla\cdot(\rho_t v_t) = 0$
    \item Boundary conditions: $\rho_0 = \mu_0,\; \rho_1 = \mu_1$
\end{itemize}
\end{theorem}

\begin{intuition}
\textbf{Analogy:} In Euclidean space, the distance between two points $a,b\in\R^n$ can be written as:
\[
|a-b|^2 = \inf_{\gamma:[0,1]\to\R^n}\left\{\int_0^1|\dot\gamma(t)|^2\,dt \;\middle|\; \gamma(0)=a,\,\gamma(1)=b\right\}
\]
(The optimal path is a straight line with constant velocity $\dot\gamma = b-a$, and cost $= |b-a|^2$.)

The Benamou--Brenier formula embodies exactly the same idea, except:
\begin{itemize}
    \item ``Points'' become probability distributions $\rho_t$.
    \item ``Paths'' become continuous evolutions of distributions (satisfying mass conservation).
    \item ``Squared velocity'' becomes $\int|v_t|^2\rho_t\,dx$ (kinetic energy weighted by $\rho_t$).
\end{itemize}
Therefore, the Wasserstein distance is precisely the \textbf{geodesic length in probability space}.
\end{intuition}

\textbf{Why $\int|v|^2\rho\,dx$ and not $\int|v|^2\,dx$?}

Consider an infinitesimal fluid element with mass $dm = \rho\,dx$; its kinetic energy is $\frac{1}{2}|v|^2\,dm = \frac{1}{2}|v|^2\rho\,dx$.
Thus $\int|v|^2\rho\,dx$ is the total kinetic energy (up to the factor $\frac{1}{2}$).
This weighting is crucial: if the density in some region is zero (no mass), the velocity there ``costs nothing''.

\begin{intuition}
\textbf{Why can ``total kinetic energy'' describe the Wasserstein distance? The deeper reason:}

\textbf{Question 1: Why are the static and dynamic formulations equivalent?}

The static Kantorovich formulation says: $W_2^2 = \min_\gamma\int|x-y|^2\,d\gamma$---find the optimal ``who goes where''.

The dynamic Benamou--Brenier formulation says: $W_2^2 = \min_{(\rho,v)}\int_0^1\!\int|v|^2\rho\,dx\,dt$---find the optimal ``how to get there''.

Why do these two entirely different formulations give the same answer?

The key insight: \textbf{when the constraint enforces ``constant-speed straight-line motion'', the two are naturally equivalent.}

Suppose the optimal transport map $T$ sends $x$ to $T(x)$. If each mass particle travels at constant speed in a straight line from $x$ to $T(x)$,
then at time $t$ it is located at $x_t = (1-t)x + tT(x)$, with constant velocity $\dot x_t = T(x)-x$.

The time integral of kinetic energy along this path:
\[
\int_0^1|T(x)-x|^2\,dt = |T(x)-x|^2
\]
Integrating over all mass: $\int|T(x)-x|^2\,d\mu(x) = W_2^2(\mu,\nu)$.

So constant-speed straight-line motion gives kinetic energy = static transport cost. Any other path (non-straight or non-constant-speed) only increases the kinetic energy (by the Cauchy--Schwarz inequality).

\textbf{Question 2: Why use kinetic energy (velocity$^2$) rather than velocity itself?}

If we use $\int|v|\rho\,dx$ (first power of velocity), we obtain the $W_1$ distance (1-Wasserstein).
Using $\int|v|^2\rho\,dx$ (squared velocity) gives $W_2^2$. Why prefer $W_2$?

\begin{enumerate}
    \item \textbf{Physical naturalness: minimum kinetic energy $\iff$ geodesic (shortest path).}
    
    This requires explanation. In Riemannian geometry, the ``distance'' between two points is defined as:
    \[
    d(a,b)^2 = \inf_{\gamma}\int_0^1|\dot\gamma(t)|^2\,dt, \quad \gamma(0)=a,\;\gamma(1)=b
    \]
    The optimal $\gamma$ is called a \textbf{geodesic}---it is the ``straight line'' (shortest path) in the given geometry.
    
    In physics, in the absence of external forces (a free particle), the equation of motion is $\ddot\gamma = 0$ (constant-speed straight-line motion).
    This is equivalent to minimizing the action $S = \int_0^1\frac{1}{2}|\dot\gamma|^2\,dt$ (the time integral of kinetic energy).
    Therefore \textbf{free particle trajectory = geodesic = minimum kinetic energy path}.
    
    On curved spaces (e.g., the sphere), geodesics are no longer ``straight lines'' but great circle arcs---yet
    they are still the paths that minimize $\int|\dot\gamma|^2\,dt$.
    
    What Benamou--Brenier does is a perfect analogy: $\int_0^1\!\int|v|^2\rho\,dx\,dt$ is the ``kinetic energy integral in probability space'',
    and the path $\rho_t$ that minimizes it is the geodesic in Wasserstein space.
    
    \item \textbf{Mathematical convenience:} $|v|^2$ is smooth and strictly convex, making $W_2^2$ differentiable with respect to $\rho$,
    whereas $|v|$ is non-smooth at zero, and $W_1$ has worse mathematical properties (non-differentiable, no Riemannian structure).
    
    \item \textbf{Riemannian structure:} Only when using the squared velocity (rather than the absolute value) as the ``metric''
    does the probability space acquire an inner product structure---an inner product requires $\|v\|^2$ to be a quadratic function of $v$.
    If one uses $\|v\| = \int|v|\rho\,dx$, that is merely a norm (Finsler geometry) without an inner product, and ``gradients'' cannot be defined.
\end{enumerate}

\textbf{Question 3: What is the deeper intuition behind the dynamic formulation?}

Imagine you are a fluid dispatcher. You must deform a body of fluid from the shape $\mu_0$ into the shape $\mu_1$.
The rules are:
\begin{itemize}
    \item Fluid cannot be created or destroyed (continuity equation $\partial_t\rho + \nabla\cdot(\rho v)=0$).
    \item You must minimize the total ``fuel'' consumed---each fluid element's fuel consumption is proportional to velocity$^2\times$mass$\times$time.
\end{itemize}

This is precisely the Benamou--Brenier formula.

\textbf{What is the optimal strategy?} Have each mass particle travel in a \textbf{straight line at constant speed} from start to finish.
\begin{itemize}
    \item Why constant speed? Because $\int_0^1|v(t)|^2\,dt \geq (\int_0^1|v(t)|\,dt)^2$ (Jensen's inequality), with equality if and only if $|v|$ is constant.
    Going slow then fast, or fast then slow, always wastes more energy than traveling at constant speed.
    \item Why straight lines? Because detours = longer path = higher average speed required = greater kinetic energy.
\end{itemize}

Therefore \textbf{the Wasserstein geodesic is ``all mass particles traveling in straight lines at constant speed''}, and the total energy is $\int|x-T(x)|^2\,d\mu = W_2^2$.
\end{intuition}

\subsection{Why continuity equation + minimum kinetic energy gives the OT path}

Let us make the previous statement precise. The continuity equation alone is not an optimization principle:
\[
\partial_t\rho_t+\nabla\cdot(\rho_t v_t)=0
\]
only says that probability mass is transported without being created or destroyed. The optimization enters through the kinetic action
\[
\mathcal{A}(\rho,v)
=
\int_0^1\!\int_{\R^d}|v_t(x)|^2\rho_t(x)\,dx\,dt.
\]
The Benamou--Brenier problem is:
\[
\inf_{\rho_t,v_t}\mathcal{A}(\rho,v)
\]
subject to
\[
\partial_t\rho_t+\nabla\cdot(\rho_t v_t)=0,\qquad
\rho_0=\mu_0,\quad \rho_1=\mu_1.
\]
In words: among all mass-preserving ways to deform $\mu_0$ into $\mu_1$, choose the one with the smallest total kinetic energy.

\begin{intuition}
\textbf{The two halves of the argument.}

\textbf{First, every admissible dynamic path contains a static transport plan.}
If particles move according to
\[
\dot X_t=v_t(X_t),
\]
then the continuity equation says that $\rho_t=\mathrm{Law}(X_t)$. Each particle has an initial point $X_0=x$ and a final point $X_1=y$, so the whole flow induces a coupling
\[
\gamma=\mathrm{Law}(X_0,X_1)\in\Pi(\mu_0,\mu_1).
\]
Thus any dynamic path tells us ``who goes where.''

\medskip
\textbf{Second, the kinetic energy of any particle path dominates its squared displacement.}
For one trajectory $X_t$, Cauchy--Schwarz gives
\[
\int_0^1|\dot X_t|^2\,dt
\geq
\left|\int_0^1\dot X_t\,dt\right|^2
=
|X_1-X_0|^2.
\]
Equality holds exactly when the particle moves along a straight line at constant speed.
Averaging over all particles,
\[
\int_0^1\!\int |v_t(x)|^2\rho_t(x)\,dx\,dt
\geq
\int |y-x|^2\,d\gamma(x,y).
\]
Since $W_2^2(\mu_0,\mu_1)$ is the minimum of the right-hand side over all couplings $\gamma$, every admissible dynamic path satisfies
\[
\mathcal{A}(\rho,v)\geq W_2^2(\mu_0,\mu_1).
\]
No mass-preserving path can spend less kinetic energy than the optimal transport cost.
\end{intuition}

Conversely, take an optimal coupling $\gamma^\star\in\Pi(\mu_0,\mu_1)$. Move each pair $(x,y)$ along the constant-speed straight line
\[
X_t=(1-t)x+ty,\qquad \dot X_t=y-x.
\]
Define
\[
\rho_t=\bigl((1-t)x+ty\bigr)_\#\gamma^\star.
\]
This path satisfies the continuity equation because it is generated by moving particles. Its kinetic action is
\[
\int_0^1\!\int |y-x|^2\,d\gamma^\star(x,y)\,dt
=
\int |y-x|^2\,d\gamma^\star(x,y)
=
W_2^2(\mu_0,\mu_1).
\]
So the lower bound is achieved, and we obtain
\[
\boxed{
\inf_{\substack{\partial_t\rho+\nabla\cdot(\rho v)=0\\
\rho_0=\mu_0,\;\rho_1=\mu_1}}
\int_0^1\!\int |v_t|^2\rho_t\,dx\,dt
=
W_2^2(\mu_0,\mu_1)
}
\]
This is the Benamou--Brenier formula.

If $\mu_0$ is sufficiently regular, the optimal coupling is induced by a Brenier map $T$, so $y=T(x)$ and the minimizer becomes
\[
\rho_t=\bigl((1-t)\id+tT\bigr)_\#\mu_0.
\]
This is the \textbf{McCann displacement interpolation}: every infinitesimal piece of mass travels along a straight line at constant speed according to the optimal transport map. This is why the OT path is the ``straight line'' in Wasserstein space.

\begin{keypoint}
The continuity equation provides the \emph{constraint}: admissible mass-preserving paths.
The kinetic energy provides the \emph{criterion}: choose the least costly path.
Together they recover the optimal transport path, i.e.\ the Wasserstein geodesic.
\end{keypoint}

\subsection{Otto's Riemannian interpretation}

Otto (2001) proposed viewing $(\Prob_2(\R^d), W_2)$ as an infinite-dimensional Riemannian manifold:

\begin{definition}[Tangent space at $\rho$]
The tangent space at a point $\rho \in \Prob_2(\R^d)$ is:
\begin{equation}
    T_\rho\Prob_2 := \overline{\{\nabla\phi : \phi \in C_c^\infty(\R^d)\}}^{L^2(\rho;\R^d)}
\end{equation}
i.e., the $L^2(\rho;\R^d)$-closure of gradient vector fields (here $L^2(\rho;\R^d)$ denotes the space of vector fields $v:\R^d\to\R^d$ with $\int|v|^2\rho\,dx<\infty$).
\end{definition}

\textbf{Why are tangent vectors gradient fields?}

A ``tangent vector'' describes the instantaneous direction of change of $\rho$. By the continuity equation, the evolution of $\rho$ is determined by the velocity field $v$:
$\dot\rho = -\nabla\cdot(\rho v)$.
However, the same $\dot\rho$ may correspond to multiple different $v$ (since $\nabla\cdot(\rho v) = \nabla\cdot(\rho(v+w))$ whenever $\nabla\cdot(\rho w)=0$).

To remove this ambiguity, we choose the $v$ with \textbf{minimum $L^2(\rho)$ norm}.
For given $\rho>0$, any vector field $v$ can be uniquely decomposed as $v = \nabla\phi + w$, where $\nabla\cdot(\rho w) = 0$.
These two components are orthogonal in the $L^2(\rho)$ inner product (verification: $\int\rho\,\nabla\phi\cdot w\,dx = -\int\phi\,\nabla\cdot(\rho w)\,dx = 0$).
Therefore $\|v\|_{L^2(\rho)}^2 = \|\nabla\phi\|_{L^2(\rho)}^2 + \|w\|_{L^2(\rho)}^2$,
and the minimum-norm $v$ must be a gradient field $v = \nabla\phi$.

\begin{definition}[Riemannian metric]
For tangent vectors $\xi = \nabla\phi$, $\eta = \nabla\psi \in T_\rho\Prob_2$:
\begin{equation}\label{eq:otto_metric}
    \langle\xi,\eta\rangle_\rho := \int_{\R^d}\nabla\phi(x)\cdot\nabla\psi(x)\,\rho(x)\,dx
\end{equation}
\end{definition}

This is the inner product version of the instantaneous cost $\int|v|^2\rho\,dx = \|\nabla\phi\|_\rho^2$ appearing in the Benamou--Brenier formula.

\subsection{Geodesics = displacement interpolation}

The geodesic (shortest path) between $\mu_0$ and $\mu_1$ in $(\Prob_2, W_2)$ is the \textbf{McCann displacement interpolation}:
\begin{equation}
    \rho_t = \bigl((1-t)\id + t\,T\bigr)_\#\mu_0, \quad t \in [0,1]
\end{equation}
where $T = \nabla\phi$ is the optimal transport map from $\mu_0$ to $\mu_1$.

Intuition: each mass particle travels in a straight line from its initial position $x$ to its target position $T(x)$, and at time $t$ it is located at $(1-t)x + tT(x)$.

\medskip
\textbf{Summary of progress so far:} We have established the geometric structure of probability space $(\Prob_2,W_2)$---it is a Riemannian manifold whose geodesics are given by the Benamou--Brenier formula. \textbf{Next step}: we will derive a physical equation (Fokker--Planck), and then in Section~5 prove that it is precisely the \textbf{gradient flow} in this geometry.

\section{The Fokker--Planck Equation}

\textbf{Where we are:} The first three sections established the geometric framework---the Wasserstein distance (Section~1), the continuity equation (Section~2), and the Riemannian structure (Section~3). But these are all descriptions of the ``space'' itself, without a concrete physical equation. This section derives the Fokker--Planck equation: a PDE describing ``how the density evolves when particles undergo Brownian motion in a potential field''. The crucial final step is to rewrite it in the form of a continuity equation $\partial_t\rho + \nabla\cdot(\rho v)=0$, thereby identifying the velocity field $v$. The next section (Section~5) will prove that this $v$ is precisely the negative Wasserstein gradient of the free energy.

\subsection{Overview}

We begin with a microscopic stochastic process---particles undergoing Brownian motion in a potential field---and ask how the macroscopic \textbf{probability density} of the particles evolves. The answer is the Fokker--Planck equation,
a partial differential equation governing the density over time.

The goals of this section:
\begin{enumerate}
    \item Explain the physical meaning of SDEs (stochastic differential equations).
    \item Derive the Fokker--Planck equation from SDEs (with detailed explanation of each step).
    \item Rewrite the Fokker--Planck equation in the form of a continuity equation (preparing the connection to gradient flows).
\end{enumerate}

\subsection{Motivation: the physical setup}

\textbf{Setting:} Imagine a pollen grain floating in water (Brownian motion). At the same time, the bottom of the pool has a bowl shape---the pollen tends toward the bottom due to gravity. The particle's motion is driven by two forces:
\begin{enumerate}
    \item \textbf{Deterministic force} $-\nabla V(x)$: the negative gradient of the potential $V(x)$, pushing the particle toward the location of lowest potential energy (the bottom of the bowl)
    \item \textbf{Random force} (noise): random collisions from water molecules, causing the particle to move irregularly
\end{enumerate}

\begin{definition}[The SDE]
The particle's position $X_t\in\R^d$ satisfies the following stochastic differential equation (SDE):
\begin{equation}\label{eq:sde}
    dX_t = \underbrace{-\nabla V(X_t)\,dt}_{\text{drift: deterministic force}} + \underbrace{\sqrt{2}\,dB_t}_{\text{diffusion: random noise}}
\end{equation}
where $B_t$ is standard $d$-dimensional Brownian motion.
\end{definition}

\begin{intuition}
\textbf{Term-by-term explanation of the SDE:}

\textbf{$dX_t$}: the displacement of the particle during the time interval $[t, t+dt]$ (how far it moves in an infinitesimal time step).

\textbf{$-\nabla V(X_t)\,dt$} (drift term):
\begin{itemize}
    \item $V(x)$ is the potential energy function---for example, $V(x) = \frac{1}{2}|x|^2$ represents a bowl-shaped potential (low at center, high at edges)
    \item $-\nabla V$ is the direction of steepest descent of the potential (like a ball rolling down the bowl wall toward the bottom)
    \item The factor $dt$ arises because this is displacement produced by a deterministic velocity: displacement $=$ velocity $\times$ time
\end{itemize}

\textbf{$\sqrt{2}\,dB_t$} (diffusion term):
\begin{itemize}
    \item $dB_t$ is the increment of Brownian motion---a random vector with zero mean and variance $dt$
    \item Intuition: at each time step $dt$, the particle receives a random kick in a random direction with magnitude $\sim\sqrt{dt}$
    \item $\sqrt{2}$ is the diffusion strength (temperature)---the larger the coefficient, the stronger the randomness
\end{itemize}

\textbf{Competition between the two terms:}
\begin{itemize}
    \item The drift term seeks to ``concentrate'' particles near the minimum of $V$
    \item The diffusion term seeks to ``spread'' particles everywhere
    \item The eventual balance (stationary distribution) is determined by the ratio of their strengths
\end{itemize}
\end{intuition}

\textbf{Central question:} If we simultaneously release a \textbf{large number} of particles (each independently following the SDE above), with initial distribution $\rho_0(x)$, what PDE does the probability density $\rho_t(x)$ at time $t$ satisfy?

This is precisely the question that the Fokker-Planck equation answers---going from ``individual particle's random trajectory'' to ``macroscopic density evolution of many particles.''

\subsection{Prerequisite: It\^o's formula}

Deriving the Fokker-Planck equation requires one tool: It\^o's formula (It\^o, 1951).
It is the ``chain rule'' of stochastic calculus---but with one extra term compared to the ordinary chain rule.

\begin{theorem}[It\^o's formula --- informal version]
Let $X_t$ satisfy $dX_t = b\,dt + \sigma\,dB_t$, and let $f$ be a twice-differentiable function. Then:
\begin{equation}\label{eq:ito_formula}
df(X_t) = \underbrace{\nabla f(X_t)\cdot dX_t}_{\text{ordinary chain rule}} + \underbrace{\frac{1}{2}\sum_{i,j}\frac{\partial^2 f}{\partial x_i\partial x_j}(X_t)\cdot(\sigma\sigma^T)_{ij}\,dt}_{\text{It\^o correction term}}
\end{equation}
\end{theorem}

\textbf{Why is there an extra term?}

In ordinary calculus, $df = f'(x)\,dx$, and we ignore the $(dx)^2$ term (because $(dx)^2$ is of higher order than $dx$).

But in stochastic calculus, the ``size'' of the Brownian motion increment $dB_t$ is $\sqrt{dt}$ rather than $dt$.
This means $(dB_t)^2 \approx dt$ is \emph{not} a higher-order infinitesimal; it is of the same order as $dt$.
Therefore, the second-order term in the Taylor expansion cannot be discarded and must be retained. This is the It\^o correction.

\textbf{Specifically:}
\begin{itemize}
    \item Ordinary calculus: $(dx)^2 = 0$ (higher-order infinitesimal, discarded)
    \item Stochastic calculus: $(dB_t)^2 = dt$ (same order as $dt$, \textbf{cannot} be discarded)
    \item Cross term: $dB_t\cdot dt = 0$ ($\sqrt{dt}\cdot dt = dt^{3/2}$, higher-order, discarded)
\end{itemize}

\subsection{Derivation of Fokker--Planck: complete details}

We now derive the equation step by step. The strategy is:
\begin{enumerate}
    \item For an arbitrary ``test function'' $\varphi$, compute the rate of change of $\E[\varphi(X_t)]$
    \item Using $\E[\varphi(X_t)] = \int\varphi\,\rho_t\,dx$, translate the result into an equation for $\rho_t$
\end{enumerate}

\medskip
\textbf{Step 1: Apply It\^o's formula.}

Take an arbitrary smooth function $\varphi\in C_c^\infty(\R^d)$ (a ``probe''---we will use it to ``probe'' changes in the density).

For our SDE \eqref{eq:sde}, $b = -\nabla V$, $\sigma = \sqrt{2}\,I$, so $\sigma\sigma^T = 2I$.
Applying It\^o's formula:
\begin{align}
    d\varphi(X_t) &= \nabla\varphi(X_t)\cdot dX_t + \frac{1}{2}\sum_{i,j}\frac{\partial^2\varphi}{\partial x_i\partial x_j}(X_t)\cdot 2\delta_{ij}\,dt \notag\\
    &= \nabla\varphi(X_t)\cdot dX_t + \Delta\varphi(X_t)\,dt \label{eq:ito1}
\end{align}

\textbf{Explanation of each term:}
\begin{itemize}
    \item $\nabla\varphi\cdot dX_t$: the ``ordinary'' chain rule part. The change in $\varphi \approx$ gradient $\times$ displacement.
    \item $\Delta\varphi\,dt$: the It\^o correction. $\Delta\varphi = \sum_i\frac{\partial^2\varphi}{\partial x_i^2}$ is the Laplacian (the sum of second derivatives of $\varphi$ in all directions). It arises from $(dB_t)^2 = dt$.
    \item $2\delta_{ij}$ comes from $(\sigma\sigma^T)_{ij} = (\sqrt{2}I)(\sqrt{2}I)^T_{ij} = 2\delta_{ij}$.
    \item $\frac{1}{2}\times 2 = 1$, so the correction term simplifies to $\Delta\varphi\,dt$.
\end{itemize}

\medskip
\textbf{Step 2: Expand $\nabla\varphi\cdot dX_t$.}

Substituting $dX_t = -\nabla V\,dt + \sqrt{2}\,dB_t$:
\begin{align}
    d\varphi(X_t) &= \nabla\varphi(X_t)\cdot\bigl(-\nabla V(X_t)\,dt + \sqrt{2}\,dB_t\bigr) + \Delta\varphi(X_t)\,dt \notag\\
    &= \bigl[-\nabla\varphi\cdot\nabla V + \Delta\varphi\bigr](X_t)\,dt + \sqrt{2}\,\nabla\varphi(X_t)\cdot dB_t \label{eq:ito2}
\end{align}

Rearranging: $d\varphi(X_t) = $ (deterministic part) $dt$ + (stochastic part) $dB_t$.

\medskip
\textbf{Step 3: Take expectations.}

Take expectations of both sides. The key fact: \textbf{the expectation of an It\^o integral is zero}.

Why is $\E[\int\nabla\varphi\cdot dB_t] = 0$?
Because the It\^o integral is a \textbf{martingale}---intuitively, $dB_t$ has zero expectation (random noise does not favor any direction on average), so its contribution to the expected value of $\varphi$ is \textbf{zero on average}.

Therefore:
\begin{equation}\label{eq:expect}
    \frac{d}{dt}\E[\varphi(X_t)] = \E\bigl[-\nabla\varphi(X_t)\cdot\nabla V(X_t) + \Delta\varphi(X_t)\bigr]
\end{equation}

\textbf{Interpretation:} The rate of change of $\E[\varphi(X_t)]$ = ``contribution of deterministic drift to $\varphi$'' + ``contribution of diffusion to $\varphi$.'' The noise term vanishes under expectation.

\medskip
\textbf{Step 4: Rewrite using density $\rho_t$.}

If $X_t$ has probability density $\rho_t(x)$, then $\E[g(X_t)] = \int g(x)\rho_t(x)\,dx$ (expectation = integral against the density).

Substituting:
\begin{equation}\label{eq:density_form}
    \frac{d}{dt}\int_{\R^d}\varphi(x)\,\rho_t(x)\,dx = \int_{\R^d}\bigl[-\nabla\varphi(x)\cdot\nabla V(x) + \Delta\varphi(x)\bigr]\rho_t(x)\,dx
\end{equation}

The left side $= \int\varphi\,\partial_t\rho_t\,dx$ (since $\varphi$ does not depend on $t$, the derivative acts on $\rho_t$).

Now the problem becomes: the right side contains $\nabla\varphi$ and $\Delta\varphi$---but the equation we want should be about $\rho_t$, not about $\varphi$.
The solution: \textbf{integration by parts}, transferring the derivatives from $\varphi$ onto $\rho_t$.

\medskip
\textbf{Step 5: Integration by parts --- the key step.}

We use integration by parts (see Appendix~\ref{app:analysis} for details) to transfer the derivatives from $\varphi$ onto $\rho_t$. The key formulas are:
\[
\int\nabla\varphi\cdot\mathbf{F}\,dx = -\int\varphi\,(\nabla\cdot\mathbf{F})\,dx, \qquad
\int(\Delta\varphi)\,g\,dx = \int\varphi\,(\Delta g)\,dx
\]

\textbf{Term 1:} $-\int\nabla\varphi\cdot\nabla V\,\rho_t\,dx$

Treating $\rho_t\nabla V$ as the vector field $\mathbf{F} = \rho_t\nabla V$ and applying integration by parts:
\begin{align}
    -\int\nabla\varphi\cdot(\rho_t\nabla V)\,dx &= +\int\varphi\,\nabla\cdot(\rho_t\nabla V)\,dx \label{eq:ibp1}
\end{align}

\textbf{Term 2:} $\int\Delta\varphi\,\rho_t\,dx$

This requires two integrations by parts. The first:
\begin{align}
    \int(\Delta\varphi)\,\rho_t\,dx &= \int\nabla\cdot(\nabla\varphi)\,\rho_t\,dx = -\int\nabla\varphi\cdot\nabla\rho_t\,dx \label{eq:ibp2a}
\end{align}
The second (applying integration by parts once more to $-\int\nabla\varphi\cdot\nabla\rho_t\,dx$):
\begin{align}
    -\int\nabla\varphi\cdot\nabla\rho_t\,dx &= +\int\varphi\,\nabla\cdot(\nabla\rho_t)\,dx = +\int\varphi\,\Delta\rho_t\,dx \label{eq:ibp2}
\end{align}

\textbf{Summary:} Two integrations by parts convert $\Delta\varphi$ into $\Delta\rho_t$---``the second-order derivative has been transferred from $\varphi$ to $\rho_t$''.

\medskip
\textbf{Step 6: Combine and conclude.}

Substituting the results of Step 5 back into Step 4:
\[
    \int\varphi\,\partial_t\rho_t\,dx = \int\varphi\,\bigl[\nabla\cdot(\rho_t\nabla V) + \Delta\rho_t\bigr]\,dx
\]

Since $\varphi$ is an \textbf{arbitrary} smooth function, if $\int\varphi\cdot(A-B)\,dx = 0$ holds for all $\varphi$, then $A = B$
(the ``fundamental lemma'' / du Bois-Reymond lemma). Therefore:

\begin{equation}\label{eq:FP}
    \boxed{\partial_t\rho_t = \nabla\cdot(\rho_t\nabla V) + \Delta\rho_t}
\end{equation}

This is the \textbf{Fokker--Planck equation} (also known as the Kolmogorov forward equation).

\begin{intuition}
\textbf{What do the two terms in the Fokker--Planck equation mean?}

$\partial_t\rho = \underbrace{\nabla\cdot(\rho\nabla V)}_{\text{(a) drift term}} + \underbrace{\Delta\rho}_{\text{(b) diffusion term}}$

\textbf{(a) $\nabla\cdot(\rho\nabla V)$}: the effect driven by the external potential. Let us unpack this layer by layer:

\textbf{Layer 1: What is $\nabla V$?}
\begin{itemize}
    \item $V(x)$ is the potential energy function. $\nabla V(x)$ is a vector pointing in the direction of \textbf{steepest increase} of $V$ at $x$.
    \item For example, if $V(x) = \frac{1}{2}|x|^2$ (bowl-shaped), then $\nabla V(x) = x$, pointing from the bottom toward the rim (outward).
\end{itemize}

\textbf{Layer 2: What is $\rho\nabla V$?}
\begin{itemize}
    \item This is a vector field (with direction and magnitude at each point), representing the ``mass flux''.
    \item Imagine each particle being pushed by $\nabla V$. Then $\rho(x)\nabla V(x)$ = at position $x$, \textbf{mass per unit volume} $\times$ \textbf{the direction and speed each mass element is pushed}.
    \item It is density times velocity, whose physical meaning is ``mass passing through a unit area per unit time''---the same as $\rho v$ in the continuity equation.
    \item Note the direction: $\rho\nabla V$ points toward \textbf{increasing} potential energy (``uphill''). So in the Fokker--Planck equation $\partial_t\rho = +\nabla\cdot(\rho\nabla V)$,
    density is actually ``flowing away'' from regions of high potential---we explain why below.
\end{itemize}

\textbf{Layer 3: What is $\nabla\cdot(\rho\nabla V)$?}
\begin{itemize}
    \item The divergence $\nabla\cdot\mathbf{F}$ measures the ``degree of outward flux'' of a vector field $\mathbf{F}$ at a point:
    \begin{itemize}
        \item $\nabla\cdot\mathbf{F}(x) > 0$: $x$ is a ``source''---flux diverges outward from $x$, matter is leaving
        \item $\nabla\cdot\mathbf{F}(x) < 0$: $x$ is a ``sink''---flux converges toward $x$, matter is accumulating
    \end{itemize}
    \item So $\nabla\cdot(\rho\nabla V)(x) > 0$ means: at $x$, the flux $\rho\nabla V$ is \textbf{diverging outward}.
\end{itemize}

\textbf{Layer 4: Why does $\partial_t\rho = +\nabla\cdot(\rho\nabla V)$ imply that density accumulates at low potential?}

\textbf{Reading the sign:} At first glance, the positive sign in $+\nabla\cdot(\rho\nabla V)$ seems to suggest density is ``moving uphill''.
But in fact this positive sign corresponds precisely to ``particles moving downhill''. Here is why:

The continuity equation reads $\partial_t\rho + \nabla\cdot(\rho\,v) = 0$, and the particle velocity is $v = -\nabla V$ (in the direction of decreasing potential). Substituting:
\[
\partial_t\rho = -\nabla\cdot(\rho\,v) = -\nabla\cdot\bigl(\rho\cdot(-\nabla V)\bigr) = +\nabla\cdot(\rho\nabla V)
\]
So the physical meaning of $+\nabla\cdot(\rho\nabla V)$ is ``the density change caused by particles moving downhill with velocity $-\nabla V$''.

\medskip
\textbf{One-dimensional verification:} Let $V(x) = \frac{1}{2}x^2$, $\nabla V = x$, $\Delta V = 1$. Expanding the divergence (product rule):
\[
\nabla\cdot(\rho\nabla V) = \nabla\rho\cdot\nabla V + \rho\,\Delta V = x\rho' + \rho
\]
Analyzing the behavior at different locations:
\begin{itemize}
    \item At $x=0$ (bottom of bowl): $\nabla\cdot(\rho\nabla V) = \rho(0) > 0$, so $\partial_t\rho(0) > 0$---density at the bottom is \textbf{increasing}
    \item At large $x$ (bowl rim): $\rho'<0$ and $|x|$ is large, so $x\rho'$ is a large negative number, giving $\partial_t\rho < 0$ overall---density at the rim is \textbf{decreasing}
\end{itemize}
Conclusion: density flows from the rim to the bottom---particles are indeed \textbf{accumulating} at the low-potential region, consistent with the intuition from $v=-\nabla V$.

\textbf{(b) $\Delta\rho$}: diffusion (the ``heat equation'' effect).
\begin{itemize}
    \item $\Delta\rho = \nabla^2\rho$ is the Laplacian of the density
    \item Sign of the Laplacian: if $\rho(x)$ is \textbf{larger} than the average of its neighbors (a density peak), then $\Delta\rho < 0$ and the density \textbf{decreases}
    \item Conversely, if $\rho(x)$ is \textbf{smaller} than its neighbors (a density trough), then $\Delta\rho > 0$ and the density \textbf{increases}
    \item The effect: smoothing out density inhomogeneities---high-density regions spread out, low-density regions fill in (just like a drop of ink naturally diffusing in water)
\end{itemize}

\textbf{Competition between the two terms:}
\begin{itemize}
    \item The drift term tries to concentrate particles at the minimum of $V$ (density becomes sharper)
    \item The diffusion term tries to spread particles uniformly (density becomes flatter)
    \item The final equilibrium = stationary distribution $\rho_\infty \propto e^{-V}$
\end{itemize}
\end{intuition}

\subsection{Rewriting as a continuity equation}

This step is the key link between Fokker--Planck and Wasserstein gradient flows. We rewrite the Fokker--Planck equation in the form
$\partial_t\rho + \nabla\cdot(\rho v) = 0$, thereby identifying the ``velocity field $v$''.
\medskip
\textbf{Step 1:} Note that $\Delta\rho = \nabla\cdot(\nabla\rho)$ (Laplacian = divergence of gradient):
\begin{align}
    \partial_t\rho_t &= \nabla\cdot(\rho_t\nabla V) + \nabla\cdot(\nabla\rho_t) \notag\\
    &= \nabla\cdot\bigl(\rho_t\nabla V + \nabla\rho_t\bigr) \label{eq:FP_div}
\end{align}

(Combining both terms into a single divergence---since ``$\nabla\cdot A + \nabla\cdot B = \nabla\cdot(A+B)$''.)
\textbf{Step 2:} Use the key identity $\nabla\rho = \rho\,\nabla\log\rho$:

\textbf{Deriving this identity:} By the chain rule, $\nabla\log\rho = \frac{1}{\rho}\nabla\rho$, so $\rho\,\nabla\log\rho = \nabla\rho$. \checkmark

Substituting:
\begin{align}
    \partial_t\rho_t &= \nabla\cdot\bigl(\rho_t\nabla V + \rho_t\nabla\log\rho_t\bigr) \notag\\
    &= \nabla\cdot\bigl(\rho_t\,\underbrace{(\nabla V + \nabla\log\rho_t)}_{= -v_t}\bigr) \label{eq:FP_continuity}
\end{align}

\textbf{Step 3:} Compare with the continuity equation $\partial_t\rho + \nabla\cdot(\rho\,v) = 0$, i.e., $\partial_t\rho = -\nabla\cdot(\rho\,v)$:
\begin{equation}\label{eq:FP_velocity}
    \boxed{v_t = -\nabla V - \nabla\log\rho_t}
\end{equation}

\begin{intuition}
\textbf{The two terms in the velocity field---the physics of two ``forces'':}

\textbf{(1) $-\nabla V$: external force (deterministic)}
\begin{itemize}
    \item This is the ``push'' from the potential---driving particles toward the minimum of $V$
    \item Analogy: a ball rolling in a bowl, always pushed toward the bottom
    \item This force \textbf{does not depend on the density} $\rho$---it is the same regardless of how many other particles are nearby
\end{itemize}

\textbf{(2) $-\nabla\log\rho_t = -\frac{\nabla\rho_t}{\rho_t}$: entropic/diffusive force (generated by the density itself)}
\begin{itemize}
    \item This is a force ``from crowded to sparse''---$\nabla\rho$ points toward increasing density; the negative sign reverses it to point toward \textbf{decreasing} density
    \item Division by $\rho$ is normalization: it is the \textbf{relative} density difference, not the absolute one, that drives diffusion
    \item Analogy: in a subway car, people naturally ``flow'' from crowded cars to empty ones
    \item This force \textbf{depends on the density itself}---the more non-uniform the density, the stronger the force
    \item Physically called ``osmotic pressure'' or ``entropic force''---no one is actually pushing; it is the macroscopic manifestation of statistical effects
\end{itemize}

\textbf{Key differences:}

\begin{center}
\renewcommand{\arraystretch}{1.3}
\begin{tabular}{@{}lll@{}}
\toprule
& $-\nabla V$ & $-\nabla\log\rho$ \\
\midrule
Source & External potential field & The density itself \\
Direction & toward minimum of $V$ & toward minimum of $\rho$ \\
Effect & concentration/aggregation & diffusion/dispersion \\
Depends on $\rho$? & no & yes \\
\end{tabular}
\end{center}
\end{intuition}

\subsection{Stationary distribution}

When the system reaches equilibrium ($\partial_t\rho = 0$), the density no longer changes. For our gradient flow, the velocity $v = -\nabla(\log\rho + V)$ is a gradient field, and $\nabla\cdot(\rho v)=0$ implies $v = 0$ (setting $\psi=\log\rho+V$ and integrating by parts gives $\int\rho|v|^2\,dx = \int\rho|\nabla\psi|^2\,dx = -\int\psi\,\nabla\cdot(\rho\nabla\psi)\,dx = \int\psi\,\nabla\cdot(\rho v)\,dx = 0$). So the equilibrium condition is equivalent to $v = 0$:
\begin{align}
    v = 0 &\implies \nabla V + \nabla\log\rho_\infty = 0 \notag\\
    &\implies \nabla\log\rho_\infty = -\nabla V \notag\\
    &\implies \log\rho_\infty = -V + C \quad\text{(integrate both sides)} \notag\\
    &\implies \rho_\infty = e^C\cdot e^{-V} \propto e^{-V} \label{eq:gibbs}
\end{align}

The normalization constant: $e^C = 1/Z$, where $Z = \int e^{-V(x)}\,dx$ (the partition function). Therefore:
\[
\rho_\infty(x) = \frac{e^{-V(x)}}{\int e^{-V(y)}\,dy}
\]

This is the \textbf{Gibbs--Boltzmann distribution}.

\begin{intuition}
\textbf{Intuition for the equilibrium:}
$\rho_\infty \propto e^{-V}$ means particles tend to accumulate at low-potential regions, but do not concentrate entirely at a single point (diffusion prevents this).

\textbf{Verification:} At $\rho_\infty \propto e^{-V}$,
the potential force $-\nabla V$ (toward the bottom) and the diffusive force $-\nabla\log\rho_\infty = +\nabla V$ (away from the bottom) exactly cancel, giving zero net velocity.

This is the \textbf{Gibbs--Boltzmann distribution}. For the role of temperature (concentration at low temperature vs.\ uniformity at high temperature) and its wide applications in statistical mechanics, machine learning, and Bayesian inference, see Appendix~\ref{app:gibbs}.
\end{intuition}

\subsection{Connection to modern generative models}

The Fokker--Planck velocity field $v_t = -\nabla V - \nabla\log\rho_t$ directly reveals the relationship between score and velocity.
For the more general SDE $dX_t = f(X_t,t)\,dt + g(t)\,dB_t$, the Fokker--Planck equation $\partial_t\rho = -\nabla\cdot(\rho f) + \frac{1}{2}g^2\Delta\rho$ can be rewritten as the continuity equation $\partial_t\rho + \nabla\cdot(\rho\,v_t) = 0$, where the probability flow ODE velocity field is:
\[
\boxed{v_t = \underbrace{f(x,t)}_{\text{drift}} - \frac{1}{2}g(t)^2\underbrace{\nabla\log\rho_t}_{\text{score}}}
\]
(Verification: for our setting $f=-\nabla V$, $g=\sqrt{2}$, we get $v_t = -\nabla V - \nabla\log\rho_t$~\checkmark.)
Here, the score $s_t := \nabla\log\rho_t$ is what diffusion models learn, while the velocity $v_t$ is what flow matching learns.
They differ only by a known drift term---they are two equivalent descriptions of the same probability evolution.
\begin{center}
\renewcommand{\arraystretch}{1.3}
\begin{tabular}{@{}ll@{}}
\toprule
\textbf{Language of this article} & \textbf{Language of generative models} \\
\midrule
Fokker--Planck equation & Density evolution of forward/reverse SDE \\
Wasserstein gradient flow & Probability flow ODE \\
JKO discretization & Discrete-time denoising steps (each step of DDPM) \\
Continuity equation & Fundamental constraint of Flow Matching \\
\bottomrule
\end{tabular}
\end{center}

For the geometric meaning of the score ($=$ the Wasserstein gradient of negative entropy), detailed formulas, and comparisons among major models (DDPM, NCSN, VE-SDE, VP-SDE, Flow Matching), see Appendix~\ref{app:genmodels}.
\medskip

\subsection{Conceptual clarification: what needs an energy, and what does not?}

The continuity equation is more general than the Fokker--Planck equation. It is a \textbf{kinematic} statement:
\[
\partial_t\rho_t+\nabla\cdot(\rho_t v_t)=0.
\]
It only says: if probability mass moves with velocity field $v_t$, then the density changes by mass conservation. No potential $V$, no free energy $\mathcal{F}$, and no entropy term is needed to write this equation.

However, the continuity equation does not decide the motion by itself. One must still provide the velocity field $v_t$. Different choices of $v_t$ produce different dynamics:
\[
\text{continuity equation} + \text{chosen velocity field}
\quad\Longrightarrow\quad
\text{a probability path}.
\]

The Fokker--Planck equation is a special case where the velocity field is not arbitrary. It is generated by a free energy:
\[
v_t
=
-\nabla\frac{\delta\mathcal{F}}{\delta\rho}(\rho_t).
\]
For
\[
\mathcal{F}(\rho)=\int V\rho\,dx+\int\rho\log\rho\,dx,
\]
we get
\[
v_t=-\nabla V-\nabla\log\rho_t,
\]
and substituting this into the continuity equation gives exactly
\[
\partial_t\rho_t=\nabla\cdot(\rho_t\nabla V)+\Delta\rho_t.
\]

\begin{keypoint}
The continuity equation defines ``how mass can flow.'' The Fokker--Planck equation additionally specifies ``why it flows that way'': it follows the negative Wasserstein gradient of a free energy.
\end{keypoint}

\subsection{Where is the modeling choice in diffusion?}

The hand-crafted part of diffusion models is not the It\^o rule itself. Once Brownian motion is chosen, the identity $(dB_t)^2=dt$ is a mathematical consequence of its quadratic variation. The modeling choice is the SDE:
\[
dX_t=f(X_t,t)\,dt+g(t)\,dB_t.
\]
Choosing the drift $f$, the diffusion strength $g$, and the noise schedule (for example the $\beta(t)$ schedule in VP-SDE) specifies the forward noising process. The Fokker--Planck equation is then the deterministic density evolution implied by that SDE.

Flow Matching removes this SDE modeling layer. It does not first choose Brownian noise and then derive a Fokker--Planck equation. Instead, it directly learns a velocity field $v_t$ for a chosen probability path, often an OT-like path from noise to data.

\textbf{What we now know:} The Fokker-Planck equation describes the evolution of particle density and has the unique stationary distribution $\pi\propto e^{-V}$. \textbf{Central question:} Why must $\rho_t$ converge to $\pi$? What objective is it ``optimizing''? The answer is in the next section.
\section{Free Energy and the Wasserstein Gradient}

\textbf{Where we are:} Section~4 derived the Fokker-Planck equation and wrote it as $\partial_t\rho = -\nabla\cdot(\rho\,v)$ with $v = -\nabla\log\rho - \nabla V$. Section~3 showed that the space of probability distributions carries a Riemannian structure. This section brings the two together: we define the free energy $\mathcal{F}(\rho) = \KL(\rho\|\pi)$, compute its Wasserstein gradient $\mathrm{grad}_W\mathcal{F} = \nabla\log\rho + \nabla V$, and thereby prove $v = -\mathrm{grad}_W\mathcal{F}$---\textbf{the Fokker-Planck equation is the gradient flow of the free energy in probability space}. This is the central theorem of the entire article.

\subsection{The free energy functional}

We now address a central question: What is the Fokker-Planck equation ``optimizing''?

The answer: it is minimizing the \textbf{free energy}---the KL divergence between the distribution $\rho$ and the equilibrium $\pi \propto e^{-V}$.
But before writing down the formula, let us first build physical intuition for what free energy is.

\begin{intuition}
\textbf{Free energy as the balance of two competing tendencies}

A system in statistical mechanics (say, a confined gas) is governed by two opposing tendencies.

\textbf{Tendency one: lower the energy.}
\begin{itemize}
    \item Each particle is driven toward the location of lowest potential $V(x)$ (the bottom of the well)
    \item Acting alone, this collapses all mass onto the global minimum of $V$, forming $\delta_{x_{\min}}$
    \item It is the dominant tendency at low temperature, where the system settles toward its ground state
\end{itemize}

\textbf{Tendency two: raise the entropy.}
\begin{itemize}
    \item Thermal agitation drives the particles to spread out
    \item Acting alone, this produces a uniform distribution, independent of the potential
    \item It is the dominant tendency at high temperature
\end{itemize}

\textbf{Free energy is the quantitative trade-off between the two:}
\[
\text{Free energy} = \text{Potential energy (favors concentration)} - \text{Temperature}\times\text{Entropy (favors dispersion)}
\]
Minimizing free energy selects the optimal balance between concentration and dispersion.

\textbf{The role of temperature:}
\begin{itemize}
    \item Low temperature: the energy term dominates; mass concentrates near the minimum of $V$
    \item High temperature: the entropy term dominates; the distribution approaches uniform
    \item Temperature $=1$ (our setting): the two terms carry equal weight
\end{itemize}
\end{intuition}

\begin{definition}[Free energy]
Define $\pi(x) = \frac{1}{Z}e^{-V(x)}$ where $Z = \int e^{-V}dx$. The free energy is:
\begin{align}
    \mathcal{F}(\rho) &:= \KL(\rho\|\pi) = \int_{\R^d}\rho\log\frac{\rho}{\pi}\,dx \notag\\
    &= \int\rho\log\rho\,dx + \int V\rho\,dx + \log Z \label{eq:free_energy}
\end{align}
\end{definition}

\textbf{Term-by-term breakdown of the formula:}

\textbf{Step one: Why does $\KL(\rho\|\pi)$ expand as above?}
\begin{align*}
\KL(\rho\|\pi) &= \int\rho\log\frac{\rho}{\pi}\,dx = \int\rho\,(\log\rho - \log\pi)\,dx\\
&= \int\rho\log\rho\,dx - \int\rho\log\pi\,dx
\end{align*}
Since $\pi = \frac{1}{Z}e^{-V}$, we have $\log\pi = -V - \log Z$. Substituting:
\begin{align*}
&= \int\rho\log\rho\,dx - \int\rho(-V-\log Z)\,dx\\
&= \int\rho\log\rho\,dx + \int V\rho\,dx + \log Z\underbrace{\int\rho\,dx}_{=1}
\end{align*}

Ignoring the constant $\log Z$ (which doesn't depend on $\rho$), the free energy decomposes as:
\begin{equation}
    \mathcal{F}(\rho) = \underbrace{\int\rho\log\rho\,dx}_{\text{negative entropy: }-S(\rho)} + \underbrace{\int V\,\rho\,dx}_{\text{potential energy: }\E_\rho[V]}
\end{equation}

More generally, if the diffusion strength is $\varepsilon>0$, the free energy is
\begin{equation}\label{eq:free_energy_epsilon}
    \mathcal{F}_\varepsilon(\rho)
    =
    \int V\rho\,dx
    +
    \varepsilon\int\rho\log\rho\,dx.
\end{equation}
Then
\[
\frac{\delta\mathcal{F}_\varepsilon}{\delta\rho}
=
V+\varepsilon(\log\rho+1),
\qquad
v
=
-\nabla\frac{\delta\mathcal{F}_\varepsilon}{\delta\rho}
=
-\nabla V-\varepsilon\nabla\log\rho.
\]
This is the clean way to ``add the diffusion force'': we do \emph{not} add it into the external potential $V(x)$. Instead, we add the negative entropy term $\varepsilon\int\rho\log\rho$ to the free energy. The resulting force $-\varepsilon\nabla\log\rho$ depends on the current density itself, so it is a collective entropic force rather than a fixed external potential.

\begin{intuition}
\textbf{What does each term mean?}

\textbf{(1) $\int\rho\log\rho\,dx$: negative entropy}

\begin{itemize}
    \item The entropy $S(\rho) = -\int\rho\log\rho\,dx$ measures the ``degree of disorder / spread'' of the distribution
    \item The more uniform the distribution, the higher the entropy (maximum entropy = uniform distribution)
    \item The more concentrated the distribution (e.g., $\rho\approx\delta_x$), the lower the entropy (extreme case: Dirac entropy $=-\infty$)
    \item The free energy contains $+\int\rho\log\rho = -S$ (negative entropy), so minimizing free energy is equivalent to \textbf{maximizing entropy} (subject to the energy constraint)
    \item Intuition: this term ``penalizes concentration'' and encourages the distribution to spread out
\end{itemize}

\textbf{(2) $\int V\rho\,dx = \E_\rho[V]$: mean potential energy}

\begin{itemize}
    \item This is the expected value of the potential $V(x)$ under the distribution $\rho$
    \item If $\rho$ concentrates where $V$ is low, $\E_\rho[V]$ is small
    \item If $\rho$ has mass where $V$ is high, $\E_\rho[V]$ is large
    \item Minimizing this term $\Rightarrow$ making $\rho$ concentrate near the minimum of $V$
    \item Intuition: this term ``rewards concentration'' in low-energy regions
\end{itemize}

\textbf{The tension between the two terms:}
\begin{center}
\renewcommand{\arraystretch}{1.3}
\begin{tabular}{@{}lll@{}}
\toprule
& \textbf{Neg.\ entropy $\int\rho\log\rho$} & \textbf{Potential energy $\int V\rho$} \\
\midrule
What it wants & $\rho$ as uniform as possible (spread) & $\rho$ concentrated at the minimum of $V$ \\
Extreme case & Uniform distribution (minimized) & Dirac at $\arg\min V$ (minimized) \\
Physical driver & Thermal motion / diffusion & External force / potential \\
\bottomrule
\end{tabular}
\end{center}

\textbf{The minimum of the free energy} is the optimal compromise between the two---which turns out to be the Gibbs distribution $\rho^* = \frac{1}{Z}e^{-V}$.
\end{intuition}

\textbf{Verification: Why is the minimum of $\mathcal{F}$ attained at $\rho^* = \pi$?}

Because $\mathcal{F}(\rho) = \KL(\rho\|\pi)$, and the KL divergence satisfies:
\begin{itemize}
    \item $\KL(\rho\|\pi) \geq 0$ (always non-negative---Gibbs' inequality)
    \item $\KL(\rho\|\pi) = 0 \iff \rho = \pi$ (equals zero if and only if the two distributions are identical)
\end{itemize}
So the global minimum of $\mathcal{F}$ is $0$, attained at $\rho = \pi = \frac{1}{Z}e^{-V}$.

\textbf{In one sentence:} The Fokker-Planck equation describes the process of the system continuously decreasing $\KL(\rho_t\|\pi)$---starting from any initial distribution $\rho_0$, $\rho_t$ gradually approaches $\pi$, with the free energy monotonically decreasing to zero.

\textbf{What is the justification for this claim? We give a rigorous proof below.}

\textbf{Theorem:} Let $\rho_t$ be a solution of the Fokker-Planck equation and $\pi \propto e^{-V}$. Then:
\[
\frac{d}{dt}\KL(\rho_t\|\pi) = -\int\rho_t\left|\nabla\log\frac{\rho_t}{\pi}\right|^2\,dx \leq 0
\]
That is, the free energy is \textbf{strictly monotonically decreasing} (unless $\rho_t = \pi$).

\textbf{Complete derivation:}

\textbf{Step 1: Write out $\frac{d}{dt}\KL(\rho_t\|\pi)$.}

$\pi$ does not depend on time, so:
\begin{align}
\frac{d}{dt}\KL(\rho_t\|\pi) &= \frac{d}{dt}\int\rho_t\log\frac{\rho_t}{\pi}\,dx \notag\\
&= \int\partial_t\rho_t\cdot\log\frac{\rho_t}{\pi}\,dx + \int\rho_t\cdot\frac{\partial_t\rho_t}{\rho_t}\,dx \notag\\
&= \int\partial_t\rho_t\cdot\log\frac{\rho_t}{\pi}\,dx + \int\partial_t\rho_t\,dx \notag\\
&= \int\partial_t\rho_t\cdot\log\frac{\rho_t}{\pi}\,dx + 0 \label{eq:dKLdt_step1}
\end{align}

(The last step uses $\int\partial_t\rho_t\,dx = \frac{d}{dt}\int\rho_t\,dx = \frac{d}{dt}1 = 0$, since total mass is conserved.)

\textbf{Step 2: Substitute the Fokker--Planck equation.}

The Fokker-Planck equation can be written as $\partial_t\rho_t = \nabla\cdot\bigl(\rho_t\nabla\log\frac{\rho_t}{\pi}\bigr)$.

Why? Because $\log\frac{\rho}{\pi} = \log\rho - \log\pi = \log\rho + V + \log Z$, so $\nabla\log\frac{\rho}{\pi} = \nabla\log\rho + \nabla V$.
And we previously derived $\partial_t\rho = \nabla\cdot(\rho(\nabla\log\rho + \nabla V))$---which is precisely $\nabla\cdot(\rho\nabla\log\frac{\rho}{\pi})$.

Substituting into Step 1:
\begin{equation}
\frac{d}{dt}\KL(\rho_t\|\pi) = \int\nabla\cdot\!\left(\rho_t\nabla\log\frac{\rho_t}{\pi}\right)\cdot\log\frac{\rho_t}{\pi}\,dx \label{eq:dKLdt_step2}
\end{equation}

\textbf{Step 3: Integration by parts.}

Apply integration by parts to \eqref{eq:dKLdt_step2}. Let $\mathbf{F} = \rho_t\nabla\log\frac{\rho_t}{\pi}$ and $g = \log\frac{\rho_t}{\pi}$:
\[
\int(\nabla\cdot\mathbf{F})\,g\,dx = -\int\mathbf{F}\cdot\nabla g\,dx
\]
(The boundary terms vanish.) Substituting $\mathbf{F}$ and $\nabla g$:
\begin{align}
\frac{d}{dt}\KL(\rho_t\|\pi) &= -\int\rho_t\nabla\log\frac{\rho_t}{\pi}\cdot\nabla\log\frac{\rho_t}{\pi}\,dx \notag\\
&= -\int\rho_t\left|\nabla\log\frac{\rho_t}{\pi}\right|^2\,dx \label{eq:dKLdt_final}
\end{align}

\textbf{Step 4: Conclude.}

The right side $= -\int\rho_t|\nabla\log\frac{\rho_t}{\pi}|^2\,dx$. Since $\rho_t \geq 0$ and $|\cdot|^2\geq 0$, the entire integral is $\geq 0$, so with the negative sign it is $\leq 0$.

\begin{equation}
\boxed{\frac{d}{dt}\KL(\rho_t\|\pi) = -\int\rho_t\left|\nabla\log\frac{\rho_t}{\pi}\right|^2\,dx \leq 0}
\end{equation}

Equality holds if and only if $\nabla\log\frac{\rho_t}{\pi} = 0$ everywhere, i.e., $\rho_t = c\cdot\pi$.
By the normalization condition $\int\rho_t = 1 = \int\pi$, we get $c = 1$, i.e., $\rho_t = \pi$.

\textbf{Summary:}
\begin{itemize}
    \item $\KL(\rho_t\|\pi)$ is \textbf{monotonically decreasing} along the Fokker-Planck equation
    \item The rate of decrease $= \int\rho_t|\nabla\log\frac{\rho_t}{\pi}|^2\,dx$, a quantity called the \textbf{relative Fisher information}
    \item The rate of decrease is zero only when $\rho_t = \pi$---once equilibrium is reached, there is no further change
    \item Physical intuition: the system cannot spontaneously move away from equilibrium (a mathematical version of the second law of thermodynamics)
\end{itemize}

For a complete introduction to Fisher information (the parametric version in statistics, the distributional version in information theory, how they are unified, and the definition of relative Fisher information and information inequalities), see Appendix~\ref{app:fisher}.

\begin{intuition}
\textbf{Why is the result a negative ``squared term''? This is no coincidence.}

This structure is ubiquitous in mathematics. Consider a gradient flow in Euclidean space, $\dot x = -\nabla f(x)$:
\[
\frac{d}{dt}f(x_t) = \nabla f(x_t)\cdot\dot x_t = \nabla f\cdot(-\nabla f) = -|\nabla f|^2 \leq 0
\]
The rate of decrease of $f$ = the squared norm of the gradient (always non-positive).

Fokker-Planck is a perfect analogue: it is the gradient flow of $\mathcal{F} = \KL(\cdot\|\pi)$ in Wasserstein space.
The ``Wasserstein gradient'' $= \nabla\log\frac{\rho}{\pi}$, so the rate of decrease $= -\|\text{gradient}\|_\rho^2 = -\int\rho|\nabla\log\frac{\rho}{\pi}|^2\,dx$.

The structure is identical, lifted from finite to infinite dimensions.
\end{intuition}

\textbf{A concrete example:}

Let $V(x) = \frac{1}{2}x^2$ (one-dimensional quadratic potential), so $\pi = \mathcal{N}(0,1)$.

\begin{itemize}
    \item Initial $\rho_0 = \delta_5$ (all particles at $x=5$):
    \begin{itemize}
        \item Potential energy $\int V\rho_0 = V(5) = 12.5$ (very high---the particles are on the bowl rim)
        \item Negative entropy $= -\infty$ (completely concentrated at a single point)
        \item Fokker-Planck will cause $\rho_t$ to slide toward the bottom while simultaneously spreading out, eventually becoming $\mathcal{N}(0,1)$
    \end{itemize}
    \item Initial $\rho_0 = \text{Uniform}[-10,10]$ (uniformly distributed over a large range):
    \begin{itemize}
        \item Potential energy $\int V\rho_0 = \frac{1}{20}\int_{-10}^{10}\frac{x^2}{2}dx = \frac{100}{6}\approx 16.7$ (very high---too much mass on the bowl rim)
        \item Negative entropy $= \log(20)\approx 3$ (very spread out, high entropy, low negative entropy)
        \item Fokker-Planck will cause the mass at the edges to contract toward the center, eventually also becoming $\mathcal{N}(0,1)$
    \end{itemize}
    \item Regardless of the starting $\rho_0$, the endpoint is always $\pi = \mathcal{N}(0,1)$---the unique minimizer of the free energy.
\end{itemize}

\textbf{Free energy in the real world}---milk diffusing in coffee ($V\approx 0$, purely entropy-driven); atmospheric pressure $\rho(h)\propto e^{-mgh/kT}$ (gravity vs.\ thermal motion); protein folding (conformational energy vs.\ conformational entropy); in Bayesian inference, SGLD (Welling \& Teh, 2011) is a discrete version of Fokker-Planck.

\begin{keypoint}
\textbf{Free energy vs.\ ordinary potential: the essential distinction}

\begin{center}
\renewcommand{\arraystretch}{1.3}
\begin{tabular}{@{}p{3.5cm}p{5cm}p{5.5cm}@{}}
\toprule
& \textbf{Potential $V(x)$} & \textbf{Free energy $\mathcal{F}(\rho)$} \\
\midrule
Input & A single point $x\in\R^d$ & An entire distribution $\rho:\R^d\to\R$ \\
Describes & Energy of one particle at $x$ & How far $\rho$ is from equilibrium \\
Gradient & $\nabla V(x)\in\R^d$ & $\grad_W\mathcal{F}(\rho)$ (a vector field) \\
Gradient flow & $\dot x = -\nabla V$ & $\partial_t\rho = -\nabla\cdot(\rho\,\grad_W\mathcal{F})$ \\
\bottomrule
\end{tabular}
\end{center}

Free energy $=$ mean potential energy $+$ negative entropy. The extra ``entropy'' term is a purely collective effect---a single particle has no entropy to speak of.
\end{keypoint}

\subsection{First variation (functional derivative)}

\textbf{Goal:} Compute $\frac{\delta\mathcal{F}}{\delta\rho}$---that is, how $\mathcal{F}$ responds when the density changes from $\rho$ to $\rho+\epsilon\,\delta\rho$.

\begin{align}
    \mathcal{F}(\rho + \epsilon\,\delta\rho) 
    &= \int(\rho+\epsilon\,\delta\rho)\log(\rho+\epsilon\,\delta\rho)\,dx + \int V(\rho+\epsilon\,\delta\rho)\,dx \notag\\
    &\approx \int\rho\log\rho\,dx + \epsilon\int\delta\rho\,(\log\rho + 1)\,dx + \int V\rho\,dx + \epsilon\int V\,\delta\rho\,dx \notag
\end{align}

(Expanding $(\rho+\epsilon\delta\rho)\log(\rho+\epsilon\delta\rho) \approx \rho\log\rho + \epsilon\delta\rho(\log\rho+1)$, which is a first-order Taylor expansion.)

Therefore:
\begin{equation}\label{eq:first_variation}
    \frac{\delta\mathcal{F}}{\delta\rho} = \log\rho + 1 + V
\end{equation}

\subsection{Computing the Wasserstein gradient}

\textbf{The core computation:} We now compute the gradient of $\mathcal{F}$ with respect to the Wasserstein metric.

\medskip
\textbf{Recall first: Why does the Euclidean gradient $\nabla f$ satisfy $\langle\nabla f, h\rangle = Df[h]$?}

This is in fact the \textbf{definition} of the gradient. Specifically:

For $f:\R^n\to\R$, the directional derivative (the Fr\'echet derivative acting on the direction $h$) is:
\[
Df(x)[h] = \lim_{\epsilon\to 0}\frac{f(x+\epsilon h)-f(x)}{\epsilon} = \sum_i\frac{\partial f}{\partial x_i}h_i
\]
The right side can be written as an inner product: $Df(x)[h] = \left(\frac{\partial f}{\partial x_1},\ldots,\frac{\partial f}{\partial x_n}\right)\cdot(h_1,\ldots,h_n) = \langle\nabla f, h\rangle$.

In other words: $Df(x)$ is a \textbf{linear functional} (takes a direction $h$ as input and outputs a real number),
and $\nabla f$ is the \textbf{vector} that ``represents'' this linear functional via the inner product---the Riesz representative.

\textbf{Why is an inner product needed to define the gradient?} Because $Df$ by itself is only a ``cotangent vector'' (linear functional);
to ``raise'' it to a ``tangent vector'' (vector), one needs a metric. Under the standard Euclidean metric $\langle\cdot,\cdot\rangle$,
$\nabla f$ is simply the vector of partial derivatives. But under a different metric $g$, the gradient changes---$\mathrm{grad}_g f = g^{-1}\nabla f$ (see the differential geometry appendix).

\textbf{Core idea:} Gradient $=$ using the metric to convert a ``differential'' (linear functional) into a ``direction'' (vector).
Different metric $\to$ different gradient $\to$ different gradient flow.

\medskip
In Wasserstein space, by complete analogy, $\grad_W\mathcal{F}(\rho)$ is a vector field $\xi$ satisfying:
\[
\langle\xi,\eta\rangle_\rho = D\mathcal{F}(\rho)[\dot\rho] \quad\text{for all tangent vectors}\;\eta
\]
where $\dot\rho = -\nabla\cdot(\rho\eta)$ is the density change induced by $\eta$.

\medskip
\textbf{Detailed explanation of each concept:}

\textbf{(i) What is a ``tangent vector'' in Wasserstein space?}

Recall: in finite dimensions, a tangent vector $v\in T_pM$ at a point $p$ on a manifold $M$ describes the ``instantaneous velocity starting from $p$.''

In Wasserstein space, a ``point'' is a probability density $\rho$. A ``motion'' starting from $\rho$ is a time-varying density $\rho_t$.
The continuity equation $\partial_t\rho_t + \nabla\cdot(\rho_t\,v_t) = 0$ tells us:
any change in density is driven by a \textbf{velocity field} $v_t$.

Therefore, a tangent vector at $\rho$ in Wasserstein space is a \textbf{velocity field} $\eta(x)$.
But not an arbitrary vector field---Otto showed that the tangent space is spanned by \textbf{gradient fields} $\eta = \nabla\psi$
(i.e., the velocity field must be the gradient of some scalar potential---``irrotational flow'').

\textbf{(ii) Why is $\dot\rho = -\nabla\cdot(\rho\eta)$? (What does ``density perturbation induced by a tangent vector'' mean?)}

In Wasserstein space: ``point'' $=$ density $\rho(x)$, ``tangent vector'' $=$ velocity field $\eta(x)$.

\textbf{Question: If all mass flows according to the velocity field $\eta$ for an instant, how does the density change?}

The answer is a rewriting of the continuity equation $\partial_t\rho + \nabla\cdot(\rho\,v) = 0$:
\[
\dot\rho := \partial_t\rho\big|_{t=0} = -\nabla\cdot(\rho\,\eta)
\]
Intuition: $\rho\eta$ is the mass flux (density $\times$ velocity), $\nabla\cdot(\rho\eta)$ is the net outflow,
so the density change $= -$ net outflow (however much flows out is subtracted).

This is the ``density perturbation induced by $\eta$''---translating from ``how mass flows'' to ``how density changes.''

\textbf{One-dimensional example:}

Let $\rho(x) = 1$ (uniform distribution on $[0,1]$).

\begin{itemize}
\item Take $\eta(x) = x$ (the further right, the faster the flow):
\[
\dot\rho = -\frac{\partial}{\partial x}(\rho\cdot x) = -\frac{\partial}{\partial x}(x) = -1
\]
Density decreases everywhere---mass accelerates to the right, and the mass that flows away from the left is not replenished.

\item Take $\eta(x) = 1$ (all mass translates to the right at the same speed):
\[
\dot\rho = -\frac{\partial}{\partial x}(\rho\cdot 1) = -\rho'(x) = 0
\]
The uniform distribution shifts as a whole; its shape is unchanged, $\dot\rho=0$.

\item Take $\eta(x) = -x$ (contraction toward the origin):
\[
\dot\rho = -\frac{\partial}{\partial x}(\rho\cdot(-x)) = -\frac{\partial}{\partial x}(-x) = 1
\]
Density increases everywhere---mass gathers from both sides toward the center.
\end{itemize}

\textbf{Finite-dimensional analogy:} Given a velocity $v$ (tangent vector), the rate of change of position is $\dot x = v$.
Here, ``position'' becomes ``density,'' ``velocity'' becomes ``velocity field,''
and ``$\dot x = v$'' becomes ``$\dot\rho = -\nabla\cdot(\rho\eta)$'' (the continuity equation).

\textbf{(iii) Why is $D\mathcal{F}(\rho)[\dot\rho] = \int\frac{\delta\mathcal{F}}{\delta\rho}\cdot\dot\rho\,dx$?}

This is the \textbf{definition} of the functional derivative. Recall:
\[
\frac{d}{d\epsilon}\bigg|_{\epsilon=0}\mathcal{F}(\rho + \epsilon\,h) = \int\frac{\delta\mathcal{F}}{\delta\rho}(x)\,h(x)\,dx
\]
Here $h = \dot\rho$ is the perturbation direction of the density. The left side is the directional derivative of $\mathcal{F}$ in direction $\dot\rho$, namely $D\mathcal{F}(\rho)[\dot\rho]$;
the right side is the $L^2$ inner product of the functional derivative $\frac{\delta\mathcal{F}}{\delta\rho}$ with the perturbation $\dot\rho$.

This is the exact analogue of the finite-dimensional identity $Df(x)[h] = \langle\nabla f, h\rangle = \sum_i\frac{\partial f}{\partial x_i}h_i$.

\medskip
\textbf{Step 1: Directional derivative along a Wasserstein tangent vector.}

Take a tangent vector $\eta = \nabla\psi$ (the gradient of some scalar function $\psi$). The density perturbation it induces is:
\[
    \dot\rho = -\nabla\cdot(\rho\,\nabla\psi)
\]

The directional derivative of $\mathcal{F}$ in this direction is (chain: first expand using the definition of functional derivative, then substitute the expression for $\dot\rho$):
\begin{align}
    D\mathcal{F}(\rho)[\dot\rho] &= \int\frac{\delta\mathcal{F}}{\delta\rho}(x)\cdot\dot\rho(x)\,dx
    &&\text{(definition of functional derivative)}\notag\\
    &= \int\underbrace{(\log\rho + 1 + V)}_{\frac{\delta\mathcal{F}}{\delta\rho}\text{ from previous section}}\cdot\underbrace{\bigl(-\nabla\cdot(\rho\nabla\psi)\bigr)}_{\text{expression for }\dot\rho}\,dx \label{eq:dir_deriv1}
\end{align}

The task now: use integration by parts to simplify the above into the form $\int\rho\,(\text{something})\cdot\nabla\psi\,dx$,
then compare with the Wasserstein inner product $\langle\xi,\eta\rangle_\rho = \int\rho\,\xi\cdot\eta\,dx = \int\rho\,\xi\cdot\nabla\psi\,dx$ to read off what $\xi$ is.

\textbf{Step 2: Integration by parts.}

The integral obtained from Step 1 is:
\[
D\mathcal{F}(\rho)[\dot\rho] = \int\underbrace{(\log\rho + 1 + V)}_{=:f(x)}\cdot\underbrace{\bigl(-\nabla\cdot(\rho\nabla\psi)\bigr)}_{=:-\nabla\cdot\mathbf{F}}\,dx
\]
where we denote $f := \log\rho+1+V$ (a scalar function) and $\mathbf{F} := \rho\nabla\psi$ (a vector field).

\textbf{Integration by parts formula} (see Appendix~\ref{app:analysis}): for compactly supported functions,
\[
\int f\,(\nabla\cdot\mathbf{F})\,dx = -\int\nabla f\cdot\mathbf{F}\,dx
\]
This formula follows from the product rule $\nabla\cdot(f\mathbf{F}) = \nabla f\cdot\mathbf{F} + f\,\nabla\cdot\mathbf{F}$, integrating both sides,
with the left side $\int\nabla\cdot(f\mathbf{F})\,dx = 0$ (divergence theorem $+$ vanishing boundary terms), and rearranging.

\textbf{Applying to our integral} (note the negative sign $-\nabla\cdot\mathbf{F}$):
\begin{align}
\int f\cdot(-\nabla\cdot\mathbf{F})\,dx &= -\int f\,(\nabla\cdot\mathbf{F})\,dx \notag\\
&= -\left(-\int\nabla f\cdot\mathbf{F}\,dx\right) \qquad\text{(by the integration by parts formula)}\notag\\
&= \int\nabla f\cdot\mathbf{F}\,dx \label{eq:ibp_grad}
\end{align}

Substituting back $f = \log\rho+1+V$, $\mathbf{F} = \rho\nabla\psi$:
\begin{align}
&= \int\nabla(\log\rho+1+V)\cdot(\rho\nabla\psi)\,dx \notag\\
&= \int\rho\,\nabla(\log\rho+V)\cdot\nabla\psi\,dx \notag
\end{align}
The last step uses $\nabla 1 = 0$ (the gradient of a constant is zero), so $\nabla(\log\rho+1+V) = \nabla(\log\rho+V)$.

\medskip
\textbf{Summary of Step 2}: Integration by parts transfers the ``$\nabla\cdot$'' from $\mathbf{F}=\rho\nabla\psi$ into a ``$\nabla$'' acting on $f = \log\rho+1+V$.
Result: the integral becomes $\int\rho\,\nabla(\log\rho+V)\cdot\nabla\psi\,dx$, already of the form $\int\rho\,(\text{vector field})\cdot\nabla\psi\,dx$.

\textbf{Step 3: Identify the gradient.}

We need $\xi$ such that $\langle\xi,\nabla\psi\rangle_\rho = D\mathcal{F}(\rho)[\dot\rho]$ for all $\psi$:
\[
    \int\rho\,\xi\cdot\nabla\psi\,dx = \int\rho\,\nabla(\log\rho + V)\cdot\nabla\psi\,dx \quad \forall\psi
\]

This forces:
\begin{equation}\label{eq:wass_grad}
    \boxed{\grad_W\mathcal{F}(\rho) = \nabla(\log\rho + V) = \nabla\log\rho + \nabla V}
\end{equation}

\subsection{The main result: Fokker--Planck as a Wasserstein gradient flow}

The definition of the Wasserstein gradient flow: ``the curve of steepest descent of $\mathcal{F}$ in probability space.''

The Wasserstein gradient flow of $\mathcal{F}$ is defined by:
\begin{equation}
    \partial_t\rho_t + \nabla\cdot\bigl(\rho_t\,v_t\bigr) = 0, \quad v_t = -\grad_W\mathcal{F}(\rho_t)
\end{equation}

Substituting \eqref{eq:wass_grad}:
\begin{align}
    v_t &= -(\nabla\log\rho_t + \nabla V) \notag\\
    \partial_t\rho_t &= -\nabla\cdot(\rho_t\,v_t) = \nabla\cdot\bigl(\rho_t(\nabla\log\rho_t + \nabla V)\bigr) \notag\\
    &= \nabla\cdot(\nabla\rho_t + \rho_t\nabla V) \notag\\
    &= \Delta\rho_t + \nabla\cdot(\rho_t\nabla V) \label{eq:FP_again}
\end{align}

\begin{keypoint}
\textbf{The Fokker-Planck equation is the gradient flow of the free energy $\mathcal{F}(\rho) = \KL(\rho\|\pi)$ in Wasserstein space.}

Just as $\dot x = -\nabla f(x)$ in Euclidean space moves $x$ in the direction of steepest descent of $f$,
the Fokker-Planck equation makes $\rho_t$ ``flow'' in probability space along the direction of steepest descent of $\KL(\cdot\|\pi)$, until reaching the minimizer $\rho_\infty = \pi$.

This is the central contribution of Jordan, Kinderlehrer, and Otto (1998).
\end{keypoint}

\subsection{End-of-section comparison: Euclidean vs.\ Wasserstein gradient-flow logic}

The whole section is the Wasserstein analogue of the most familiar Euclidean story:
\[
\text{function} \;\longrightarrow\; \text{gradient} \;\longrightarrow\; \text{gradient flow} \;\longrightarrow\; \text{energy decreases}.
\]
The only difference is that the ``point'' is now a probability density and the ``velocity'' is represented by a vector field through the continuity equation.

\begin{center}
\renewcommand{\arraystretch}{1.35}
\begin{tabular}{@{}p{3.1cm}p{5.1cm}p{6.2cm}@{}}
\toprule
\textbf{Step in the analysis} & \textbf{Euclidean space $\R^n$} & \textbf{Wasserstein space $\Prob_2(\R^d)$} \\
\midrule
Point
& A vector $x\in\R^n$
& A probability density $\rho\in\Prob_2(\R^d)$ \\[4pt]

Tangent vector
& A velocity $\dot x\in\R^n$
& A velocity field $v$ inducing $\dot\rho=-\nabla\cdot(\rho v)$ through the continuity equation \\[4pt]

Metric
& Euclidean inner product $\langle u,v\rangle=u\cdot v$
& Otto inner product $\langle v,w\rangle_\rho=\int v\cdot w\,\rho\,dx$ \\[4pt]

Objective
& A function $f(x)$
& A functional $\mathcal{F}(\rho)=\int V\rho\,dx+\int\rho\log\rho\,dx=\KL(\rho\|\pi)+\mathrm{const}$ \\[4pt]

First derivative
& $Df(x)[h]=\nabla f(x)\cdot h$
& $D\mathcal{F}(\rho)[\dot\rho]=\int\frac{\delta\mathcal{F}}{\delta\rho}\dot\rho\,dx$ \\[4pt]

Gradient
& $\nabla f(x)$, defined by $\langle\nabla f,h\rangle=Df[h]$
& $\grad_W\mathcal{F}(\rho)=\nabla\frac{\delta\mathcal{F}}{\delta\rho}=\nabla(\log\rho+V)$ \\[4pt]

Gradient flow
& $\dot x=-\nabla f(x)$
& $\partial_t\rho+\nabla\cdot(\rho v)=0,\quad v=-\grad_W\mathcal{F}(\rho)$ \\[4pt]

Concrete PDE
& Ordinary differential equation for $x_t$
& Fokker--Planck equation $\partial_t\rho=\Delta\rho+\nabla\cdot(\rho\nabla V)$ \\[4pt]

Energy dissipation
& $\frac{d}{dt}f(x_t)=-|\nabla f(x_t)|^2\le 0$
& $\frac{d}{dt}\mathcal{F}(\rho_t)=-\int\rho_t|\nabla\log(\rho_t/\pi)|^2dx\le 0$ \\[4pt]

Equilibrium
& Critical point $\nabla f=0$
& Gibbs distribution $\rho_\infty=\pi\propto e^{-V}$ \\
\bottomrule
\end{tabular}
\end{center}

\begin{keypoint}
The proof strategy is exactly the same as in Euclidean space.
First identify the metric, then compute the gradient of the objective under that metric, then write the negative-gradient dynamics.
In Wasserstein space, this negative-gradient dynamics is not an ODE for points but a continuity equation for densities; for the free energy above, it becomes the Fokker--Planck equation.
\end{keypoint}

\medskip
\textbf{The next natural question:} Since Fokker-Planck is a gradient flow, can we---just as in $\R^n$---use discrete optimization to ``approximate'' this continuous descent process? This is the JKO scheme.

\section{The JKO Scheme}

\textbf{Where we are:} Section~5 proved that the Fokker-Planck equation is the gradient flow of the free energy $\mathcal{F}$ in Wasserstein space. This section completes the final step: \textbf{discretizing} this continuous process into a step-by-step optimization problem.

\subsection{Implicit Euler method: from gradient flow to optimization problem}

Consider the gradient flow $\dot x = -\nabla f(x)$ in Euclidean space. The implicit Euler method discretizes it using a backward difference:
\[
\frac{x_{k+1}-x_k}{\tau} = -\nabla f(x_{k+1})
\]
Rearranging gives $\nabla f(x_{k+1}) + \frac{1}{\tau}(x_{k+1}-x_k) = 0$. The left side is precisely $\nabla_x\bigl[f(x)+\frac{1}{2\tau}|x-x_k|^2\bigr]\big|_{x_{k+1}}$, so implicit Euler is equivalent to:
\begin{equation}\label{eq:implicit_euler}
x_{k+1} = \arg\min_{x}\left\{f(x) + \frac{1}{2\tau}|x - x_k|^2\right\}
\end{equation}

\textbf{Intuition}: Each step finds the optimal trade-off between decreasing $f$ and not straying too far from $x_k$. The step size $\tau$ controls the step length: small $\tau$ lets the penalty term dominate, giving small but stable steps; large $\tau$ lets the objective term dominate, giving large steps that may overshoot.

\textbf{Why is the stationary point a minimum?} The Hessian of $g(x) = f(x)+\frac{1}{2\tau}|x-x_k|^2$ is $\nabla^2 f + \frac{1}{\tau}I$. When $\tau$ is sufficiently small ($1/\tau$ exceeds the most negative eigenvalue of $\nabla^2 f$), $g$ is strictly convex and the stationary point is the unique global minimum.

\begin{center}
\renewcommand{\arraystretch}{1.3}
\begin{tabular}{@{}lll@{}}
\toprule
& \textbf{Explicit Euler} & \textbf{Implicit Euler} \\
\midrule
Formula & $x_{k+1} = x_k - \tau\nabla f(x_k)$ & $x_{k+1} = \arg\min\{f(x)+\frac{1}{2\tau}|x-x_k|^2\}$ \\
Gradient evaluated at & Current point $x_k$ & New point $x_{k+1}$ \\
Stability & Step size restricted (CFL condition) & Unconditionally stable \\
\bottomrule
\end{tabular}
\end{center}

\subsection{JKO scheme: Implicit Euler in Wasserstein space}

The core idea of Jordan--Kinderlehrer--Otto (1998): replace every Euclidean concept in implicit Euler with its Wasserstein-space counterpart.

\begin{definition}[JKO scheme]\label{def:jko}
Given an initial distribution $\rho_0^\tau = \rho_0$ and step size $\tau>0$, the JKO iteration is:
\begin{equation}\label{eq:jko}
\boxed{\rho_{k+1}^\tau = \arg\min_{\rho\in\Prob_2(\R^d)}\left\{\mathcal{F}(\rho) + \frac{1}{2\tau}W_2^2(\rho, \rho_k^\tau)\right\}}
\end{equation}
where $\mathcal{F}(\rho) = \int\rho\log\rho\,dx + \int V\rho\,dx$ is the free energy functional.
\end{definition}

\begin{keypoint}
\textbf{Orientation convention for JKO.}
In the JKO step
\[
\rho_{k+1}
=
\arg\min_\rho
\left\{
\mathcal{F}(\rho)+\frac{1}{2\tau}W_2^2(\rho,\rho_k)
\right\},
\]
$\rho_k$ is the old distribution and $\rho_{k+1}$ is the new distribution.
We use the optimal map
\[
T_k:\rho_{k+1}\to\rho_k,\qquad (T_k)_\#\rho_{k+1}=\rho_k.
\]
Thus, for $x\sim\rho_{k+1}$, the point $T_k(x)$ is its previous location in $\rho_k$, and
\[
x-T_k(x)
\]
is the displacement from old to new. Therefore the discrete velocity is
\[
\frac{x-T_k(x)}{\tau}.
\]
The JKO optimality condition is consequently
\[
\boxed{
\frac{x-T_k(x)}{\tau}
=
-\nabla\frac{\delta\mathcal{F}}{\delta\rho}(x)
}
\]
with the force evaluated at the \textbf{new} point $x$ and the \textbf{new} density $\rho_{k+1}$.
\end{keypoint}

\textbf{Complete analogy dictionary:}
\begin{center}
\renewcommand{\arraystretch}{1.4}
\begin{tabular}{@{}lcc@{}}
\toprule
& \textbf{Euclidean space} & \textbf{Wasserstein space} \\
\midrule
Space & $\R^n$ & $\Prob_2(\R^d)$ \\
Point & $x$ & $\rho$ \\
Distance$^2$ & $|x-y|^2$ & $W_2^2(\mu,\nu)$ \\
Objective & $f(x)$ & $\mathcal{F}(\rho) = \KL(\rho\|\pi)$ \\
Gradient & $\nabla f$ & $\mathrm{grad}_W\mathcal{F} = \nabla(\log\rho+V)$ \\
Gradient flow & $\dot x = -\nabla f$ & $\partial_t\rho = \Delta\rho + \nabla\cdot(\rho\nabla V)$ (FP equation) \\
Implicit Euler & $\min_x\{f(x)+\frac{1}{2\tau}|x-x_k|^2\}$ & $\min_\rho\{\mathcal{F}(\rho)+\frac{1}{2\tau}W_2^2(\rho,\rho_k)\}$ \\
\bottomrule
\end{tabular}
\end{center}

\textbf{Why must the objective be the free energy $\mathcal{F}$?} This is not a ``choice.'' Section~5 already showed that the velocity field of the Fokker-Planck equation is $v = -\nabla(\log\rho + V)$, while the Wasserstein gradient flow satisfies $v = -\mathrm{grad}_W\mathcal{F} = -\nabla\frac{\delta\mathcal{F}}{\delta\rho}$. Comparing the two immediately gives $\frac{\delta\mathcal{F}}{\delta\rho} = \log\rho + V + C$, and the unique functional satisfying this is $\mathcal{F}(\rho) = \int\rho\log\rho + \int V\rho$. The Fokker-Planck equation \textbf{determines} the JKO objective.

\subsection{Deriving JKO from the gradient flow}

\textbf{Starting point:} The Wasserstein gradient flow (conclusion of Section~5):
\[
v_t = -\mathrm{grad}_W\mathcal{F}(\rho_t) = -\nabla(\log\rho_t + V)
\]

\textbf{Step 1: Backward difference.}
Let $T_k$ be the optimal transport map from $\rho_{k+1}$ to $\rho_k$ ($(T_k)_\#\rho_{k+1} = \rho_k$). Then $x-T_k(x)$ is the displacement from the previous step to the current step, and the approximate velocity is $v\approx\frac{\mathrm{id}-T_k}{\tau}$. Substituting into the gradient flow equation:
\begin{equation}\label{eq:jko_implicit}
\frac{\mathrm{id} - T_k}{\tau} = -\mathrm{grad}_W\mathcal{F}(\rho_{k+1}) = -\nabla\log\rho_{k+1} - \nabla V
\end{equation}
The right-hand side is evaluated at the \textbf{new} distribution $\rho_{k+1}$---this is precisely the meaning of ``implicit.''

\textbf{Step 2: Recognize the optimality condition of the optimization problem.}
Rearranging \eqref{eq:jko_implicit} as
\[
\mathrm{grad}_W\mathcal{F}(\rho_{k+1}) + \frac{\mathrm{id}-T_k}{\tau} = 0
\]
and $\frac{\mathrm{id}-T_k}{\tau}$ is precisely the Wasserstein gradient of $\frac{1}{2\tau}W_2^2(\rho,\rho_k)$ (proved in the next subsection), so the above is equivalent to
\[
\mathrm{grad}_W\left[\mathcal{F}(\rho) + \frac{1}{2\tau}W_2^2(\rho,\rho_k)\right]\bigg|_{\rho_{k+1}} = 0
\]
The critical point $=$ the minimizer (guaranteed by displacement convexity of the $W_2^2$ term), yielding the JKO formula~\eqref{eq:jko}.

\subsection{First-order optimality conditions of JKO}

The argmin formula of JKO specifies ``what to do at each step,'' but it must be expanded into a computable form. This subsection derives the \textbf{first-order conditions} satisfied by the JKO minimizer, which are the key to subsequently proving convergence to the Fokker-Planck equation.

\subsubsection*{Step 1: Variational condition}

$\rho_{k+1}$ is the solution to the following problem:
\[
\min_{\rho:\,\int\rho=1,\,\rho\geq 0}\left\{\mathcal{F}(\rho) + \frac{1}{2\tau}W_2^2(\rho,\rho_k)\right\}
\]
For admissible perturbations $\delta\rho$ (satisfying $\int\delta\rho=0$ to preserve mass normalization), the necessary condition for a minimum is:
\[
\int\left[\frac{\delta\mathcal{F}}{\delta\rho}\bigg|_{\rho_{k+1}} + \frac{1}{2\tau}\frac{\delta W_2^2(\cdot,\rho_k)}{\delta\rho}\bigg|_{\rho_{k+1}}\right]\delta\rho\,dx = 0, \quad \forall\,\delta\rho:\;\int\delta\rho=0
\]
By the constrained du Bois-Reymond lemma (Appendix~\ref{app:analysis}), the function inside the brackets must be a constant $C$ (the Lagrange multiplier corresponding to the constraint $\int\rho=1$):
\begin{equation}\label{eq:jko_opt}
\frac{\delta\mathcal{F}}{\delta\rho}\bigg|_{\rho_{k+1}} + \frac{1}{2\tau}\frac{\delta W_2^2(\cdot,\rho_k)}{\delta\rho}\bigg|_{\rho_{k+1}} = C
\end{equation}

\subsubsection*{Step 2: Functional derivative of $W_2^2$}

Let $T_k$ be the optimal transport map from $\rho_{k+1}$ to $\rho_k$. The following is a classical result (Santambrogio, \textit{Optimal Transport for Applied Mathematicians}, Prop.\ 7.17):
\begin{equation}\label{eq:W2_variation}
\frac{1}{2}\frac{\delta W_2^2(\cdot,\rho_k)}{\delta\rho}\bigg|_{\rho_{k+1}}(x) = \frac{1}{2}|x - T_k(x)|^2
\end{equation}

\textbf{Intuition}: Adding a small amount of mass at $x$ means that this mass also needs to be transported to $\rho_k$, at a cost of $|x-T_k(x)|^2$.

\textbf{Derivation}: Let $T_\epsilon$ be the optimal map from $\rho_\epsilon := \rho_{k+1}+\epsilon\,\delta\rho$ to $\rho_k$. By the optimality condition, the variation of $T_\epsilon$ is of higher order in $\epsilon$ (the ``envelope theorem''), so
\[
\left.\frac{d}{d\epsilon}\right|_0 \frac{1}{2}\int|x-T_\epsilon(x)|^2\rho_\epsilon\,dx = \frac{1}{2}\int|x-T_k(x)|^2\,\delta\rho\,dx
\]
Comparing with the definition of the functional derivative $\int\frac{\delta(\cdots)}{\delta\rho}\,\delta\rho\,dx$, we read off~\eqref{eq:W2_variation}.

\subsubsection*{Step 3: Substitute and take the spatial gradient}

Substituting $\frac{\delta\mathcal{F}}{\delta\rho} = \log\rho_{k+1} + 1 + V$ and \eqref{eq:W2_variation} into \eqref{eq:jko_opt}:
\[
\log\rho_{k+1}(x) + V(x) + \frac{1}{2\tau}|x-T_k(x)|^2 = C'
\]
Taking the spatial gradient $\nabla_x$ of both sides (the constant $C'$ vanishes):
\begin{equation}\label{eq:jko_grad}
\nabla\log\rho_{k+1} + \nabla V + \frac{1}{\tau}(x - T_k(x)) = 0
\end{equation}
(The last term is the Wasserstein gradient of $\frac{1}{2\tau}W_2^2(\cdot,\rho_k)$ at $\rho_{k+1}$. By a standard result of optimal transport, the first variation $\frac{\delta}{\delta\rho}\big[\frac{1}{2}W_2^2(\cdot,\rho_k)\big]\big|_{\rho_{k+1}}$ is the Kantorovich potential $\varphi_k$, whose spatial gradient is \emph{exactly} the displacement $\nabla\varphi_k(x)=x-T_k(x)$ (by definition of the Brenier map, $T_k(x)=x-\nabla\varphi_k(x)$). Hence the gradient is $\frac{x-T_k(x)}{\tau}$, with \emph{no} $\nabla T_k$ correction term; see Santambrogio Prop.~7.17 / AGS Ch.~10 for the rigorous statement.)

\subsubsection*{Step 4: Recognize the Fokker-Planck velocity field}

Solving for the discrete velocity from \eqref{eq:jko_grad}:
\begin{equation}\label{eq:jko_velocity}
\frac{x - T_k(x)}{\tau} = -\nabla\log\rho_{k+1}(x) - \nabla V(x)
\end{equation}
The left-hand side is the approximate velocity from $\rho_k$ to $\rho_{k+1}$ ($T_k$ pushes $\rho_{k+1}$ back to $\rho_k$, so $x-T_k(x)$ is the displacement).
The right-hand side is precisely the Fokker-Planck velocity field $v = -\nabla\log\rho - \nabla V$.

Therefore: \textbf{the discrete velocity field produced by each JKO step is exactly the Fokker-Planck velocity field at that time instant}. As $\tau\to 0$, the discrete JKO solution converges to the continuous-time solution of the Fokker-Planck equation.

\subsection{Rigorous convergence theorem}

\begin{theorem}[JKO, 1998; Ambrosio--Gigli--Savar\'e, 2008]\label{thm:jko_convergence}
Assume $V$ is $\lambda$-convex ($\nabla^2 V \geq \lambda I$, $\lambda\in\R$) and satisfies a quadratic growth condition. Define the piecewise constant interpolation:
\[
\rho^\tau(t) := \rho_k^\tau \quad\text{for } t \in [k\tau, (k+1)\tau)
\]
Then as $\tau \to 0$, $\rho^\tau(t)$ converges in the $W_2$ metric to the unique solution of the Fokker-Planck equation.
\end{theorem}

\textbf{Four key steps of the proof:}
\begin{enumerate}
\item \textbf{Existence}: Each JKO subproblem has a solution---the $W_2^2$ term provides compactness, and $\mathcal{F}$ is lower semicontinuous in the $W_2$ topology.
\item \textbf{Energy decay}: $\mathcal{F}(\rho_{k+1}^\tau) \leq \mathcal{F}(\rho_k^\tau)$, since $\rho_k$ itself is a feasible solution (with cost $\mathcal{F}(\rho_k)+0$).
\item \textbf{Uniform estimates}: Energy decay $+$ control of the Fisher information yields compactness of $\{\rho_k^\tau\}$, enabling passage to the limit $\tau\to 0$.
\item \textbf{Identification of the limit}: Via the optimality condition \eqref{eq:jko_grad}, the discrete velocity field converges to $-\nabla\log\rho-\nabla V$ as $\tau\to 0$, recovering the Fokker-Planck equation.
\end{enumerate}

\subsection{Structural advantages of JKO}

\begin{enumerate}[label=(\alph*)]
\item \textbf{Positivity and mass conservation are automatic}: Each step optimizes over $\Prob_2(\R^d)$, so $\rho_{k+1}\geq 0$ and $\int\rho_{k+1}=1$ are built-in constraints.
\item \textbf{Unconditional stability}: No CFL condition is needed; arbitrarily large $\tau$ still produces a meaningful (though coarse) approximation.
\item \textbf{Built-in energy dissipation}: $\mathcal{F}(\rho_{k+1})\leq\mathcal{F}(\rho_k)$ without any additional verification.
\item \textbf{Modularity}: Changing $\mathcal{F}$ handles different PDEs:
\end{enumerate}

\begin{center}
\renewcommand{\arraystretch}{1.3}
\begin{tabular}{@{}ll@{}}
\toprule
\textbf{Functional} $\mathcal{F}(\rho)$ & \textbf{Corresponding PDE} \\
\midrule
$\displaystyle\int\rho\log\rho\,dx$(neg.\ entropy) & Heat equation: $\partial_t\rho = \Delta\rho$ \\[6pt]
$\displaystyle\int V\rho\,dx$(potential energy) & Aggregation equation: $\partial_t\rho = \nabla\cdot(\rho\nabla V)$ \\[6pt]
$\displaystyle\int\rho\log\rho + \int V\rho$(free energy) & Fokker--Planck \\[6pt]
$\displaystyle\frac{1}{m-1}\int\rho^m\,dx$(internal energy) & Porous medium equation: $\partial_t\rho = \Delta(\rho^m)$ \\[6pt]
$\displaystyle\frac{1}{2}\iint W(x\!-\!y)\rho(x)\rho(y)\,dxdy$ & McKean--Vlasov equation (McKean, 1966) \\[6pt]
\bottomrule
\end{tabular}
\end{center}

\subsection{Example: Heat equation}

Take $\mathcal{F}(\rho) = \int\rho\log\rho\,dx$ and $V = 0$. Then $\frac{\delta\mathcal{F}}{\delta\rho} = \log\rho + 1$, and the Wasserstein gradient $= \nabla\log\rho$. The gradient flow velocity is $v = -\nabla\log\rho$; substituting into the continuity equation:
\[
\partial_t\rho = -\nabla\cdot(\rho v) = \nabla\cdot(\rho\nabla\log\rho) = \nabla\cdot(\nabla\rho) = \Delta\rho
\]
The heat equation describes ``diffusion purely due to crowding''---the system only seeks to maximize entropy, with no external force.

Each JKO step becomes: $\rho_{k+1} = \arg\min_\rho\bigl\{\int\rho\log\rho\,dx + \frac{1}{2\tau}W_2^2(\rho,\rho_k)\bigr\}$.
This provides an ``optimization perspective'' on the heat equation: the progressive spreading of the Gaussian heat kernel $=$ stepwise entropy maximization (subject to transport cost constraints).

\subsection{Application: Energy Matching (Balcerak et al., 2025)}

The JKO scheme is not only a theoretical device; it is used directly in the design of generative models.

\textbf{Key idea}: Energy Matching uses a single \textbf{time-independent} scalar potential $V_\theta(x)$ as the only learnable parameter,
achieving both transport and equilibrium through the JKO scheme:
\begin{equation}\label{eq:energy_matching_jko}
\rho_{t+\Delta t} = \arg\min_\rho\left\{
\frac{W_2^2(\rho,\rho_t)}{2\Delta t} + 
\int V_\theta\,d\rho + 
\varepsilon(t)\int\rho\log\rho\,dx
\right\}
\end{equation}
Compared to the standard JKO, the innovation lies in allowing the \textbf{temperature $\varepsilon(t)$ to vary with time}.

\textbf{First-order optimality condition} (directly from the derivation in Section~6.4):
\[
\frac{x-y}{\Delta t} + \nabla V_\theta(x) + \varepsilon(t)\nabla\log\rho_{t+\Delta t}(x) = 0
\]
where $y = T(x)$ is the point in $\rho_t$ corresponding to $x$. This condition yields strikingly different behavior in two limiting regimes:

\textbf{Phase 1 ($\varepsilon=0$, far from data)}: The condition reduces to $\frac{x-y}{\Delta t} = -\nabla V_\theta(x)$. Particles undergo deterministic transport with velocity $-\nabla V_\theta$---a pure OT flow, similar to Flow Matching.

\textbf{Phase 2 ($\varepsilon=\varepsilon_{\max}$, near equilibrium)}: Samples barely move ($x\approx y$), and the condition becomes $\nabla V_\theta + \varepsilon_{\max}\nabla\log\rho_\text{eq}=0$, yielding the equilibrium distribution $\rho_\text{eq}\propto e^{-V_\theta/\varepsilon_{\max}}$---a Boltzmann distribution. $V_\theta$ directly encodes the log-likelihood of the data.

\begin{keypoint}
\textbf{One $V_\theta(x)$, two roles}: Far from data, $-\nabla V_\theta$ serves as the transport velocity (Flow Matching); near data, $V_\theta/\varepsilon$ serves as the energy function (EBM). There is no need to learn a time-dependent $s_\theta(x,t)$ or $v_\theta(x,t)$.
\end{keypoint}

\textbf{Correspondence with concepts in this article:}
\begin{center}
\renewcommand{\arraystretch}{1.3}
\begin{tabular}{@{}ll@{}}
\toprule
\textbf{Concept in this article} & \textbf{Role in Energy Matching} \\
\midrule
$W_2$ distance & Transport cost in JKO \\
Continuity equation & Conservation law for sample evolution \\
Fokker-Planck equation & Continuous limit when $\varepsilon>0$ \\
Free energy $\mathcal{F}$ & Per-step minimization objective \\
Gibbs distribution $e^{-V/\varepsilon}$ & Equilibrium = data distribution \\
JKO optimality condition & Core tool for deriving the two-phase behavior \\
\bottomrule
\end{tabular}
\end{center}

\subsection{Unified derivation from JKO to mainstream generative algorithms}

Instead of going through the continuous-time Fokker--Planck equation, we can read JKO directly as an \textbf{iterative generative algorithm}. The basic loop is:

\begin{center}
\fbox{\parbox{0.88\textwidth}{
\textbf{JKO generative template}
\begin{enumerate}[leftmargin=2em]
\item Choose a free energy
\[
\mathcal{F}_k(\rho)=\int V_k\rho\,dx+\varepsilon_k\int\rho\log\rho\,dx.
\]
\item Initialize from a simple distribution, e.g.\ $\rho_0=\mathcal{N}(0,I)$.
\item For $k=0,\ldots,K-1$, compute the proximal Wasserstein step
\[
\rho_{k+1}=\arg\min_\rho\left\{\mathcal{F}_k(\rho)+\frac{1}{2\tau_k}W_2^2(\rho,\rho_k)\right\}.
\]
\item Implement this density step on particles using a learned score, energy, or velocity model.
\end{enumerate}
}}
\end{center}

For one step, let $T_k$ be the optimal map from the new density $\rho_{k+1}$ back to the old density $\rho_k$:
\[
T_k:\rho_{k+1}\to\rho_k,
\qquad
(T_k)_\#\rho_{k+1}=\rho_k.
\]
Then the JKO optimality condition is
\begin{equation}\label{eq:jko_unified}
\frac{x-T_k(x)}{\tau_k}
=
-\nabla V_k(x)-\varepsilon_k\nabla\log\rho_{k+1}(x),
\qquad x\sim\rho_{k+1}.
\end{equation}
Equivalently,
\[
x
=
T_k(x)
+
\tau_k\bigl[-\nabla V_k(x)-\varepsilon_k\nabla\log\rho_{k+1}(x)\bigr].
\]
This is already an algorithmic update rule: the new point equals the old point plus a force evaluated at the new point (implicit Euler). Mainstream generative algorithms differ in how they choose $\mathcal{F}_k$, how they schedule $(\tau_k,\varepsilon_k)$, and how they approximate the unknown score or energy in \eqref{eq:jko_unified}.

\textbf{Indexing note for diffusion papers}: DDPM papers often index sampling backward, from $x_t$ to $x_{t-1}$. In this subsection, we use the JKO convention consistently: $k$ is old and $k+1$ is new. To compare with DDPM notation, read ``$k+1$'' here as the next denoising state, even if the diffusion paper labels that state by a smaller physical time index.

\subsubsection*{1. DDPM: stochastic JKO with Gaussian reference potential}

\textbf{JKO choice}: use the Gaussian reference potential $V_k(x)=\frac{1}{2}|x|^2$ and retain the entropy term. One step is
\[
\rho_{k+1}=\arg\min_\rho\left\{
\int\frac{1}{2}|x|^2\rho\,dx+\varepsilon_k\int\rho\log\rho\,dx+\frac{1}{2\tau_k}W_2^2(\rho,\rho_k)
\right\}.
\]
The optimality condition is
\[
\frac{x-T_k(x)}{\tau_k}=-x-\varepsilon_k\nabla\log\rho_{k+1}(x),
\qquad x\sim\rho_{k+1}.
\]
Let a particle in the old distribution be denoted by $x_k\sim\rho_k$, and suppose its paired new position is $x_{k+1}\sim\rho_{k+1}$, so $T_k(x_{k+1})=x_k$. Substituting $x=x_{k+1}$ into the optimality condition gives
\[
\frac{x_{k+1}-x_k}{\tau_k}
=
-x_{k+1}-\varepsilon_k\nabla\log\rho_{k+1}(x_{k+1}).
\]
Therefore the exact JKO particle equation is
\begin{equation}\label{eq:jko_ddpm_implicit_particle}
x_{k+1}
=
x_k
+
\tau_k\left[-x_{k+1}-\varepsilon_k\nabla\log\rho_{k+1}(x_{k+1})\right].
\end{equation}
If we include finite-temperature Langevin sampling inside the same proximal step, we add the Gaussian fluctuation:
\begin{equation}\label{eq:jko_ddpm_stochastic_particle}
x_{k+1}
\approx
x_k
+
\tau_k\left[-x_{k+1}-\varepsilon_k\nabla\log\rho_{k+1}(x_{k+1})\right]
+
\sqrt{2\varepsilon_k\tau_k}\,z,
\qquad z\sim\mathcal{N}(0,I).
\end{equation}

\textbf{This is the exact place where JKO is implicit}: both the drift $-x_{k+1}$ and the score $\nabla\log\rho_{k+1}(x_{k+1})$ are evaluated at the unknown new point.

Let us first see where the $x_{k+1}$ on the right-hand side goes. Ignoring the Gaussian fluctuation for one line, \eqref{eq:jko_ddpm_implicit_particle} gives
\[
x_{k+1}
=
x_k-\tau_kx_{k+1}-\tau_k\varepsilon_k\nabla\log\rho_{k+1}(x_{k+1}).
\]
Move the linear drift term to the left:
\[
(1+\tau_k)x_{k+1}
=
x_k-\tau_k\varepsilon_k\nabla\log\rho_{k+1}(x_{k+1}).
\]
Therefore
\begin{equation}\label{eq:jko_forward_solved}
x_{k+1}
=
\frac{1}{1+\tau_k}x_k
-
\frac{\tau_k\varepsilon_k}{1+\tau_k}\nabla\log\rho_{k+1}(x_{k+1}).
\end{equation}
So the $x_{k+1}$ on the right has not disappeared: it has been absorbed into the prefactor $\frac{1}{1+\tau_k}$ after solving the implicit linear drift term.

\medskip
\textbf{Reverse/generative implementation.} DDPM generation runs a reverse denoising update. In the same $k\to k+1$ notation, write the explicit score-form update as
\begin{equation}\label{eq:score_update_gamma}
x_{k+1}
=
\frac{1}{\sqrt{\alpha_k}}x_k
+
\gamma_k\,s_\theta(x_k,k)
+
\sigma_k z.
\end{equation}
Here $\alpha_k\in(0,1)$ is the usual DDPM schedule coefficient, and $\gamma_k$ is the score coefficient. To match the standard DDPM formula, the relation is
\begin{equation}\label{eq:gamma_alpha_relation}
\boxed{\gamma_k=\frac{1-\alpha_k}{\sqrt{\alpha_k}}.}
\end{equation}

\textbf{Important: what is the relation between the two scores?}
The implicit JKO equation contains
\[
\nabla\log\rho_{k+1}(x_{k+1}),
\]
where both the distribution and the point are the \emph{new} unknowns. The explicit DDPM-style update uses
\[
s_\theta(x_k,k)\approx\nabla\log\rho_k(x_k),
\]
the score at the currently available point and current marginal. These two quantities are not exactly equal:
\[
\nabla\log\rho_{k+1}(x_{k+1})
\neq
\nabla\log\rho_k(x_k)
\quad\text{in general.}
\]
The connection is an \textbf{explicit time-discretization approximation}: for small steps and a smooth density path,
\[
\rho_{k+1}\approx\rho_k,\qquad x_{k+1}\approx x_k,
\]
so one replaces the implicit new-point score by the available old-point score:
\[
\nabla\log\rho_{k+1}(x_{k+1})
\approx
\nabla\log\rho_k(x_k)
\approx
s_\theta(x_k,k).
\]
This is exactly the same conceptual move as replacing implicit Euler by an explicit Euler-like implementation. Thus DDPM's practical update should not be read as an exact solution of the JKO subproblem; it is a score-trained, explicit particle approximation to that implicit proximal step.

Now introduce the score estimator and the noise-prediction parameterization:
\[
s_\theta(x_k,k)\approx\nabla\log\rho_k(x_k),
\qquad
\epsilon_\theta(x_k,k)\approx-\sqrt{1-\bar\alpha_k}\,s_\theta(x_k,k).
\]
Equivalently,
\[
s_\theta(x_k,k)
=
-\frac{\epsilon_\theta(x_k,k)}{\sqrt{1-\bar\alpha_k}}.
\]
Substitute this into \eqref{eq:score_update_gamma}:
\[
x_{k+1}
=
\frac{1}{\sqrt{\alpha_k}}x_k
-
\frac{\gamma_k}{\sqrt{1-\bar\alpha_k}}\epsilon_\theta(x_k,k)
+
\sigma_k z.
\]
Using \eqref{eq:gamma_alpha_relation},
\[
\frac{\gamma_k}{\sqrt{1-\bar\alpha_k}}
=
\frac{1-\alpha_k}{\sqrt{\alpha_k}\sqrt{1-\bar\alpha_k}}.
\]
Therefore
\[
x_{k+1}
=
\frac{1}{\sqrt{\alpha_k}}x_k
-
\frac{1-\alpha_k}{\sqrt{\alpha_k}\sqrt{1-\bar\alpha_k}}\epsilon_\theta(x_k,k)
+
\sigma_k z.
\]
Factoring out $\frac{1}{\sqrt{\alpha_k}}$ gives the standard DDPM update:
\begin{equation}\label{eq:ddpm_explicit_from_jko}
x_{k+1}
=
\frac{1}{\sqrt{\alpha_k}}
\left(
x_k
-
\frac{1-\alpha_k}{\sqrt{1-\bar\alpha_k}}\epsilon_\theta(x_k,k)
\right)
+
\sigma_k z.
\end{equation}
Thus \eqref{eq:ddpm_explicit_from_jko} should be read as an \textbf{explicit reverse-time score implementation/approximation} of the implicit JKO particle equation \eqref{eq:jko_ddpm_stochastic_particle}. The implicit new-point score $\nabla\log\rho_{k+1}(x_{k+1})$ is replaced by the learned score at the currently available point $x_k$, and the coefficient relation $\gamma_k=(1-\alpha_k)/\sqrt{\alpha_k}$ converts the score-form update into the standard $\epsilon_\theta$-form update.
Thus DDPM is a stochastic particle implementation of repeated JKO proximal steps: the score implements the entropy force, the Gaussian drift implements the reference potential, and the random $z$ implements finite-temperature Langevin sampling inside each proximal step.

\subsubsection*{2. DDIM: deterministic JKO transport using the same learned score}

DDIM (Song, Meng \& Ermon, 2021) uses the same learned score but removes the fresh stochastic sampling noise. In JKO language, it keeps the deterministic transport part of the proximal map:
\[
x_{k+1}\approx x_k+\tau_k[-\nabla V_k(x_{k+1})-\varepsilon_k s_\theta(x_{k+1},k+1)].
\]
In the DDPM parameterization, this gives the deterministic DDIM-style update
\[
x_{\mathrm{next}}
=
\sqrt{\bar\alpha_{\mathrm{next}}}\,\hat x_0(x_t,t)
+\sqrt{1-\bar\alpha_{\mathrm{next}}}\,\epsilon_\theta(x_t,t),
\]
where $\hat x_0$ is the usual estimate of the clean sample from $(x_t,\epsilon_\theta)$. So DDIM is not a different free energy; it is a deterministic implementation of the same learned JKO displacement field. This explains why DDIM is deterministic and can use larger step sizes: it follows a smoother transport map rather than repeatedly injecting fresh noise.

\subsubsection*{3. NCSN/SMLD: pure-entropy JKO at multiple scales}

\textbf{JKO choice}: set $V=0$ and use only entropy,
\[
\mathcal{F}_k(\rho)=\varepsilon_k\int\rho\log\rho\,dx.
\]
The proximal step is
\[
\rho_{k+1}=\arg\min_\rho\left\{\varepsilon_k\int\rho\log\rho\,dx+\frac{1}{2\tau_k}W_2^2(\rho,\rho_k)\right\},
\]
with optimality condition
\[
\frac{x-T_k(x)}{\tau_k}=-\varepsilon_k\nabla\log\rho_{k+1}(x),
\qquad x\sim\rho_{k+1}.
\]
NCSN/SMLD learns the score at a sequence of noise scales $\sigma_1>\cdots>\sigma_N$. Replacing $\nabla\log\rho_{k+1}$ by $s_\theta(x,\sigma_{k+1})$ gives the annealed Langevin/JKO particle update
\[
x_{k+1}=x_k+\eta_k s_\theta(x_{k+1},\sigma_{k+1})+\sqrt{2\eta_k}\,z,
\]
Large $\sigma_k$ corresponds to coarse entropy-driven moves; small $\sigma_k$ gives fine denoising near the data manifold.

\subsubsection*{4. Energy Matching: learning the JKO energy directly}

Energy Matching keeps the JKO iteration explicit:
\[
\rho_{k+1}=\arg\min_\rho\left\{
\int V_\theta\,d\rho+\varepsilon_k\int\rho\log\rho\,dx+\frac{1}{2\tau_k}W_2^2(\rho,\rho_k)
\right\}.
\]
Unlike DDPM/NCSN, which learn the score $\nabla\log\rho_{k+1}$, Energy Matching learns the scalar energy $V_\theta$ itself. The same $V_\theta$ plays two roles: when $\varepsilon_k$ is small, $-\nabla V_\theta$ acts as the transport velocity; near equilibrium, $V_\theta$ defines the Boltzmann density $\rho\propto e^{-V_\theta/\varepsilon_k}$.

\subsubsection*{5. Flow Matching: Benamou--Brenier OT paths}

\textbf{Choice}: Do not define the target distribution through a potential energy $V$ or a free energy $\mathcal{F}$. Instead, prescribe a source distribution $\rho_0$ (noise), a target distribution $\rho_1$ (data), and learn a velocity field that transports one into the other.

\textbf{Important distinction}: Flow Matching is not literally the JKO scheme with $\mathcal{F}=0$. If we put $\mathcal{F}=0$ into the standard JKO update
\[
\rho_{k+1}=\arg\min_\rho\left\{\mathcal{F}(\rho)+\frac{1}{2\tau}W_2^2(\rho,\rho_k)\right\},
\]
then the minimizer is simply $\rho_{k+1}=\rho_k$---nothing moves. A Wasserstein geodesic appears only after we also prescribe the endpoint $\rho_1$. Thus OT Flow Matching is better viewed through the \textbf{Benamou--Brenier dynamic OT formulation}:
\[
W_2^2(\rho_0,\rho_1)
=
\inf_{\substack{\partial_t\rho_t+\nabla\cdot(\rho_t v_t)=0\\
\rho_0,\rho_1\ \text{fixed}}}
\int_0^1\!\int |v_t(x)|^2\rho_t(x)\,dx\,dt.
\]

The geodesic from $\rho_0$ to $\rho_1$ in Wasserstein space is precisely the \textbf{McCann displacement interpolation}:
\[
\rho_t = [(1-t)\mathrm{id} + t\,T]_\#\rho_0, \quad T = \nabla\varphi\;\text{(Brenier map)}
\]
Each particle moves at constant speed along the straight line $X_t = (1-t)x_0 + t\,T(x_0)$, with velocity $v_t(X_t) = T(x_0)-x_0$.

\textbf{From dynamic OT to Flow Matching}: In practice, $T$ is unknown, so Flow Matching approximates it using \textbf{conditional optimal transport} (mini-batch OT):
\[
v_t(x) \approx \frac{x_1 - x_0}{1}, \quad x_0\sim\rho_0,\;x_1\sim\rho_1,\;(x_0,x_1)\text{ paired by OT}
\]
A neural network $v_\theta(x,t)$ is trained to fit this velocity field:
\[
\boxed{\mathcal{L}_{\text{FM}} = \mathbb{E}_{t,(x_0,x_1)}\left[\|v_\theta(x_t,t) - (x_1-x_0)\|^2\right], \quad x_t = (1-t)x_0 + t\,x_1}
\]

\textbf{Variational perspective}: OT Flow Matching learns the velocity field of the Benamou--Brenier minimizer. There is no free energy whose gradient is being followed, and there is no equilibrium distribution $\pi\propto e^{-V}$. The target information enters through the boundary condition $\rho_1=p_{\text{data}}$, not through a potential $V$.

\textbf{Same endpoint, different path}: In the idealized limit of infinite model capacity, perfect training, and exact numerical integration, both diffusion models and Flow Matching can transport a simple base distribution to the data distribution. The endpoint can be the same:
\[
\rho_1=p_{\text{data}}.
\]
But the paths are different. Diffusion constructs a stochastic noising/denoising process and follows the free-energy/Fokker--Planck geometry; OT Flow Matching directly learns a deterministic transport path, often close to a Wasserstein geodesic. Thus Flow Matching removes some SDE machinery that is unnecessary if the only goal is to learn a transport map, while diffusion retains useful extra structure: scores, reverse SDEs, Langevin sampling, likelihood tools, and inverse-problem machinery.

\textbf{Can we recover a potential from a learned Flow Matching velocity?}
Suppose a Flow Matching model gives a velocity field $v_t(x)$ and the induced density path $p_t(x)$. If we try to reinterpret it as a free-energy gradient flow with entropy coefficient $\varepsilon$, we would need
\[
v_t(x)
=
-\nabla V_t(x)-\varepsilon\nabla\log p_t(x).
\]
Formally this means
\[
\nabla V_t(x)
=
-v_t(x)-\varepsilon\nabla\log p_t(x).
\]
This defines a scalar potential $V_t$ only if the right-hand side is a conservative vector field:
\[
\nabla\times\bigl(-v_t-\varepsilon\nabla\log p_t\bigr)=0.
\]
Since $\nabla\log p_t$ is already a gradient field, the obstruction is the rotational component of $v_t$. A general neural velocity field need not be curl-free, so a scalar potential may not exist. Even when it exists, it is usually a \textbf{time-dependent effective potential} $V_t(x)$, not a fixed equilibrium potential $V(x)$ with $\pi\propto e^{-V}$.

\subsubsection*{Unified summary}

\begin{center}
\renewcommand{\arraystretch}{1.4}
\begin{tabular}{@{}p{2.2cm}p{3cm}p{2.5cm}p{2.5cm}p{2.7cm}@{}}
\toprule
\textbf{Algorithm} & \textbf{Variational principle} & \textbf{Stochasticity} & \textbf{Geometric role} & \textbf{Generative path} \\
\midrule
DDPM & Free energy $\int\rho\log\rho + \int V\rho$ & Yes (Langevin) & JKO gradient flow & Along $\mathcal{F}$ gradient + noise \\[4pt]
DDIM & Same free energy, deterministic probability flow & No & Deterministic transport & Prob.\ Flow ODE \\[4pt]
NCSN & Entropy $\int\rho\log\rho$ & Yes & Multi-scale gradient flow & Along entropy gradient + noise \\[4pt]
Flow Matching & Benamou--Brenier kinetic action & No & Wasserstein geodesic & OT straight line \\[4pt]
Energy Matching & $\int V_\theta\rho + \varepsilon\int\rho\log\rho$ & Depends on $\varepsilon$ & Variable-temperature JKO & OT$\to$Boltzmann \\
\bottomrule
\end{tabular}
\end{center}

\textbf{Key insight}: JKO and Benamou--Brenier are two complementary variational principles on Wasserstein space:
\begin{enumerate}
\item \textbf{JKO} is an initial-value problem: given $\rho_k$, descend a free energy $\mathcal{F}$ while staying close in $W_2$.
\item \textbf{Benamou--Brenier} is a boundary-value problem: given both $\rho_0$ and $\rho_1$, minimize kinetic energy among all mass-preserving paths.
\item \textbf{Diffusion models} live naturally in the JKO/free-energy picture; \textbf{OT Flow Matching} lives naturally in the Benamou--Brenier/geodesic picture.
\item \textbf{Probability flow ODEs} are the deterministic continuity-equation representation of diffusion: the entropy-driven diffusion term is absorbed into the velocity field through the score.
\end{enumerate}

\section{The Big Picture}

Let us review the logical chain of the entire story:

\begin{summary}
\begin{enumerate}
    \item \textbf{Wasserstein distance} $W_2$ gives $\Prob_2(\R^d)$ the structure of a metric space --- in fact, a (formal) Riemannian manifold.
    
    \item The \textbf{continuity equation} $\partial_t\rho + \nabla\cdot(\rho v) = 0$ describes how probability mass flows. It is the equation of motion in Wasserstein space.
    It does not require an energy functional; it only requires a velocity field.
    
    \item The \textbf{Benamou--Brenier formula} reveals that $W_2$ = geodesic distance in this Riemannian structure, with kinetic energy $\int|v|^2\rho\,dx$ as the metric.
    
    \item The \textbf{Fokker--Planck equation} is the Wasserstein gradient flow of the free energy $\mathcal{F}(\rho) = \KL(\rho\|\pi)$:
    \[
        \text{``}\dot\rho = -\grad_W\mathcal{F}(\rho)\text{''}
    \]
    This is the infinite-dimensional analog of $\dot x = -\nabla f(x)$.
    In this special case, the velocity is not arbitrary: it is determined by the free energy.
    
    \item The \textbf{JKO scheme} discretizes this gradient flow:
    \[
        \rho_{k+1} = \arg\min_\rho\left\{\mathcal{F}(\rho) + \frac{W_2^2(\rho,\rho_k)}{2\tau}\right\}
    \]
    As $\tau\to 0$, it recovers the continuous-time Fokker--Planck equation.
\end{enumerate}
\end{summary}

\textbf{One-sentence summary:}

\begin{center}
\fbox{\parbox{0.85\textwidth}{\centering
Wasserstein distance endows probability space with geometry;\\
the continuity equation describes mass-conserving flow;\\
the Fokker--Planck equation is the free-energy-driven special case;\\
Flow Matching learns a transport velocity, often the OT geodesic velocity;\\
the JKO scheme discretizes this continuous ``descent'' into step-by-step optimization problems.
}}
\end{center}

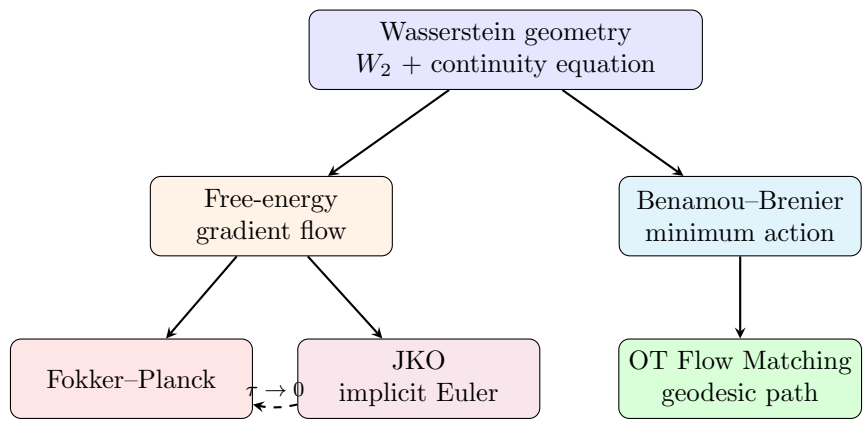
\begin{figure}[h]
\centering
\begin{tikzpicture}[
    box/.style={rectangle, draw, rounded corners, minimum height=1.05cm, minimum width=3.2cm, align=center, font=\small},
    arrow/.style={->, thick, >=stealth},
    dashedarrow/.style={->, thick, dashed, >=stealth},
    label/.style={font=\scriptsize, align=center}
]
\node[box, fill=blue!10, minimum width=5.2cm] (G) at (0,0)
{Wasserstein geometry\\$W_2$ + continuity equation};

\node[box, fill=orange!10] (GF) at (-3.1,-2.2)
{Free-energy\\gradient flow};
\node[box, fill=cyan!10] (OT) at (3.1,-2.2)
{Benamou--Brenier\\minimum action};

\node[box, fill=red!10] (FP) at (-4.95,-4.35)
{Fokker--Planck};
\node[box, fill=purple!10] (JKO) at (-1.15,-4.35)
{JKO\\implicit Euler};
\node[box, fill=green!15] (FM) at (3.1,-4.35)
{OT Flow Matching\\geodesic path};

\draw[arrow] (G) -- (GF);
\draw[arrow] (G) -- (OT);

\draw[arrow] (GF) -- (FP);
\draw[arrow] (GF) -- (JKO);
\draw[dashedarrow] (JKO) to[bend left=12] node[above, label]{$\tau\to0$} (FP);

\draw[arrow] (OT) -- (FM);
\end{tikzpicture}
\end{figure}

\section*{References and Further Reading}

\subsection*{Optimal transport and gradient flows}
\begin{enumerate}
    \item Monge, G. (1781). \textit{M\'emoire sur la th\'eorie des d\'eblais et des remblais.} Histoire de l'Acad\'emie Royale des Sciences de Paris, 666--704.
    \item Kantorovich, L. V. (1942). \textit{On the translocation of masses.} Dokl. Akad. Nauk SSSR \textbf{37}, 199--201.
    \item Brenier, Y. (1991). \textit{Polar factorization and monotone rearrangement of vector-valued functions.} Comm. Pure Appl. Math. \textbf{44}(4), 375--417.
    \item Vaserstein, L. N. (1969). \textit{Markov processes over denumerable products of spaces, describing large systems of automata.} Problemy Peredachi Informatsii \textbf{5}(3), 64--72. [The source of the name ``Wasserstein distance.'']
    \item McCann, R. J. (1997). \textit{A convexity principle for interacting gases.} Adv. Math. \textbf{128}(1), 153--179. [Displacement interpolation.]
    \item Benamou, J.-D. \& Brenier, Y. (2000). \textit{A computational fluid mechanics solution to the Monge--Kantorovich mass transfer problem.} Numer. Math. \textbf{84}(3), 375--393.
    \item Jordan, R., Kinderlehrer, D., \& Otto, F. (1998). \textit{The variational formulation of the Fokker--Planck equation.} SIAM J. Math. Anal. \textbf{29}(1), 1--17.
    \item Otto, F. (2001). \textit{The geometry of dissipative evolution equations: the porous medium equation.} Comm. PDE \textbf{26}(1-2), 101--174.
    \item Ambrosio, L., Gigli, N., \& Savar\'e, G. (2008). \textit{Gradient Flows in Metric Spaces and in the Space of Probability Measures.} Birkh\"auser. [The definitive reference.]
    \item Villani, C. (2003). \textit{Topics in Optimal Transportation.} AMS. [Excellent introduction.]
    \item Santambrogio, F. (2015). \textit{Optimal Transport for Applied Mathematicians.} Birkh\"auser. [Very readable.]
    \item Peyr\'e, G. \& Cuturi, M. (2019). \textit{Computational Optimal Transport.} Found. Trends ML \textbf{11}(5-6), 355--607. [Computational perspective.]
    \item Lavenant, H. \& Santambrogio, F. (2022). \textit{The flow map of the Fokker--Planck equation does not provide optimal transport.} Appl. Math. Lett. \textbf{133}, 108225.
    \item Mokrov, P., Korotin, A., Li, L., Genevay, A., Solomon, J., \& Burnaev, E. (2021). \textit{Large-scale Wasserstein gradient flows.} NeurIPS.
    \item Xu, C., Cheng, X., \& Xie, Y. (2023). \textit{Normalizing flow neural networks by JKO scheme.} NeurIPS. [JKO-iFlow.]
\end{enumerate}

\subsection*{Generative models}
\begin{enumerate}
    \item Hyv\"arinen, A. (2005). \textit{Estimation of non-normalized statistical models by score matching.} J. Mach. Learn. Res. \textbf{6}, 695--709.
    \item Vincent, P. (2011). \textit{A connection between score matching and denoising autoencoders.} Neural Comput. \textbf{23}(7), 1661--1674.
    \item Welling, M. \& Teh, Y. W. (2011). \textit{Bayesian learning via stochastic gradient Langevin dynamics.} ICML, 681--688.
    \item Song, Y. \& Ermon, S. (2019). \textit{Generative modeling by estimating gradients of the data distribution.} NeurIPS. [NCSN/SMLD.]
    \item Ho, J., Jain, A., \& Abbeel, P. (2020). \textit{Denoising diffusion probabilistic models.} NeurIPS. [DDPM.]
    \item Song, J., Meng, C., \& Ermon, S. (2021). \textit{Denoising diffusion implicit models.} ICLR. [DDIM.]
    \item Song, Y., Sohl-Dickstein, J., Kingma, D. P., Kumar, A., Ermon, S., \& Poole, B. (2021). \textit{Score-based generative modeling through stochastic differential equations.} ICLR.
    \item Karras, T., Aittala, M., Aila, T., \& Laine, S. (2022). \textit{Elucidating the design space of diffusion-based generative models.} NeurIPS. [EDM.]
    \item Lipman, Y., Chen, R. T. Q., Ben-Hamu, H., Nickel, M., \& Le, M. (2023). \textit{Flow matching for generative modeling.} ICLR.
    \item Balcerak, M., et al. (2025). \textit{Energy Matching: unifying flow matching and energy-based models for generative modeling.} arXiv:2504.10612 (NeurIPS 2025).
    \item Wang, R., et al. (2025). \textit{Equilibrium Matching: generative modeling with implicit energy-based models.} arXiv:2510.02300.
    \item Albergo, M. S., Boffi, N. M., \& Vanden-Eijnden, E. (2025). \textit{Stochastic interpolants: a unifying framework for flows and diffusions.} J. Mach. Learn. Res. \textbf{26}(209), 1--80.
    \item Lipman, Y., et al. (2024). \textit{Flow Matching Guide and Code.} arXiv:2412.06264.
    \item Vuong, A. B., McCann, M. T., Santos, J. E., \& Lin, Y. T. (2025). \textit{Are we really learning the score function? Reinterpreting diffusion models through Wasserstein gradient flow matching.} CIKM. (arXiv:2509.00336)
    \item Holderrieth, P. \& Erives, E. (2025). \textit{An Introduction to Flow Matching and Diffusion Models.} MIT 6.S184 lecture notes.
\end{enumerate}

\subsection*{Classical foundations}
\begin{enumerate}
    \item Fisher, R. A. (1925). \textit{Theory of statistical estimation.} Proc. Cambridge Philos. Soc. \textbf{22}, 700--725.
    \item Langevin, P. (1908). \textit{Sur la th\'eorie du mouvement brownien.} C. R. Acad. Sci. Paris \textbf{146}, 530--533.
    \item Uhlenbeck, G. E. \& Ornstein, L. S. (1930). \textit{On the theory of the Brownian motion.} Phys. Rev. \textbf{36}(5), 823--841.
    \item Kolmogorov, A. N. (1931). \textit{\"Uber die analytischen Methoden in der Wahrscheinlichkeitsrechnung.} Math. Ann. \textbf{104}, 415--458. [Kolmogorov forward/backward equations.]
    \item It\^o, K. (1951). \textit{On stochastic differential equations.} Mem. Amer. Math. Soc. \textbf{4}, 1--51.
    \item Kullback, S. \& Leibler, R. A. (1951). \textit{On information and sufficiency.} Ann. Math. Statist. \textbf{22}(1), 79--86.
    \item McKean, H. P. (1966). \textit{A class of Markov processes associated with nonlinear parabolic equations.} Proc. Natl. Acad. Sci. USA \textbf{56}(6), 1907--1911. [McKean--Vlasov.]
    \item Anderson, B. D. O. (1982). \textit{Reverse-time diffusion equation models.} Stochastic Process. Appl. \textbf{12}(3), 313--326.
\end{enumerate}

\appendix

\section{Measures and Couplings}\label{app:measures}

Before introducing the Wasserstein distance, we need to clarify two foundational concepts: \textbf{measures} and \textbf{couplings}.
If you are already familiar with measure theory, you may skip this section.

\subsection{What is a measure?}

\textbf{The most intuitive understanding:} A measure is a way of ``assigning size to sets.''

There are many notions of ``size'' that we use in everyday life:
\begin{itemize}
    \item The \textbf{length} of a line segment: $[0,3]$ has length 3
    \item The \textbf{area} of a region: a circle of radius $r$ has area $\pi r^2$
    \item The \textbf{volume} of a solid: a cube has volume $l^3$
\end{itemize}
These are all special cases of the Lebesgue measure. But measures go far beyond these---they provide a unified framework for expressing various ways of ``distributing mass.''

\begin{definition}[Measure --- informal]
A measure $\mu$ on $\R^d$ is a function that assigns a non-negative number $\mu(A) \geq 0$ to each (measurable) subset $A \subseteq \R^d$, satisfying:
\begin{enumerate}
    \item $\mu(\emptyset) = 0$ (empty set has zero measure)
    \item \textbf{Countable additivity:} If $A_1, A_2, \ldots$ are disjoint, then $\mu\!\left(\bigcup_{i=1}^\infty A_i\right) = \sum_{i=1}^\infty \mu(A_i)$
\end{enumerate}
A \textbf{probability measure} additionally satisfies $\mu(\R^d) = 1$ (total mass is 1).
\end{definition}

\textbf{Three measures you must know:}

\begin{enumerate}[label=\textbf{(\arabic*)}]
    \item \textbf{Lebesgue measure} $\lambda$: This is simply ``volume.'' $\lambda([0,2]\times[0,3]) = 6$. It uniformly assigns ``mass'' to every region of space.
    
    \item \textbf{Dirac measure} $\delta_x$: Concentrates all mass at a single point $x$:
    \[
    \delta_x(A) = \begin{cases} 1 & \text{if } x \in A \\ 0 & \text{if } x \notin A\end{cases}
    \]
    Think of it as a single grain of sand placed at position $x$.
    
    \item \textbf{Absolutely continuous measure}: A measure with a density function $\rho(x)$: $\mu(A) = \int_A \rho(x)\,dx$.
    
    For example, the normal distribution $\mathcal{N}(0,1)$: $\mu(A) = \int_A \frac{1}{\sqrt{2\pi}}e^{-x^2/2}\,dx$.
    
    Intuition: $\rho(x)$ describes the ``concentration'' of mass in space---where $\rho(x)$ is large the mass is dense, where $\rho(x)$ is small the mass is sparse.
\end{enumerate}

\begin{intuition}
\textbf{Measures vs.\ density functions:}

Beginners often ask: ``What is the difference between a measure and a probability density function (PDF)?''

The answer: \textbf{a PDF is one way of representing a measure}, but not every measure has a PDF.
\begin{itemize}
    \item $\delta_x$ (a single grain of sand) has no PDF---you cannot write down a function $\rho(x)$ such that $\int_A\rho\,dx = \delta_x(A)$
    \item The normal distribution does have a PDF: $\rho(x) = \frac{1}{\sqrt{2\pi}}e^{-x^2/2}$
    \item The mixture $\frac{1}{2}\delta_0 + \frac{1}{2}\mathcal{N}(1,1)$ also has no PDF (in the classical sense)
\end{itemize}
Working with measures instead of PDFs allows us to handle all these cases in a unified framework.
\end{intuition}

\subsection{Absolutely continuous: what does it really mean?}

In optimal transport, you will repeatedly encounter the condition: ``$\mu$ is absolutely continuous with respect to the Lebesgue measure.''
What does this mean? Why is it important?

\begin{definition}[Absolute continuity of measures]
A measure $\mu$ is \textbf{absolutely continuous} with respect to another measure $\lambda$ (written $\mu \ll \lambda$) if:
\[
    \lambda(A) = 0 \implies \mu(A) = 0 \quad \text{for all measurable sets } A
\]
In words: every set that is ``invisible'' to $\lambda$ is also ``invisible'' to $\mu$.
\end{definition}

\textbf{In plain language:} If a set has ``zero volume'' ($\lambda(A)=0$), then $\mu$ also places no mass on it ($\mu(A)=0$).

\begin{intuition}
\textbf{Everyday analogy: Spreading butter vs.\ placing pebbles}

Imagine distributing something on a slice of bread ($\R^d$ space). There are two ways:
\begin{itemize}
    \item \textbf{Spreading butter} (absolutely continuous): Butter is spread uniformly or non-uniformly on the bread surface.
    If you cut along a line of zero area, there will be no butter on it---because butter is ``continuously spread,'' and no area means no butter.
    \item \textbf{Placing pebbles} (not absolutely continuous): You place several pebbles on the bread.
    The pebbles occupy ``zero area'' (points have no area), yet they have finite mass.
    This violates absolute continuity---there exists a set of zero area (the points where the pebbles sit), yet it carries positive mass.
\end{itemize}

\textbf{The essence of absolute continuity is: mass is spread like butter across space, not piled up like pebbles at isolated points.}
\end{intuition}

\textbf{A finer understanding: What are ``sets of zero volume''?}

In spaces of different dimensions, ``zero volume'' means different things:
\begin{itemize}
    \item In $\R^1$ (one-dimensional/line): Points, countable sets (e.g., the rationals $\mathbb{Q}$), and the Cantor set are all null sets
    \item In $\R^2$ (two-dimensional/plane): Points, line segments, and curves are all null sets (they have no ``area'')
    \item In $\R^3$ (three-dimensional/space): Points, line segments, and surfaces are all null sets (they have no ``volume'')
\end{itemize}
General rule: In $\R^d$, subsets of dimension $< d$ are typically null.

\textbf{Key corollary:} An absolutely continuous measure $\mu$ on $\R^d$ satisfies $\mu(\{x\})=0$ for all individual points $x$.
That is, \textbf{the mass at any single point is zero}---there are no ``atoms.''
Mass must be ``spread out'' over regions with positive volume.

\begin{theorem}[Radon--Nikodym]
$\mu \ll \lambda$ if and only if there exists a non-negative measurable function $\rho:\R^d\to[0,\infty)$ such that:
\begin{equation}
    \mu(A) = \int_A \rho(x)\,dx \quad \text{for all measurable } A
\end{equation}
The function $\rho$ is called the \textbf{Radon--Nikodym derivative} (or \textbf{density}) of $\mu$ w.r.t.\ $\lambda$, written $\rho = \frac{d\mu}{d\lambda}$.
\end{theorem}

\textbf{Significance of the Radon--Nikodym theorem:}

``$\mu$ is absolutely continuous'' $\iff$ ``$\mu$ has a density function $\rho(x)$''

These two statements are \textbf{completely equivalent}! So whenever you see ``$\mu$ is absolutely continuous with respect to Lebesgue measure,'' you can mentally translate it to ``$\mu$ has a PDF $\rho(x)$.''

\textbf{Intuition for the density $\rho(x)$:} $\rho(x)$ describes how much mass is contained per unit volume near $x$:
\[
\mu(\text{small ball $B_\epsilon(x)$ centered at $x$}) \approx \rho(x)\cdot\text{Vol}(B_\epsilon(x))
\]
Where $\rho(x)$ is large the mass is dense, where $\rho(x)$ is small the mass is sparse, and where $\rho(x)=0$ there is no mass at all.
The key point is: \textbf{no matter how large $\rho(x)$ is, the mass at a single point $\{x\}$ is always zero} (because $\text{Vol}(\{x\})=0$).

\textbf{Caution: ``Absolutely continuous'' can be confused with another concept}

In mathematics, ``absolutely continuous'' has \textbf{two distinct usages}:
\begin{enumerate}
    \item \textbf{Absolute continuity of measures} (what we are discussing): $\mu \ll \lambda$, meaning one measure does not concentrate on null sets of another.
    \item \textbf{Absolute continuity of functions}: A function $f:[a,b]\to\R$ is absolutely continuous, roughly meaning ``the variation of $f$ can be controlled by integration''---weaker than Lipschitz, stronger than uniform continuity.
\end{enumerate}
The two are deeply connected ($f$ is absolutely continuous $\iff$ $f'$ exists a.e.\ and $f(x)=f(a)+\int_a^x f'(t)dt$, which is essentially saying the measure induced by $f$ is absolutely continuous with respect to Lebesgue measure), but in the context of optimal transport we always mean the first.

\textbf{Why does optimal transport care about absolute continuity?}

Recall the Brenier theorem: If $\mu$ is absolutely continuous (has a density), then the optimal transport map $T$ exists and is unique.
Why is this condition needed? The intuition is as follows:

\textbf{Core idea: ``Infinitesimal mass does not need to be split''}

A map $T$ requires that \textbf{all} mass at position $x$ goes to the same destination $T(x)$.
\begin{itemize}
    \item If $\mu$ has an ``atom'' (finite mass $\mu(\{x_0\})>0$ at some point $x_0$),
          and $\nu$ requires this mass to be distributed to multiple locations, then a map $T$ cannot achieve this---
          since $T(x_0)$ can only be a single point.
    \item But if $\mu$ is absolutely continuous, then every point $x$ carries only ``infinitesimal'' mass $\rho(x)dx$.
          This infinitesimal mass going entirely to one destination $T(x)$ is perfectly fine---
          since each portion of mass from the source is infinitesimal, there is no need to ``split'' it for the map $T$ to transport it to a unique destination.
          The target $\nu$ is formed by ``accumulating'' infinitesimal masses sent from many different $x$.
          Note: $\nu$ itself \textbf{need not} be absolutely continuous---the Brenier theorem only requires $\mu$ to be absolutely continuous.
\end{itemize}

\textbf{Analogy: Distributing fruit}
\begin{itemize}
    \item 10 apples (discrete/atomic) must be divided among 3 people, each requiring 3.33---\textbf{cannot divide into integers}, must cut (split).
    \item 10 liters of juice (continuous/absolutely continuous) must be divided among 3 people, each requiring 3.33 liters---\textbf{easy to divide}, liquid can be poured to arbitrary precision.
\end{itemize}
An absolutely continuous measure is like a liquid: it can be ``poured'' by the map $T$ to any destination without needing to be ``cut.''

\begin{example}
\textbf{Classification exercises:}
\begin{enumerate}
    \item $\mu = \text{Uniform}[0,1]$: absolutely continuous~\checkmark~(density $\rho(x) = \mathbf{1}_{[0,1]}(x)$)
    \item $\mu = \mathcal{N}(0,1)$: absolutely continuous~\checkmark~(density $\rho(x) = \frac{1}{\sqrt{2\pi}}e^{-x^2/2}$)
    \item $\mu = \delta_0$: not absolutely continuous~$\times$~(mass concentrated on the null set $\{0\}$)
    \item $\mu = \frac{1}{2}\delta_0 + \frac{1}{2}\text{Uniform}[0,1]$: not absolutely continuous~$\times$~(has an atom at $\{0\}$)
    \item The measure on $\R^2$ uniformly distributed on the segment $\{(x,0):x\in[0,1]\}$: not absolutely continuous~$\times$
    
    (The segment has Lebesgue measure zero in $\R^2$, yet this measure places all its mass there)
    \item $\mu$ with density $\rho(x) = \frac{1}{|x|^{1/2}}\mathbf{1}_{[-1,1]}(x)$ (density tends to $\infty$ at $x=0$): absolutely continuous~\checkmark
    
    (Although $\rho(0)=\infty$, we have $\mu(\{0\})=\int_{\{0\}}\rho\,dx = 0$---the density can ``blow up'' at certain points, as long as it remains integrable)
\end{enumerate}
\end{example}

\textbf{One subtle point: the density can be $\infty$ yet the measure is still absolutely continuous}

Beginners often have a misconception: ``the density $\rho(x)$ being large at some point'' means ``there is a large concentration of mass at that point.''
This is \textbf{wrong}!

$\rho(x) = 100$ means there is a lot of mass per \textbf{unit volume} near $x$, but the mass at the single point $\{x\}$ is still zero.
Think of a glass of concentrated sugar water---the concentration in some region is very high, but if you take a single drop (volume $\to 0$), the amount of sugar also tends to zero.

What actually causes a measure to be ``not absolutely continuous'' is an \textbf{atom}: a finite amount of mass concentrated on a set of zero volume.
This cannot be represented within the density function framework---because $\int_{\{x\}}\rho(x)\,dx$ is always zero, no matter how large $\rho(x)$ is.

\subsection{Integration against a measure}

When we write $\int f(x)\,d\mu(x)$, it means ``weighted summation of $f$ using $\mu$'':
\begin{itemize}
    \item If $\mu$ has density $\rho$: $\int f\,d\mu = \int f(x)\rho(x)\,dx$ (ordinary integral)
    \item If $\mu = \delta_a$: $\int f\,d\mu = f(a)$ (point evaluation)
    \item If $\mu = \sum_i w_i\delta_{x_i}$ (discrete measure): $\int f\,d\mu = \sum_i w_i f(x_i)$ (weighted sum)
\end{itemize}
So $\int f\,d\mu$ is an extremely flexible notation---it unifies ``integration,'' ``evaluation,'' and ``weighted summation.''

\subsection{Pushforward measure}

\begin{definition}[Pushforward]
Given a measure $\mu$ on $\R^d$ and a map $T:\R^d\to\R^d$, the \textbf{pushforward} $T_\#\mu$ is a measure on $\R^d$ defined by:
\[
    (T_\#\mu)(A) := \mu(T^{-1}(A)) = \mu(\{x : T(x)\in A\})
\]
\end{definition}

\textbf{Intuition:} Imagine you have a bag of marbles (distributed as $\mu$), and you apply the transformation $T$ to every marble---moving the marble at position $x$ to $T(x)$.
The new distribution of marbles after the transformation is $T_\#\mu$.

\textbf{Equivalent formulation} (using integrals): For any function $f$,
\[
\int f(y)\,d(T_\#\mu)(y) = \int f(T(x))\,d\mu(x)
\]
Left side: integrate against the new distribution. Right side: integrate against the old distribution with a change of variables.

\textbf{Examples:}
\begin{itemize}
    \item $\mu = \delta_0$, $T(x) = x+1$. Then $T_\#\mu = \delta_1$ (point moves from 0 to 1).
    \item $\mu = \text{Uniform}[0,1]$, $T(x) = 2x$. Then $T_\#\mu = \text{Uniform}[0,2]$ (stretching).
    \item $\mu = \mathcal{N}(0,1)$, $T(x) = \sigma x + m$. Then $T_\#\mu = \mathcal{N}(m,\sigma^2)$.
\end{itemize}

\subsection{Coupling: the key concept for optimal transport}

We now arrive at the most important concept of this section---\textbf{coupling}.

\begin{definition}[Coupling]
Given two probability measures $\mu$ on $\mathcal{X}$ and $\nu$ on $\mathcal{Y}$, a \textbf{coupling} of $(\mu,\nu)$ is a joint probability measure $\gamma$ on $\mathcal{X}\times\mathcal{Y}$ whose marginals are $\mu$ and $\nu$:
\begin{align}
    \text{First marginal:}&\quad \gamma(A\times\mathcal{Y}) = \mu(A) \quad\forall A\\
    \text{Second marginal:}&\quad \gamma(\mathcal{X}\times B) = \nu(B) \quad\forall B
\end{align}
\end{definition}

\textbf{What does this mean? Let us understand it on three levels:}

\textbf{Level one: The discrete case (most intuitive)}

Suppose $\mu$ and $\nu$ are both discrete distributions:
\begin{itemize}
    \item $\mu$: Factory A produces 3 tons, Factory B produces 2 tons (total 5 tons)
    \item $\nu$: City 1 demands 1 ton, City 2 demands 4 tons (total 5 tons)
\end{itemize}

A coupling $\gamma$ is a \textbf{transport plan table}:
\[
\begin{array}{c|cc|c}
 & \text{City 1} & \text{City 2} & \text{Row sum}=\mu \\
\hline
\text{Factory A} & \gamma_{A1} & \gamma_{A2} & 3 \\
\text{Factory B} & \gamma_{B1} & \gamma_{B2} & 2 \\
\hline
\text{Col sum}=\nu & 1 & 4 & 5
\end{array}
\]

The constraints are:
\begin{itemize}
    \item Each row sum = the factory's total production (marginal is $\mu$)
    \item Each column sum = the city's total demand (marginal is $\nu$)
\end{itemize}

For example, one feasible plan is $\gamma_{A1}=1, \gamma_{A2}=2, \gamma_{B1}=0, \gamma_{B2}=2$.
Another feasible plan is $\gamma_{A1}=0, \gamma_{A2}=3, \gamma_{B1}=1, \gamma_{B2}=1$.
There are many plans satisfying the constraints---optimal transport finds the one with minimum cost.

\textbf{Level two: The random variable perspective}

If $(X,Y)$ is a joint random variable with $X\sim\mu$ and $Y\sim\nu$, then the joint distribution of $(X,Y)$ is a coupling of $\mu$ and $\nu$.

The \textbf{marginal conditions} mean: regardless of what value $Y$ takes, the (marginal) distribution of $X$ is $\mu$; and vice versa.

Note: Given marginal distributions $\mu$ and $\nu$, the joint distribution is \textbf{not unique}---it also depends on the correlation between $X$ and $Y$.
\begin{itemize}
    \item Independent coupling: $\gamma = \mu\otimes\nu$ ($X$ and $Y$ are independent)
    \item Perfect correlation: $\gamma = (\text{id}, T)_\#\mu$ ($Y = T(X)$ is deterministically determined by $X$)
    \item Various intermediate cases
\end{itemize}
The set of couplings $\Gamma(\mu,\nu)$ encompasses all these possibilities.

\textbf{Level three: The transport plan perspective}

Returning to the earth-moving metaphor. $\gamma(x,y)$ represents ``the amount of earth moved from position $x$ to position $y$.''

The marginal conditions guarantee:
\begin{itemize}
    \item The total amount moved from $x$ = the amount of earth at $\mu(x)$ (all earth at $x$ is moved away)
    \item The total amount arriving at $y$ = the amount needed at $\nu(y)$ (all demand at $y$ is satisfied)
\end{itemize}

\begin{example}[Couplings between two Dirac masses]
Let $\mu = \delta_a$, $\nu = \delta_b$. The unique coupling is $\gamma = \delta_{(a,b)}$---all mass goes from $a$ to $b$, with no other choice.
\end{example}

\begin{example}[Couplings between uniform distributions]
Let $\mu = \nu = \text{Uniform}[0,1]$. Some examples of couplings:
\begin{itemize}
    \item \textbf{Independent coupling:} $\gamma = $ uniform distribution on $[0,1]^2$ (scatter points uniformly in the square)
    \item \textbf{Identity coupling:} $\gamma$ is concentrated on the diagonal $\{(x,x):x\in[0,1]\}$ (every point $x$ stays put)
    \item \textbf{Reflection coupling:} $\gamma$ is concentrated on the anti-diagonal $\{(x,1-x):x\in[0,1]\}$ ($x$ is moved to $1-x$)
\end{itemize}
The ``transport costs'' $\int|x-y|^2\,d\gamma$ for these three plans are $\frac{1}{6}$, $0$, and $\frac{1}{3}$, respectively.
The optimal one is the identity coupling (cost 0), which makes sense---if the two distributions are identical, the best ``transport plan'' is to move nothing.
\end{example}

\textbf{Understanding the ``transport style'' of the independent coupling:}

The independent coupling $\gamma = \mu\otimes\nu$ is the most ``chaotic'' transport plan. It means:

\textbf{Each grain of sand forgets where it came from and is randomly thrown to some location in the target distribution $\nu$.}

Specifically, when $\mu=\nu=\text{Uniform}[0,1]$, $\gamma$ is uniformly distributed on the square $[0,1]^2$. This means:
\begin{itemize}
    \item Sand at position $x=0.2$ might be moved to $y=0.1$, or might be moved to $y=0.9$---
    the destination $y$ is uniformly distributed on $[0,1]$, completely independent of the origin $x$.
    \item Sand at position $x=0.8$ behaves the same---the destination is uniformly random.
    \item There is \textbf{no correlation whatsoever} between origin and destination (statistical independence).
\end{itemize}

\textbf{Why is this a valid transport plan?} Verify the marginal conditions:
\begin{align*}
    \gamma(A\times[0,1]) &= \text{Uniform}_{[0,1]^2}(A\times[0,1]) = |A| = \mu(A) \;\checkmark\\
    \gamma([0,1]\times B) &= \text{Uniform}_{[0,1]^2}([0,1]\times B) = |B| = \nu(B) \;\checkmark
\end{align*}
So the marginals are indeed $\mu$ and $\nu$---the total amount moved from each $x$ is correct, and the total arriving at each $y$ is also correct. The transport is simply extremely chaotic.

\textbf{Computing the cost} $\int|x-y|^2\,d\gamma$: Since $x$ and $y$ are \textbf{independent} under the independent coupling,
\begin{align*}
    \int_{[0,1]^2}|x-y|^2\,dxdy &= \int_0^1\int_0^1(x-y)^2\,dxdy\\
    &= \int_0^1\int_0^1(x^2 - 2xy + y^2)\,dxdy\\
    &= \E[X^2] - 2\E[X]\E[Y] + \E[Y^2]\\
    &= \frac{1}{3} - 2\cdot\frac{1}{2}\cdot\frac{1}{2} + \frac{1}{3} = \frac{1}{6}
\end{align*}
(where $\E[X^2] = \int_0^1 x^2 dx = \frac{1}{3}$, $\E[X]=\frac{1}{2}$.)

\textbf{Comparing the ``transport styles'' of the three couplings:}

\begin{center}
\renewcommand{\arraystretch}{1.3}
\begin{tabular}{@{}lll@{}}
\toprule
\textbf{Coupling} & \textbf{Transport style} & \textbf{Cost} \\
\midrule
Identity $\gamma = (\text{id},\text{id})_\#\mu$ & Every grain stays in place & $0$ \\
Independent $\gamma = \mu\otimes\nu$ & Every grain is randomly thrown to an arbitrary location & $1/6$ \\
Reflection $\gamma = (\text{id}, 1-\text{id})_\#\mu$ & Every grain is moved to the ``opposite side'' & $1/3$ \\
\bottomrule
\end{tabular}
\end{center}

\textbf{In one sentence:} The independent coupling has nonzero cost because it ``forgets'' the connection between origin and destination---
even though nothing needs to move (cost zero), it insists on randomly reshuffling the sand, needlessly increasing the transport distance.
This is why the independent coupling is typically a \textbf{very poor} transport plan (though still valid).

\begin{keypoint}
\textbf{Coupling = encoding a transport plan via a joint distribution}

This is the most crucial intuition for understanding optimal transport. A joint distribution $\gamma(x,y)$ simultaneously tells you two things:
\begin{enumerate}
    \item \textbf{Who is paired with whom:} The larger the ``density'' of $\gamma$ at $(x,y)$, the more mass is transported from $x$ to $y$
    \item \textbf{The pairing satisfies supply-demand balance:} The marginal conditions ensure every source point's mass is fully allocated, and every target point's demand is fully met
\end{enumerate}

Different joint distributions = different transport strategies:
\begin{itemize}
    \item $\gamma$ concentrated on the diagonal $x=y$ $\Rightarrow$ ``stay in place''
    \item $\gamma$ concentrated on some curve $y=T(x)$ $\Rightarrow$ ``deterministic transport'' (Monge map)
    \item $\gamma$ spread over the entire $x$-$y$ plane $\Rightarrow$ ``random/chaotic transport'' (mass from each $x$ is dispersed to multiple $y$)
\end{itemize}

A joint distribution is an \textbf{extremely flexible} representation---it can express deterministic maps ($\gamma$ degenerates onto the graph of the map),
as well as ``one-to-many'' mass splitting ($\gamma$ spreads in the $y$ direction).
This is why the Kantorovich formulation is more powerful than the Monge formulation.
\end{keypoint}

\begin{keypoint}
\textbf{The essence of coupling:}

Given two distributions $\mu$ and $\nu$, a coupling $\gamma$ describes one specific plan for ``who goes where.''
The same pair $(\mu,\nu)$ admits infinitely many couplings (transport plans), each with different cost.
\textbf{Optimal transport} finds the coupling with minimum cost.
The \textbf{Wasserstein distance} is that minimum cost (square root).
\end{keypoint}

\textbf{Why use couplings instead of maps $T$ directly?}

Monge's original idea was to find a map $T$ such that $T_\#\mu = \nu$. But a map requires ``one-to-one''---all mass at $x$ must go to a single destination.
If $\mu = \delta_0$ and $\nu = \frac{1}{2}\delta_{-1} + \frac{1}{2}\delta_1$ (one point splits into two), no such map exists!

Kantorovich's coupling relaxes this restriction: the mass at a point $x$ can be \textbf{split} among multiple destinations.
This ensures: (1) the optimization problem \textbf{always} has a solution; (2) the problem becomes \textbf{linear programming}, which is easier to handle.

\section{Differential Geometry Prerequisites}\label{app:diffgeom}

\begin{intuition}
\textbf{Why are these concepts needed?}

Later we will say ``Wasserstein space is a Riemannian manifold'' and ``Fokker--Planck is a gradient flow.''
The core tools underlying these statements are \textbf{manifolds}, \textbf{tangent spaces}, \textbf{Riemannian metrics}, and \textbf{gradients}.

If you are unfamiliar with these concepts, the following builds them from scratch using only intuition and examples.
\end{intuition}

\textbf{(1) Manifold = a space that locally looks like $\R^n$}

\begin{itemize}
\item Intuition: A manifold is a ``possibly curved space,'' but if you zoom into a small neighborhood, it looks just like ordinary $\R^n$.
\item Example 1: The surface of the Earth $S^2$---globally a sphere (curved), but locally described by latitude and longitude (like a patch of $\R^2$).
\item Example 2: $\R^n$ itself---the simplest manifold, globally flat.
\item Example 3: The probability simplex $\{(p_1,\ldots,p_n) : p_i\geq 0, \sum p_i = 1\}$---an $(n-1)$-dimensional hyperplane in $\R^n$.
\end{itemize}

\medskip
\textbf{(2) Tangent space = all possible ``instantaneous velocity directions'' at a point}

\begin{itemize}
\item Intuition: You stand at a point $p$ on the manifold; all directions \textbf{you can currently move} form a vector space---this is $T_pM$.
\item Example 1 (a point $p$ in $\R^n$): The tangent space $T_p\R^n = \R^n$. You can move in any direction without restriction.
\item Example 2 (the north pole $N$ on $S^2$): The tangent space $T_N S^2$ is the ``tangent plane'' at the north pole---an $\R^2$ plane tangent to the sphere. You cannot ``tunnel into the sphere'' or ``fly off the surface''---you can only move along the surface.
\end{itemize}

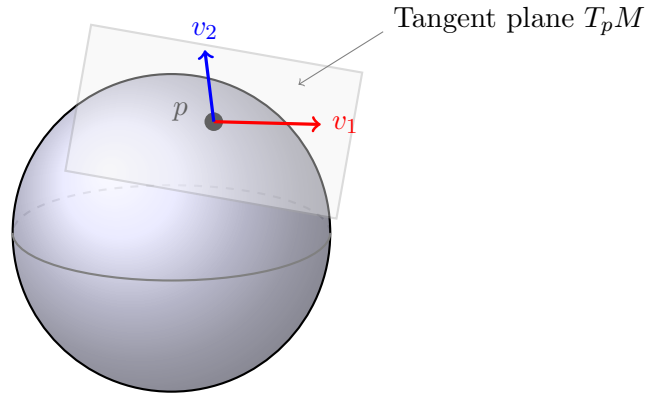
\begin{center}
\begin{tikzpicture}[scale=1.4]
    \draw[thick, dashed, gray!50] (1.5,0) arc (0:180:1.5 and 0.45);
    
    \shade[ball color=blue!15!white, opacity=0.7] (0,0) circle (1.5);
    \draw[thick] (0,0) circle (1.5);
    
    \draw[thick, gray] (-1.5,0) arc (180:360:1.5 and 0.45);
    
    \coordinate (P) at (0.4, 1.05);
    \fill (P) circle (2.5pt);
    \node[left] at (0.25, 1.15) {$p$};
    
    \begin{scope}[shift={(P)}, rotate=-10]
        \draw[thick, gray!70, fill=gray!10, opacity=0.4] 
            (-1.3,-0.7) -- (1.3,-0.7) -- (1.3,0.7) -- (-1.3,0.7) -- cycle;
        \draw[->, very thick, red] (0,0) -- (1.0, 0.15) node[right]{$v_1$};
        \draw[->, very thick, blue] (0,0) -- (-0.2, 0.65) node[above]{$v_2$};
    \end{scope}
    
    \node[right] at (2.0, 2.0) {Tangent plane $T_pM$};
    \draw[->, thin, gray] (2.0, 1.9) -- (1.2, 1.4);
    
    \node at (0,-2.3) {Tangent space $T_pM$ at a point $p$ on the sphere $M=S^2$};
    \node[below] at (0,-2.7) {\small $v_1, v_2$ are tangent vectors---directions one can move along the sphere};
\end{tikzpicture}
\end{center}

\textbf{More rigorous definitions (two equivalent approaches):}

\begin{itemize}
\item \textbf{Curve definition:} A tangent vector $v\in T_pM$ is the velocity $\dot\gamma(0)$ at $t=0$ of some curve $\gamma(t)$ passing through $p$ ($\gamma(0)=p$).
    
    Intuition: An ant stands at point $p$ and begins walking along some path $\gamma$---the instantaneous velocity at the moment of departure is a tangent vector.

\item \textbf{Directional derivative definition:} A tangent vector $v$ is an operator that ``takes directional derivatives of functions''---$v(f) := \left.\frac{d}{dt}\right|_{t=0}f(\gamma(t))$.

    Intuition: A tangent vector $v$ tells you ``how much any function $f$ changes if you take a small step in the direction $v$.''
\end{itemize}

\medskip
\textbf{(3) Key properties of tangent vectors}

\begin{itemize}
\item The tangent space $T_pM$ is a \textbf{vector space} (with addition and scalar multiplication), of dimension $= \dim M$.
\item Tangent spaces at different points are \textbf{different}: $T_p M \neq T_q M$ (one cannot directly compare tangent vectors at different points---this is why one needs ``connections'' or ``parallel transport,'' but we do not need this here).
\item The velocity of a curve $\gamma(t)$ satisfies $\dot\gamma(t) \in T_{\gamma(t)}M$---at each moment the velocity lives in the tangent space \textbf{at the point where the curve currently is}.
\end{itemize}

\medskip
\textbf{(4) Dual space and cotangent space}

\textbf{First, the dual space---a purely linear-algebraic concept:}

\begin{definition}[Dual space]
Let $V$ be a vector space over $\R$. The \textbf{dual space} $V^*$ is defined as:
\[
V^* := \{\omega : V \to \R \mid \omega\text{~is a linear map}\}
\]
Elements of $V^*$ are called \textbf{linear functionals} on $V$ or \textbf{covectors}.
\end{definition}

\textbf{Concrete example ($V = \R^n$):}

Elements of $V = \R^n$ are column vectors $v = \begin{pmatrix}v^1\\\vdots\\v^n\end{pmatrix}$.

What are elements of $V^*$? They are linear maps that ``eat a column vector and output a number.''
The most general form is: $\omega(v) = a_1 v^1 + a_2 v^2 + \cdots + a_n v^n$---
i.e., taking the inner product of $v$ with fixed coefficients $(a_1,\ldots,a_n)$.

In matrix form: $\omega = (a_1, a_2, \ldots, a_n)$ is a \textbf{row vector}, and the pairing is row $\times$ column:
\[
\omega(v) = \begin{pmatrix}a_1 & \cdots & a_n\end{pmatrix}\begin{pmatrix}v^1\\\vdots\\v^n\end{pmatrix} = \sum_{i=1}^n a_i v^i
\]

So $(\R^n)^*$ is the space of all row vectors---also $n$-dimensional, but ``living in a different place.''

\medskip
\textbf{Key distinction: Pairing vs.\ inner product}

\begin{itemize}
\item \textbf{Pairing} $\omega(v) = \sum a_i v^i$: A row vector $\omega\in V^*$ acts on a column vector $v\in V$.
This \textbf{does not require an inner product}---it is a purely algebraic operation.

\item \textbf{Inner product} $\langle u, v\rangle = \sum g_{ij}u^i v^j$: An operation between two vectors $u,v\in V$ \textbf{of the same type}.
This \textbf{requires an additional choice} of metric $g$.
\end{itemize}

Pairing is ``row $\times$ column'' (a natural operation between different types); inner product is ``column $\times$ column'' (between the same type, requiring the introduction of a metric matrix $g$ to ``flip'' one of them).

\medskip
\textbf{What is an isomorphism?}

\begin{definition}[Vector space isomorphism]
An \textbf{isomorphism} between two vector spaces $V$ and $W$ is a \textbf{linear bijection} $\varphi: V \to W$
(linear + injective + surjective).

If such a $\varphi$ exists, we say $V$ and $W$ are \textbf{isomorphic}, written $V\cong W$.
\end{definition}

Intuition: $V\cong W$ means the two spaces have ``exactly the same structure''---you can losslessly ``translate'' elements of one space into elements of the other via $\varphi$, preserving all linear relationships.

\textbf{Simple examples:}
\begin{itemize}
\item $\R^2 \cong \R^2$---the identity map $\varphi(v) = v$ is an isomorphism (trivial).
\item $\R^2 \cong \R^2$---a $90^\circ$ rotation is also an isomorphism (a different ``translation'').
\item $\R^2$ and $\{$all $2\times 1$ matrices$\}$ are isomorphic---column vectors $\leftrightarrow$ matrices, just different notation.
\item $\R^2$ and $\R^3$ are \textbf{not} isomorphic---different dimensions make a linear bijection impossible.
\end{itemize}

\textbf{Key concept: ``Natural'' isomorphisms vs.\ ``choice-dependent'' isomorphisms}

\begin{itemize}
\item \textbf{Canonical isomorphism}: An isomorphism that can be written down without making any choices.
    
    Example: $V\cong V^{**}$ (double dual) has a canonical isomorphism $v\mapsto \mathrm{ev}_v$, where $\mathrm{ev}_v(\omega):=\omega(v)$. The definition of this map involves no ``choices.''

\item \textbf{Non-canonical isomorphism}: An isomorphism that requires additional choices to specify.
    
    Example: $V\cong V^*$---one must choose a basis or an inner product to write down a specific map.
    Different choices $\Rightarrow$ different isomorphisms.
\end{itemize}

\medskip
\textbf{Where exactly is the difference between $V$ and $V^*$?}

As abstract vector spaces, $V\cong V^* \cong \R^n$ (same dimension $\Rightarrow$ an isomorphism exists). But the key point is:

\begin{itemize}
\item There is no \textbf{canonical} isomorphism $V \xrightarrow{\sim} V^*$---you cannot turn a column vector into a row vector without making a choice.
\item Choosing an inner product $g$ gives a \textbf{specific} isomorphism: $v\mapsto g(v,\cdot)$, i.e., $v^i \mapsto g_{ij}v^j$ (column to row, multiply by $g$).
\item Choosing a different inner product $\tilde g$ yields a \textbf{different} isomorphism.
\end{itemize}

\textbf{``Isn't the transpose a canonical isomorphism?''---A common misconception}

You might think: $v = \begin{pmatrix}v^1\\v^2\end{pmatrix} \mapsto v^T = (v^1, v^2)$,
isn't this a ``choice-free'' column $\to$ row map?

Answer: This map \textbf{depends on the basis you chose}. Specifically:

Let $V = \R^2$, basis $\{e_1, e_2\}$. The vector $v = 3e_1 + 2e_2$ has coordinates $\binom{3}{2}$ in this basis,
and ``transposing'' gives the row vector $(3,2)$, corresponding to the linear functional $\omega_1$ with $\omega_1(e_1)=3$, $\omega_1(e_2)=2$.

Now change to a different basis $\{e_1' = e_1+e_2,\; e_2' = e_2\}$. The same $v$ has coordinates $\binom{3}{-1}$ in the new basis
(since $v = 3e_1' + (-1)e_2'$), and ``transposing'' gives $(3,-1)$, corresponding to the linear functional $\omega_2$ with $\omega_2(e_1')=3$, $\omega_2(e_2')=-1$.

$\omega_1 \neq \omega_2$!---the same $v$, after changing basis, ``transposing'' yields a different covector.

So ``transpose'' is not a map from $V$ to $V^*$---it is a map from ``coordinates of $V$ in some basis'' to ``coordinates of $V^*$ in the dual basis.''
The moment you say ``take coordinates,'' you have already made a choice (choosing a basis).

\textbf{A more fundamental statement:} The mathematical meaning of ``natural'' is a \textbf{natural transformation in the sense of category theory}---
$V\to V^*$ is not natural because it does not commute with all linear maps $A:V\to W$.
In contrast, $V\to V^{**}$ is natural because $\mathrm{ev}$ commutes with any $A$: $(A^{**}\circ\mathrm{ev}_V)(v) = \mathrm{ev}_W(Av)$.

This distinction is invisible in $\R^n$ with the standard inner product (since $g = I$, columns and rows ``look the same''). But once $g\neq I$, or on a curved surface (where $g$ varies from point to point), the distinction becomes crucial.

\medskip
\textbf{Now back to manifolds: The cotangent space $T_p^*M$}

\begin{definition}[Cotangent space]
The \textbf{cotangent space} at the point $p$ is defined as the dual of the tangent space:
\[
T_p^*M := (T_pM)^* = \{\omega : T_pM \to \R \mid \omega\text{~is a linear map}\}
\]
Elements of $T_p^*M$ are called \textbf{covectors} or \textbf{1-forms} (at $p$).
\end{definition}

\textbf{Intuition:} A tangent vector $v\in T_pM$ is a ``direction'' (an arrow); a covector $\omega\in T_p^*M$ is a ``ruler for measuring directions''---
it accepts a direction $v$ and outputs a number $\omega(v)\in\R$.

\medskip
\textbf{Coordinate representation:} Let $(x^1,\ldots,x^n)$ be local coordinates.
\begin{itemize}
\item Natural basis of $T_pM$: $\{\partial_1,\ldots,\partial_n\}$, where $\partial_i := \frac{\partial}{\partial x^i}\big|_p$
\item Dual basis of $T_p^*M$: $\{dx^1,\ldots,dx^n\}$, defined by $dx^i(\partial_j) = \delta^i_j$
\end{itemize}
Tangent vectors are written as $v = v^i\partial_i$ (upper indices, column vectors); covectors are written as $\omega = \omega_i\,dx^i$ (lower indices, row vectors).

Pairing: $\omega(v) = \omega_i v^i$ (row $\times$ column $=$ scalar, no metric needed).

\medskip
\textbf{The differential $df$ of a function is a covector:}

\begin{definition}[Differential of a function]
Let $f : M \to \R$ be smooth. The \textbf{differential} of $f$ at $p$, $df_p \in T_p^*M$, is defined as:
\[
df_p(v) := v(f) = \left.\frac{d}{dt}\right|_{t=0}f(\gamma(t)), \quad \gamma(0)=p,\;\dot\gamma(0)=v
\]
\end{definition}

In coordinates: $df_p = \frac{\partial f}{\partial x^i}\big|_p\,dx^i$. This is a row vector $\bigl(\frac{\partial f}{\partial x^1},\ldots,\frac{\partial f}{\partial x^n}\bigr)$.

\medskip
\textbf{Key point: $df_p$ requires only the smooth structure, not a metric.} It is a ``natural'' object---
given the function $f$ and the manifold structure, $df$ is completely determined.

\medskip
\textbf{Comparison of $T_pM$ vs $T_p^*M$:}
\begin{center}
\renewcommand{\arraystretch}{1.3}
\begin{tabular}{@{}lll@{}}
\toprule
& \textbf{Tangent space $T_pM$} & \textbf{Cotangent space $T_p^*M$} \\
\midrule
Elements called & Tangent vectors & Covectors / 1-forms \\
Coordinate basis & $\partial_i = \frac{\partial}{\partial x^i}$ & $dx^i$ \\
Component notation & $v = v^i\partial_i$ (upper indices) & $\omega = \omega_i\,dx^i$ (lower indices) \\
$\R^n$ analogy & Column vectors & Row vectors \\
Typical example & Velocity of a curve $\dot\gamma(0)$ & Differential of a function $df_p$ \\
Pairing & \multicolumn{2}{c}{$\omega(v) = \omega_i v^i \in \R$ (no metric needed)} \\
\bottomrule
\end{tabular}
\end{center}

\medskip
\textbf{Why are $T_pM$ and $T_p^*M$ not ``the same thing''?}

In finite dimensions, $T_pM$ and $T_p^*M$ have the same dimension and are isomorphic as abstract vector spaces.
But there is \textbf{no canonical isomorphism}---you cannot ``turn'' a tangent vector $v\in T_pM$ into a covector $\omega\in T_p^*M$ without introducing additional structure.

If one chooses an inner product $g_p$, one obtains the \textbf{musical isomorphism}:
\begin{align*}
\flat:\; T_pM &\to T_p^*M, \quad v \mapsto v^\flat := g_p(v,\,\cdot\,)  \quad\text{(``lower indices'': column$\to$row)}\\
\sharp:\; T_p^*M &\to T_pM, \quad \omega \mapsto \omega^\sharp \quad\text{(``raise indices'': row$\to$column)}
\end{align*}
In coordinates: $(v^\flat)_i = g_{ij}v^j$ (multiply by $g$), $(\omega^\sharp)^i = g^{ij}\omega_j$ (multiply by $g^{-1}$).

\textbf{Concrete example:} Let $V = \R^2$, metric $g = \begin{pmatrix}2 & 1\\1 & 3\end{pmatrix}$.

\medskip
\textbf{Lowering indices $\flat$} (tangent vector $\to$ covector): Take $v = \binom{1}{2}\in T_pM$.

The definition of $v^\flat$ is: $v^\flat(w) := g(v, w)$ for all $w$. In coordinates:
\[
(v^\flat)_i = g_{ij}v^j: \qquad
v^\flat = \begin{pmatrix}2&1\\1&3\end{pmatrix}\begin{pmatrix}1\\2\end{pmatrix} = \begin{pmatrix}4\\7\end{pmatrix}^T \!=\; 4\,dx^1 + 7\,dx^2
\]
(Note: the result is the row vector $(4,7)$, living in $T_p^*M$.)\vspace{1pt}

Verification: $v^\flat(w) = 4w^1 + 7w^2$. And $g(v,w) = (1,2)\begin{pmatrix}2&1\\1&3\end{pmatrix}\binom{w^1}{w^2} = 4w^1+7w^2$. \checkmark

\medskip
\textbf{Raising indices $\sharp$} (cotangent vector $\to$ tangent vector): Take $\omega = 4\,dx^1 + 7\,dx^2$, i.e., $(\omega_i) = (4,7)$.

\[
(\omega^\sharp)^i = g^{ij}\omega_j: \qquad
\omega^\sharp = g^{-1}\begin{pmatrix}4\\7\end{pmatrix} = \frac{1}{5}\begin{pmatrix}3&-1\\-1&2\end{pmatrix}\begin{pmatrix}4\\7\end{pmatrix} = \frac{1}{5}\begin{pmatrix}5\\10\end{pmatrix} = \begin{pmatrix}1\\2\end{pmatrix}
\]
We recover $v = \binom{1}{2}$! This verifies that $\sharp = \flat^{-1}$ (raising and lowering are inverses of each other).

\medskip
\textbf{Origin of the names:} $\flat$ (flat) lowers ``high'' objects (tangent vectors, superscript $v^i$) to ``low'' ones (cotangent vectors, subscript $\omega_i$);
$\sharp$ (sharp) does the reverse. The notation matches the musical symbols for lowering a note by a half step ($\flat$) and raising it by a half step ($\sharp$).

\medskip
\textbf{Intuitive meaning of the metric $g = \begin{pmatrix}2&1\\1&3\end{pmatrix}$:}

The metric $g$ tells you the ``cost of moving in space.'' Specifically, the squared length of a tangent vector $v$ is:
\[
\|v\|_g^2 = g(v,v) = v^T g\, v = 2(v^1)^2 + 2v^1 v^2 + 3(v^2)^2
\]

\begin{itemize}
\item $g_{11}=2$: the cost of taking one step in the $x^1$ direction is $2$ (more expensive than the standard metric's $1$)
\item $g_{22}=3$: the cost of taking one step in the $x^2$ direction is $3$ (even more expensive)
\item $g_{12}=1 \neq 0$: the $x^1$ and $x^2$ directions are \textbf{not orthogonal}---moving simultaneously in both directions incurs a ``coupling cost''
\end{itemize}

Geometric picture: under the standard metric $g=I$, the set of all directions at ``unit distance'' forms a circle ($\|v\|=1$ is a circle).
Under this $g$, $\|v\|_g = 1$ becomes an \textbf{ellipse}:
\[
2(v^1)^2 + 2v^1 v^2 + 3(v^2)^2 = 1
\]
The major axis of the ellipse points in the ``cheap direction'' (low cost for the same coordinate distance), and the minor axis points in the ``expensive direction.''

\textbf{Physical analogy: } Imagine walking on muddy ground. $g=I$ is a flat dry road (equal cost in all directions).
$g \neq I$ is uneven terrain---some directions are downhill (cheap), some are muddy uphill (expensive),
and there may be slopes that cause ``walking north'' to inevitably ``slide a bit east'' (off-diagonal terms $g_{12}\neq 0$).

\begin{keypoint}
\textbf{Gradient = raising the index of the differential:} $\mathrm{grad}\,f = (df)^\sharp$

That is: $\mathrm{grad}\,f$ is obtained by ``lifting'' the cotangent vector $df\in T_p^*M$ into $T_pM$ via the metric $g$.
In coordinates: $(\mathrm{grad}\,f)^i = g^{ij}\frac{\partial f}{\partial x^j}$.

\medskip
This is why \textbf{the same $df$, with different $g$, yields different $\mathrm{grad}\,f$}.
$df$ is a fixed row vector; $\mathrm{grad}\,f = g^{-1}\cdot df^T$ depends on which $g^{-1}$ is used to ``flip'' it.
\end{keypoint}

\medskip
\textbf{(5) Riemannian Metric}

So far, we have the manifold $M$, the tangent space $T_pM$, the cotangent space $T_p^*M$, and the musical isomorphisms.
But the musical isomorphisms require an inner product $g$---this is the Riemannian metric.

\begin{definition}[Riemannian metric]
A \textbf{Riemannian metric} $g$ on a smooth manifold $M$ assigns to each point $p\in M$ an \textbf{inner product} $g_p$ on $T_pM$, varying smoothly with $p$.

That is, $g_p : T_pM\times T_pM \to \R$ satisfies:
\begin{itemize}
\item Bilinearity: $g_p(\alpha u + \beta v, w) = \alpha g_p(u,w) + \beta g_p(v,w)$
\item Symmetry: $g_p(u,v) = g_p(v,u)$
\item Positive definiteness: $g_p(v,v) > 0$ (when $v\neq 0$)
\end{itemize}

In coordinates: $g_p(\partial_i,\partial_j) = g_{ij}(p)$, so $g$ is represented by a positive definite symmetric matrix function $[g_{ij}(x)]$.
\end{definition}

\textbf{What can a Riemannian metric do?}
\begin{itemize}
\item Length of a tangent vector: $\|v\|_g = \sqrt{g_p(v,v)}$
\item Angle between two tangent vectors: $\cos\theta = \frac{g_p(u,v)}{\|u\|_g\,\|v\|_g}$
\item Length of a curve: $L(\gamma) = \int_0^1\|\dot\gamma(t)\|_{\gamma(t)}\,dt$
\item Distance between two points: $d(p,q) = \inf_\gamma L(\gamma)$ (geodesic distance)
\item Raising and lowering indices: $\flat: T_pM\to T_p^*M$, $\sharp: T_p^*M\to T_pM$
\item Gradient: $\mathrm{grad}\,f = (df)^\sharp$
\end{itemize}

A manifold $(M,g)$ equipped with a metric $g$ is called a \textbf{Riemannian manifold}.

\medskip
\textbf{(6) Gradient = $(df)^\sharp$ = raising the index of the differential}

\begin{definition}[Riemannian gradient]
Let $(M,g)$ be a Riemannian manifold and $f:M\to\R$ smooth. The \textbf{gradient} of $f$ at $p$ is:
\[
\mathrm{grad}\,f\big|_p := (df_p)^\sharp \in T_pM
\]
That is: first take the differential $df_p\in T_p^*M$ (a cotangent vector), then raise the index using the metric to obtain a tangent vector.

Equivalent definition (expanding the meaning of $\sharp$): $\mathrm{grad}\,f\big|_p$ is the unique tangent vector satisfying
\begin{equation}\label{eq:grad_def}
g_p\bigl(\mathrm{grad}\,f,\; v\bigr) = df_p(v) \quad \forall\,v\in T_pM
\end{equation}
\end{definition}

\textbf{Coordinate formula:}
\[
(\mathrm{grad}\,f)^i = \sum_j g^{ij}\,\frac{\partial f}{\partial x^j}, \qquad\text{i.e.,}\quad \mathrm{grad}\,f = g^{-1}\cdot(\text{column vector of partial derivatives})
\]

\medskip
\textbf{Concrete example: full computation of ``gradient = raising the index of the differential''}

Let $M = \R^2$, $f(x^1,x^2) = 2x^1 + 3x^2$, with metric $g = \begin{pmatrix}1 & 1\\1 & 4\end{pmatrix}$.

\medskip
\textbf{Step 1.} Compute the differential $df\in T_p^*M$ (independent of the metric):
\[
df = 2\,dx^1 + 3\,dx^2, \qquad (df)_i = (2, 3)\;\text{---a row vector}
\]

\textbf{Step 2.} Raise the index: $\mathrm{grad}\,f = (df)^\sharp = g^{-1}\cdot(df)^T$:
\[
g^{-1} = \frac{1}{3}\begin{pmatrix}4 & -1\\-1 & 1\end{pmatrix}, \qquad
\mathrm{grad}\,f = \frac{1}{3}\begin{pmatrix}4 & -1\\-1 & 1\end{pmatrix}\begin{pmatrix}2\\3\end{pmatrix}
= \begin{pmatrix}5/3\\1/3\end{pmatrix}
\]

\textbf{Step 3.} Verify $g(\mathrm{grad}\,f, v) = df(v)$:

Take arbitrary $v = (v^1, v^2)^T$:
\[
g(\mathrm{grad}\,f,\,v) = \begin{pmatrix}5/3\\1/3\end{pmatrix}^T\!\begin{pmatrix}1&1\\1&4\end{pmatrix}\begin{pmatrix}v^1\\v^2\end{pmatrix}
= 2v^1 + 3v^2 = df(v) \quad\checkmark
\]

\textbf{Comparison with the standard metric $g = I$:} $\mathrm{grad}\,f = I^{-1}\binom{2}{3} = \binom{2}{3}$---recovering the familiar $\nabla f$ from calculus.

\begin{keypoint}
\textbf{Core summary: }

\medskip
\begin{center}
\renewcommand{\arraystretch}{1.4}
\begin{tabular}{@{}lll@{}}
\toprule
\textbf{Object} & \textbf{Lives in} & \textbf{Depends on} \\
\midrule
Differential $df = (\partial_1 f,\ldots,\partial_n f)$ & $T_p^*M$ (row vector) & Only $f$ and the manifold structure \\
Gradient $\mathrm{grad}\,f = g^{-1}\cdot(df)^T$ & $T_pM$ (column vector) & $f$ \textbf{and the metric $g$} \\
\bottomrule
\end{tabular}
\end{center}

\medskip
\begin{itemize}
\item Standard metric $g=I$: the numerical values of $\mathrm{grad}\,f$ equal $(df)^T$; they ``look the same''
\item Non-standard metric: $\mathrm{grad}\,f = g^{-1}\cdot(df)^T \neq (df)^T$
\item \textbf{The Wasserstein space has a non-standard metric}---so the gradient $\neq$ the functional derivative
\end{itemize}
\end{keypoint}

\textbf{Physical intuition: } $g^{-1}$ performs a ``cost correction''---in directions that are expensive under the metric $g$, it moves less; in cheap directions, it moves more.
The gradient does not point in the ``direction of largest partial derivatives,'' but rather in the ``direction of steepest descent of $f$ per unit cost.''

\medskip
\textbf{Three equivalent geometric meanings of the gradient:}

Let $\|v\|_g = 1$ (a unit tangent vector in the metric $g$). Then:
\begin{enumerate}
\item The direction of $\mathrm{grad}\,f$ is the direction that maximizes $df(v)$ (finding the steepest direction in the $g$-unit ball)
\item The direction of $\mathrm{grad}\,f$ is the $g$-normal direction to the level set $\{f=c\}$
\item $\|\mathrm{grad}\,f\|_g = \max_{\|v\|_g=1} df(v)$
\end{enumerate}

\begin{intuition}
\textbf{Preview: why are the Wasserstein gradient and the $L^2$ gradient different?}

For the same functional $\mathcal{F}(\rho)$ and the same functional derivative $\frac{\delta\mathcal{F}}{\delta\rho}$, but:
\begin{itemize}
\item $L^2$ metric $\langle u,v\rangle_{L^2} = \int u\,v\,dx$
$\;\Rightarrow\;$ $\mathrm{grad}_{L^2}\mathcal{F} = \frac{\delta\mathcal{F}}{\delta\rho}$ (the $g=I$ case)
\item Wasserstein metric $\langle\xi,\eta\rangle_\rho = \int\rho\,\xi\cdot\eta\,dx$
$\;\Rightarrow\;$ $\mathrm{grad}_W\mathcal{F} = \nabla\frac{\delta\mathcal{F}}{\delta\rho}$ ($g\neq I$, with an extra $\nabla$)
\end{itemize}

This will be derived in detail in Section~5.
\end{intuition}

\medskip
\textbf{(7) Gradient Flow}

\begin{definition}[Gradient flow]
Let $(M,g)$ be a Riemannian manifold and $f:M\to\R$ smooth. The \textbf{gradient flow} of $f$ is a curve $\gamma(t)$ satisfying
\[
\dot\gamma(t) = -\mathrm{grad}\,f\big|_{\gamma(t)}
\]
That is: at each instant, move in the direction of steepest descent of $f$.
\end{definition}

\begin{itemize}
\item In $\R^n$ (standard metric): $\dot x = -\nabla f(x)$ (continuous-time gradient descent)
\item On a Riemannian manifold: $\dot\gamma = -\mathrm{grad}\,f = -(df)^\sharp$ (``rolling downhill'' along the surface)
\item In Wasserstein space: the gradient flow yields a PDE (the Fokker--Planck equation)
\end{itemize}

Key property of gradient flow: $f$ strictly decreases along the flow---$\frac{d}{dt}f(\gamma(t)) = -\|\mathrm{grad}\,f\|_g^2 \leq 0$.

\medskip
\textbf{FAQ 1: Is $df$ a complete object on its own? Or does it need a direction to be defined?}

$df_p$ is itself a \textbf{fully defined object}---it is an element of $T_p^*M$ (a cotangent vector) and does not need a ``given direction'' to exist.
It simultaneously encodes information about \textbf{all directions}: given any $v\in T_pM$, it outputs $df_p(v)\in\R$.

Analogy: in $\R^n$, $df_p$ corresponds to the row vector $\bigl(\frac{\partial f}{\partial x^1},\ldots,\frac{\partial f}{\partial x^n}\bigr)$.
This row vector is already completely determined; no direction needs to be specified.
Its ``action'' is multiplication with a column vector $v$: $df_p(v) = \sum_i\frac{\partial f}{\partial x^i}v^i$.

\begin{center}
\renewcommand{\arraystretch}{1.3}
\begin{tabular}{@{}lll@{}}
\toprule
\textbf{Object} & \textbf{What is it} & \textbf{What is needed to define it} \\
\midrule
$df_p\in T_p^*M$ & Linear functional (row vector) & Only $f$ and the manifold structure \\
$df_p(v)\in\R$ & A number (directional derivative) & $f$ and a direction $v$ \\
$\mathrm{grad}\,f\in T_pM$ & Tangent vector (column vector) & $f$ and a metric $g$ \\
\bottomrule
\end{tabular}
\end{center}

\medskip
\textbf{FAQ 2: Is $\mathrm{grad}\,f$ a ``direction''?}

Yes. $\mathrm{grad}\,f\big|_p \in T_pM$ is a tangent vector---it lives in the tangent space, so it has both a \textbf{direction} and a \textbf{magnitude}.

Specifically:
\begin{itemize}
\item $df_p\in T_p^*M$ is a cotangent vector (a ``measuring instrument''); it is \textbf{not} itself a direction
\item $\mathrm{grad}\,f\big|_p = (df_p)^\sharp \in T_pM$ is a tangent vector (an ``arrow''); it \textbf{is} a direction
\end{itemize}

The meaning of this direction: under the metric $g$, it is the direction of fastest increase of $f$ (the direction of $v$ that maximizes $df(v)/\|v\|_g$).

So in the gradient flow $\dot\gamma(t) = -\mathrm{grad}\,f$, both sides of the equation are tangent vectors in $T_{\gamma(t)}M$:
``the velocity of the curve'' $=$ ``the negative gradient direction''---two vectors in the same tangent space being equal.

\medskip
\textbf{FAQ 3: Why does $\frac{d}{dt}f(\gamma(t)) = -\|\mathrm{grad}\,f\|_g^2$ involve a square?}

Derivation: 

\textbf{Step 1.} Chain rule (for functions on manifolds):
\[
\frac{d}{dt}f(\gamma(t)) = df_{\gamma(t)}\bigl(\dot\gamma(t)\bigr)
\]

\textbf{Step 2.} Substitute the gradient flow definition $\dot\gamma = -\mathrm{grad}\,f$:
\[
= df\bigl(-\mathrm{grad}\,f\bigr) = -df\bigl(\mathrm{grad}\,f\bigr)
\]

\textbf{Step 3.} Use the definition of the gradient $g(\mathrm{grad}\,f,\,v) = df(v)$, taking $v = \mathrm{grad}\,f$:
\[
df(\mathrm{grad}\,f) = g(\mathrm{grad}\,f,\,\mathrm{grad}\,f) = \|\mathrm{grad}\,f\|_g^2
\]

\textbf{Conclusion: }
\[
\frac{d}{dt}f(\gamma(t)) = -\|\mathrm{grad}\,f\|_g^2 \leq 0
\]

The reason a square appears: \textbf{the gradient plays two roles simultaneously}---it is both the velocity (via $\dot\gamma = -\mathrm{grad}\,f$) and the ``metric dual'' for the directional derivative (via $df(v) = g(\mathrm{grad}\,f, v)$).
When $v$ happens to equal $\mathrm{grad}\,f$ itself, $g(\mathrm{grad}\,f, \mathrm{grad}\,f) = \|\mathrm{grad}\,f\|^2$.

Analogy in $\R^n$: $\frac{d}{dt}f(x(t)) = \nabla f\cdot\dot x = \nabla f\cdot(-\nabla f) = -|\nabla f|^2$.

\medskip
\textbf{(8) Summary---the logical chain of all concepts:}

\[
\boxed{\text{Manifold $M$} \xrightarrow[\text{smooth structure}]{\text{defines}} T_pM,\,T_p^*M \xrightarrow[\text{choose $g$}]{\text{extra structure}} \sharp,\flat \xrightarrow[\text{raise index}]{df\mapsto(df)^\sharp} \mathrm{grad}\,f \xrightarrow[\text{negative direction}]{-\mathrm{grad}\,f} \text{gradient flow}}
\]

Key insight: \textbf{The gradient and gradient flow depend on the choice of metric.}
Different metrics $\Rightarrow$ different $\sharp$ $\Rightarrow$ different gradients $\Rightarrow$ different gradient flow PDEs.

This is why the same free energy functional $\mathcal{F}$, under the $L^2$ metric gives the heat equation, while under the Wasserstein metric gives the Fokker--Planck equation.

\section{Analysis Tools: Integration by Parts and du Bois-Reymond Lemma}\label{app:analysis}

This appendix collects the analytical tools used repeatedly throughout the text: the family of integration by parts formulas and the du Bois-Reymond lemma.

\subsection{Integration by Parts}

The essence of integration by parts is \textbf{transferring a derivative from one function to another}, at the cost of introducing a minus sign (and boundary terms).

\subsubsection{One-dimensional case}

\begin{theorem}[One-dimensional integration by parts]
Let $f, g \in C^1([a,b])$. Then:
\[
\int_a^b f'(x)\,g(x)\,dx = \big[f(x)g(x)\big]_a^b - \int_a^b f(x)\,g'(x)\,dx
\]
If $f$ or $g$ vanishes at the endpoints (e.g., has compact support), then the boundary terms vanish:
\[
\int_a^b f'\,g\,dx = -\int_a^b f\,g'\,dx
\]
\end{theorem}

\textbf{Intuition: } Integrate the product rule $(fg)' = f'g + fg'$ on both sides. The left-hand side becomes the boundary term $[fg]_a^b$; rearranging yields integration by parts.

\subsubsection{Higher-dimensional case: divergence theorem version}

\begin{theorem}[Integration by Parts in Divergence Form / Green's First Identity]
Let $\varphi$ be a smooth function with compact support (or vanishing on $\partial\Omega$), and let $\mathbf{F}$ be a smooth vector field. Then:
\begin{equation}\label{eq:ibp_div}
\int_{\R^d} \varphi\,(\nabla\cdot\mathbf{F})\,dx = -\int_{\R^d} \nabla\varphi\cdot\mathbf{F}\,dx
\end{equation}
\end{theorem}

\textbf{Derivation: } Integrate the product rule $\nabla\cdot(\varphi\mathbf{F}) = \nabla\varphi\cdot\mathbf{F} + \varphi\,\nabla\cdot\mathbf{F}$ on both sides.
The left-hand side becomes the boundary integral $\oint\varphi\mathbf{F}\cdot\mathbf{n}\,dS = 0$ by the divergence theorem (since $\varphi$ has compact support); rearranging gives the result.

\textbf{Typical usage in the main text:}
\begin{itemize}
\item $\mathbf{F} = \rho\nabla V$: $\displaystyle\int\nabla\varphi\cdot(\rho\nabla V)\,dx = -\int\varphi\,\nabla\cdot(\rho\nabla V)\,dx$
\item $\mathbf{F} = \rho\,\eta$ ($\eta$ is a perturbation vector field): $\displaystyle\int f\,\nabla\cdot(\rho\eta)\,dx = -\int\nabla f\cdot(\rho\eta)\,dx$
\end{itemize}

\subsubsection{Transferring the Laplacian (two integrations by parts)}

\begin{theorem}[Green's second identity]
Let $\varphi$ and $g$ be smooth functions with $\varphi$ compactly supported. Then:
\begin{equation}\label{eq:ibp_laplacian}
\int_{\R^d} (\Delta\varphi)\,g\,dx = \int_{\R^d} \varphi\,(\Delta g)\,dx
\end{equation}
\end{theorem}

\textbf{Derivation: } Apply \eqref{eq:ibp_div} twice:
\begin{align*}
\int(\Delta\varphi)\,g\,dx &= \int(\nabla\cdot\nabla\varphi)\,g\,dx = -\int\nabla\varphi\cdot\nabla g\,dx \quad\text{(first application, minus sign)}\\
&= +\int\varphi\,(\nabla\cdot\nabla g)\,dx = \int\varphi\,\Delta g\,dx \quad\text{(second application, $-\times-=+$)}
\end{align*}

\textbf{Intuition: } The second-order derivative can ``jump'' from one function to another without changing sign---because two integrations by parts produce two minus signs that cancel.

\subsubsection{Integration by parts in time}

The same procedure applies to integration in time. Let $\varphi(t,x)$ vanish at $t=0$ and $t=T$:
\[
\int_0^T\!\int \varphi\,\partial_t\rho\,dx\,dt = -\int_0^T\!\int (\partial_t\varphi)\,\rho\,dx\,dt
\]
This is used when deriving the weak form of the continuity equation.

\subsubsection{Formula summary table}

\begin{center}
\renewcommand{\arraystretch}{1.5}
\begin{tabular}{@{}lll@{}}
\toprule
\textbf{Formula} & \textbf{Condition} & \textbf{Appears in} \\
\midrule
$\displaystyle\int\varphi\,\nabla\cdot\mathbf{F} = -\int\nabla\varphi\cdot\mathbf{F}$ & $\varphi$ compactly supported & FP derivation Step 5 \\[6pt]
$\displaystyle\int(\Delta\varphi)\,g = \int\varphi\,\Delta g$ & $\varphi$ compactly supported & FP derivation Term 2 \\[6pt]
$\displaystyle\int\varphi\,\partial_t\rho = -\int(\partial_t\varphi)\,\rho$ & $\varphi|_{t=0,T}=0$ & Weak form of continuity eq. \\[6pt]
$\displaystyle\int f\,\nabla\cdot(\rho\eta) = -\int\nabla f\cdot(\rho\eta)$ & $\rho\eta$ decays at $\infty$ & Wasserstein gradient derivation \\
\bottomrule
\end{tabular}
\end{center}

\subsection{Du Bois-Reymond Lemma (Fundamental Lemma of the Calculus of Variations)}

\begin{theorem}[Du Bois-Reymond Lemma / Fundamental Lemma of Calculus of Variations]\label{thm:dubois}
Let $h:\R^d\to\R$ be a locally integrable function. If for all compactly supported smooth functions $\varphi\in C_c^\infty(\R^d)$ we have:
\begin{equation}\label{eq:dubois}
\int_{\R^d} h(x)\,\varphi(x)\,dx = 0
\end{equation}
then $h(x) = 0$ almost everywhere (a.e.). If furthermore $h$ is continuous, then $h(x) = 0$ holds everywhere.
\end{theorem}

\subsubsection{Intuitive understanding}

\textbf{Finite-dimensional analogy: } If a vector $v\in\R^n$ has zero inner product with every vector ($\langle v, w\rangle = 0\;\forall w$), then $v=0$.

The Du Bois-Reymond lemma is the \textbf{infinite-dimensional version} of this fact:
\begin{itemize}
\item ``vector'' $\to$ function $h(x)$
\item ``inner product'' $\to$ integral $\int h\varphi\,dx$
\item ``zero inner product with all vectors'' $\to$ ``zero integral against all test functions $\varphi$''
\item Conclusion: $h = 0$ (a.e.)
\end{itemize}

\subsubsection{Proof sketch}

Proof by contradiction: suppose $h(x_0) > 0$ at some point. By continuity, $h > 0$ on some small ball $B_\epsilon(x_0)$ around $x_0$.
Choose a nonnegative $\varphi$ supported inside $B_\epsilon(x_0)$ (e.g., a smooth bump function). Then:
\[
\int h\,\varphi\,dx = \int_{B_\epsilon(x_0)} h\,\varphi\,dx > 0
\]
This contradicts the hypothesis. Hence $h$ cannot be positive at any point. By the same argument it cannot be negative. Therefore $h=0$.

\subsubsection{Constrained version (used in the main text)}

What appears in Step 1 of the JKO derivation in the main text is a \textbf{constrained version}:

\begin{theorem}[Version with mass conservation constraint]
Let $h$ be a continuous function. If for all smooth perturbations $\delta\rho$ satisfying $\int\delta\rho\,dx = 0$ we have:
\[
\int h(x)\,\delta\rho(x)\,dx = 0
\]
then $h(x) = \mathrm{const}$ (a constant).
\end{theorem}

\textbf{Why a constant rather than zero?}

Because the test functions $\delta\rho$ are constrained to satisfy $\int\delta\rho = 0$---they cannot be ``arbitrary'' but must have ``zero total mass.''
This means we lose one ``probing direction'' (the constant-function direction).
\begin{itemize}
\item Unconstrained: $\int h\varphi = 0\;\forall\varphi$ $\implies$ $h=0$
\item Constrained $\int\delta\rho=0$: $\int h\,\delta\rho = 0\;\forall\,\delta\rho\text{ s.t. }\int\delta\rho=0$ $\implies$ $h=\text{const}$
\end{itemize}

\textbf{Proof: } For any two points $x_1, x_2$, take $\delta\rho = \varphi_1 - \varphi_2$, where $\varphi_i$ is a nonnegative smooth bump function centered at $x_i$ with $\int\varphi_i = 1$. Then $\int\delta\rho = 0$, and by hypothesis:
\[
\int h\,\varphi_1\,dx = \int h\,\varphi_2\,dx
\]
As the support of the bump functions shrinks to $x_i$ ($\varphi_i\to\delta_{x_i}$), continuity of $h$ gives $h(x_1) = h(x_2)$. Since $x_1,x_2$ are arbitrary, $h$ is constant.

\subsubsection{Where it is used in the main text}

\begin{center}
\renewcommand{\arraystretch}{1.4}
\begin{tabular}{@{}p{5cm}p{7.5cm}@{}}
\toprule
\textbf{Context} & \textbf{How it is used} \\
\midrule
Fokker-Planck derivation Step 6 & $\int\varphi[\partial_t\rho - \Delta\rho - \nabla\cdot(\rho\nabla V)]\,dx = 0\;\forall\varphi$ $\implies$ bracket $=0$ \\[4pt]
JKO optimality condition Step 1 & $\int h\,\delta\rho = 0$ for all $\delta\rho$ with $\int\delta\rho=0$ $\implies$ $h = \text{const}$ (Lagrange multiplier) \\[4pt]
Weak form of continuity equation & Uniqueness of weak solutions relies on this lemma \\
\bottomrule
\end{tabular}
\end{center}

\begin{intuition}
\textbf{One-sentence summary: }

The Du Bois-Reymond lemma is the ``starting tool'' of the calculus of variations---it lets us pass from integral equalities (weak/variational form) to pointwise equalities (strong form/PDE).

Without it, we can only say ``the weighted average is zero''; with it, we can say ``the function itself is zero (or constant).''
This is the bridge from ``integral equations'' to ``differential equations.''
\end{intuition}

\section{Gibbs--Boltzmann Distribution}\label{app:gibbs}

The form $\pi(x) = \frac{1}{Z}e^{-\beta V(x)}$ appears repeatedly across all corners of science---whenever a system strikes a balance between ``energy'' and ``randomness,'' its stationary distribution is the Gibbs--Boltzmann distribution.

\subsection{Core intuition}

At thermal equilibrium, low-energy states occur with \textbf{exponentially} higher probability than high-energy states.
The probability ratio between two states depends only on their energy difference:
\[
\frac{\pi(x_1)}{\pi(x_2)} = e^{-\beta(V(x_1)-V(x_2))}
\]
For every $k_BT$ unit of energy difference, the probabilities differ by a factor of $e\approx 2.7$.

\subsection{The role of temperature $T$---the single control knob}

\begin{itemize}
\item $T\to 0$ ($\beta\to\infty$): the distribution degenerates to a Dirac delta; the system ``freezes'' at the lowest-energy state
\item $T\to\infty$ ($\beta\to 0$): the distribution approaches uniform; the system completely ignores energy and wanders randomly
\item Finite $T$: a compromise between the two---low-energy states are ``preferred'' but not exclusively
\end{itemize}

Intuition: temperature controls the ratio of ``exploration'' to ``exploitation.'' Low temperature = pure exploitation (stay at the optimum); high temperature = pure exploration (wander everywhere).

\subsection{Applications across disciplines}

\begin{center}
\renewcommand{\arraystretch}{1.4}
\begin{tabular}{@{}p{3.2cm}p{4cm}p{5.2cm}@{}}
\toprule
\textbf{Field} & \textbf{What is $V(x)$} & \textbf{Specific application} \\
\midrule
Statistical mechanics & System Hamiltonian & Canonical ensemble, ideal gas velocity distribution (Maxwell dist.) \\
Machine learning & Learned energy function & Boltzmann machine, Energy-based models, contrastive learning \\
Bayesian inference & $-\log p(D|\theta)p(\theta)$ & Posterior $p(\theta|D)\propto e^{-V(\theta)}$ \\
Optimization & Objective function & Simulated annealing: gradually lower $T$ to concentrate distribution at extremum \\
Sampling/MCMC & $-\log(\text{target dist.})$ & Langevin dynamics, Hamiltonian MC \\
Diffusion generative models & Potential energy field & Fokker-Planck stationary dist. $=$ data distribution \\
Chemistry & Activation energy $E_a$ & Arrhenius reaction rate $\propto e^{-E_a/k_BT}$ \\
Biophysics & Free energy landscape & Protein folding, ion channel gating \\
\bottomrule
\end{tabular}
\end{center}

\subsection{Core connection to the main text}

The Fokker-Planck equation $\partial_t\rho = \Delta\rho + \nabla\cdot(\rho\nabla V)$ describes precisely the process by which the system evolves from an arbitrary initial distribution $\rho_0$ toward the Gibbs--Boltzmann equilibrium $\pi\propto e^{-V}$.

\begin{itemize}
\item The free energy $\mathcal{F}(\rho) = \KL(\rho\|\pi) + \text{const}$ measures ``how far from equilibrium''
\item The unique minimizer of $\mathcal{F}$ is $\pi$ ($\KL(\rho\|\pi)=0 \iff \rho=\pi$)
\item Each step of the JKO scheme minimizes $\mathcal{F}+\frac{1}{2\tau}W_2^2$, discretely approximating this approach to equilibrium
\end{itemize}

The endpoint of the entire story is this elegant $e^{-V}$.

\subsection{Why the exponential form?}

The $e^{-\beta V}$ form of the Gibbs distribution is not an arbitrary choice but is dictated by the \textbf{maximum entropy principle}:

Subject to the constraint that the mean energy $\langle V\rangle = \int V\,d\rho = E$ is fixed, the distribution that maximizes the entropy $-\int\rho\log\rho$ is $\rho\propto e^{-\beta V}$,
where $\beta$ is the Lagrange multiplier (which happens to correspond to the inverse temperature).

In other words: \textbf{among all distributions with mean energy $E$, the Gibbs distribution is the ``most random'' (most uncertain) one.}
Nature selects it because there is no reason to prefer a more ordered state---this is the statistical interpretation of the second law of thermodynamics.

\section{Fisher Information}\label{app:fisher}

The name ``Fisher information'' has \textbf{two different meanings} in mathematics, originating from statistics and information theory/PDE respectively.
Below we introduce each in turn, explain how they are unified, and finally present the relative version used in the main text.

\subsection{Fisher Information in Statistics (Parametric Version)}

\begin{definition}[Fisher Information --- Statistical Version]
Given a parametric family of probability distributions $\{p_\theta(x)\}_{\theta\in\R^k}$, the \textbf{Fisher information matrix} (Fisher, 1925) is defined as:
\begin{equation}
\mathcal{I}(\theta)_{ij} := \mathbb{E}_{x\sim p_\theta}\left[\frac{\partial\log p_\theta(x)}{\partial\theta_i}\cdot\frac{\partial\log p_\theta(x)}{\partial\theta_j}\right]
= -\mathbb{E}_{x\sim p_\theta}\left[\frac{\partial^2\log p_\theta(x)}{\partial\theta_i\partial\theta_j}\right]
\end{equation}
In the one-dimensional parameter case: $\mathcal{I}(\theta) = \mathbb{E}_{x\sim p_\theta}\left[\left(\frac{\partial\log p_\theta}{\partial\theta}\right)^2\right]$.
\end{definition}

\textbf{Intuition}: $\mathcal{I}(\theta)$ measures ``how much information the data carry about the parameter $\theta$.''
\begin{itemize}
\item $\mathcal{I}(\theta)$ large $\to$ changing $\theta$ causes a drastic change in the distribution $\to$ data can easily distinguish different $\theta$ $\to$ easy to estimate $\theta$
\item $\mathcal{I}(\theta)$ small $\to$ changing $\theta$ barely affects the distribution $\to$ data are insensitive to $\theta$ $\to$ hard to estimate
\item Cram\'er-Rao lower bound: the variance of any unbiased estimator $\hat\theta$ satisfies $\mathrm{Var}(\hat\theta)\geq \mathcal{I}(\theta)^{-1}$
\end{itemize}

\textbf{Example}: $p_\theta = \mathcal{N}(\theta, \sigma^2)$ (mean as parameter).
$\log p_\theta = -\frac{(x-\theta)^2}{2\sigma^2} + C$,
$\frac{\partial}{\partial\theta}\log p_\theta = \frac{x-\theta}{\sigma^2}$,
$\mathcal{I}(\theta) = \mathbb{E}\left[\frac{(x-\theta)^2}{\sigma^4}\right] = \frac{1}{\sigma^2}$.
The smaller the variance, the larger the Fisher information---the more ``concentrated'' the data, the easier it is to estimate the mean.

\subsection{Fisher Information in Information Theory/PDE (Distribution Version)}

\begin{definition}[Fisher Information --- Distribution Version / de Bruijn Version]
Given a probability density $\rho$ (not depending on any parameter), its \textbf{Fisher information} is defined as:
\begin{equation}
I(\rho) := \int_{\R^d}\rho(x)\,|\nabla_x\log\rho(x)|^2\,dx = \int_{\R^d}\frac{|\nabla\rho(x)|^2}{\rho(x)}\,dx
\end{equation}
\end{definition}

\textbf{Intuition}: $I(\rho)$ measures the ``spatial irregularity'' of the density $\rho$.
\begin{itemize}
\item The smoother/flatter $\rho$ is $\to$ the smaller $\nabla\log\rho$ $\to$ the smaller $I(\rho)$
\item The sharper/more rapidly varying $\rho$ is $\to$ the larger $\nabla\log\rho$ $\to$ the larger $I(\rho)$
\item Uniform distribution: $I = 0$ (completely flat)
\item Dirac delta: $I = \infty$ (infinitely sharp)
\end{itemize}

\textbf{Relation to the heat equation} (de Bruijn identity): if $\rho_t$ solves the heat equation $\partial_t\rho = \Delta\rho$, then:
\[
\frac{d}{dt}H(\rho_t) = -I(\rho_t) \qquad\text{(rate of entropy change $=$ negative Fisher information)}
\]
where $H(\rho) = -\int\rho\log\rho$ is the entropy. Diffusion flattens the distribution (entropy increases), and the rate is controlled by the Fisher information.

\subsection{Unification of the two versions}

These two forms of Fisher information look different; they are unified as follows:

\textbf{Key observation}: take the parametric family to be the ``translation family'' $p_\theta(x) = \rho(x-\theta)$. Then:
\[
\frac{\partial}{\partial\theta}\log p_\theta(x)\bigg|_{\theta=0} = -\nabla_x\log\rho(x)
\]
Substituting into the statistical version of the Fisher information:
\[
\mathcal{I}(0) = \mathbb{E}_{x\sim\rho}\left[|\nabla_x\log\rho(x)|^2\right] = I(\rho)
\]
\textbf{The two coincide!} The distribution-version Fisher information is precisely the statistical-version Fisher information evaluated at the ``translation parameter.''

More generally:
\begin{itemize}
\item The statistical version measures ``how easily different $\theta$ can be distinguished in parameter space''
\item The distribution version measures ``how rapidly the density varies across different locations $x$ in physical space''
\item When ``the parameter is the location,'' the two are equivalent
\end{itemize}

\textbf{Derivation details}: why does $\frac{\partial}{\partial\theta}\log p_\theta(x)\big|_{\theta=0} = -\nabla_x\log\rho(x)$?

Let $p_\theta(x) = \rho(x-\theta)$, so $\log p_\theta(x) = \log\rho(x-\theta)$.
\[
\frac{\partial}{\partial\theta}\log\rho(x-\theta) = -\nabla_y\log\rho(y)\big|_{y=x-\theta}\cdot 1 = -\nabla\log\rho(x-\theta)
\]
At $\theta=0$: $= -\nabla\log\rho(x)$.

So the statistical version becomes:
\[
\mathcal{I}(0) = \mathbb{E}_{x\sim\rho}\left[(-\nabla\log\rho(x))^2\right] = \int\rho\,|\nabla\log\rho|^2\,dx = I(\rho) \quad\checkmark
\]

\subsection{Relative Fisher Information}

\begin{definition}[Relative Fisher Information]
Given probability densities $\rho$ and a target distribution $\pi$, the \textbf{relative Fisher information} is defined as:
\begin{equation}
I(\rho\|\pi) := \int\rho(x)\left|\nabla\log\frac{\rho(x)}{\pi(x)}\right|^2 dx
= \int\rho(x)\left|\nabla\log\rho(x) - \nabla\log\pi(x)\right|^2 dx
\end{equation}
\end{definition}

When $\pi(x) \propto e^{-V(x)}$, we have $\nabla\log\pi = -\nabla V$, so:
\[
I(\rho\|\pi) = \int\rho\,|\nabla\log\rho + \nabla V|^2\,dx = \int\rho\,|s_\rho(x) - s_\pi(x)|^2\,dx
\]
where $s_\rho = \nabla\log\rho$ and $s_\pi = \nabla\log\pi = -\nabla V$.

\textbf{Relation to ordinary Fisher information}: when $\pi$ is the uniform distribution ($V=0$):
\[
I(\rho\|\pi) = \int\rho\,|\nabla\log\rho|^2\,dx = I(\rho) \qquad\text{(reduces to the ordinary version)}
\]

\textbf{Intuition}: $I(\rho\|\pi)$ measures the discrepancy between the score of $\rho$ and the score of $\pi$.
When $\rho=\pi$, the two scores coincide and $I=0$; the further $\rho$ deviates from $\pi$, the larger $I$ becomes.

\subsection{Three identities}

\begin{center}
\renewcommand{\arraystretch}{1.4}
\begin{tabular}{@{}p{3.5cm}p{9cm}@{}}
\toprule
\textbf{Identity} & \textbf{Meaning} \\
\midrule
Rate of KL descent & $\frac{d}{dt}\KL(\rho_t\|\pi) = -I(\rho_t\|\pi)$ \\[3pt]
Norm of Wasserstein gradient & $I(\rho\|\pi) = \|\mathrm{grad}_W\mathcal{F}(\rho)\|_\rho^2 = \int\rho\,|\mathrm{grad}_W\mathcal{F}|^2\,dx$ \\[3pt]
Score deviation from target & $I(\rho\|\pi) = \mathbb{E}_{x\sim\rho}\left[|s_\rho(x) - s_\pi(x)|^2\right]$ \\
\bottomrule
\end{tabular}
\end{center}

\begin{itemize}
\item \textbf{Measures ``how far from equilibrium'' (dynamic sense)}: $I(\rho\|\pi)$ quantifies how fast $\rho$ is approaching $\pi$. $I=0$ means equilibrium has been reached; large $I$ means the system is still far from equilibrium.

\item \textbf{$L^2$ discrepancy of scores}: $I(\rho\|\pi) = \int\rho\,|s_\rho - s_\pi|^2\,dx$. This is precisely the loss function used in score matching! Training a score-based diffusion model is essentially minimizing the Fisher information.

\item \textbf{``Steepness'' of the gradient}: by analogy with $|\nabla f|^2$ in Euclidean gradient flow, $I(\rho\|\pi)$ measures the Wasserstein ``slope'' of the free energy at the current $\rho$.
\end{itemize}

\subsection{Information inequalities}

\begin{itemize}
\item \textbf{Log-Sobolev inequality} (LSI): if $\pi$ satisfies an LSI with constant $\lambda > 0$, then:
\[
I(\rho\|\pi) \geq 2\lambda\,\KL(\rho\|\pi)
\]
Combined with $\frac{d}{dt}\KL = -I$, this yields a Gronwall inequality:
\[
\frac{d}{dt}\KL(\rho_t\|\pi) \leq -2\lambda\,\KL(\rho_t\|\pi)
\implies \KL(\rho_t\|\pi)\leq e^{-2\lambda t}\,\KL(\rho_0\|\pi)
\]
i.e., \textbf{exponential convergence} to equilibrium.

\item \textbf{Talagrand inequality}: $W_2^2(\rho,\pi) \leq \frac{2}{\lambda}\KL(\rho\|\pi)$

\item \textbf{Information hierarchy} (from strongest to weakest):
\[
I(\rho\|\pi) \;\xrightarrow{\text{LSI}}\; \KL(\rho\|\pi) \;\xrightarrow{\text{Talagrand}}\; W_2^2(\rho,\pi)
\]
Fisher information is the ``strongest'' measure, and the Wasserstein distance is the ``weakest.''
\end{itemize}

\subsection{Example}

Let $\rho = \mathcal{N}(m, \sigma^2)$ and $\pi = \mathcal{N}(0,1)$ (one-dimensional).

\begin{itemize}
\item $\nabla\log\rho = -\frac{x-m}{\sigma^2}$, $\nabla\log\pi = -x$
\item $\nabla\log\frac{\rho}{\pi} = -\frac{x-m}{\sigma^2} + x = x(1-\frac{1}{\sigma^2}) + \frac{m}{\sigma^2}$
\item $I(\rho\|\pi) = \mathbb{E}_{x\sim\rho}\left[\left(x(1-\frac{1}{\sigma^2}) + \frac{m}{\sigma^2}\right)^2\right]$
\end{itemize}

Expanding (with $x\sim\mathcal{N}(m,\sigma^2)$, letting $a = 1-\frac{1}{\sigma^2}$ and $b = \frac{m}{\sigma^2}$):
\begin{align*}
I &= \mathbb{E}[(ax+b)^2] = a^2\mathbb{E}[x^2] + 2ab\mathbb{E}[x] + b^2\\
&= a^2(m^2+\sigma^2) + 2abm + b^2\\
&= \left(1-\frac{1}{\sigma^2}\right)^2(m^2+\sigma^2) + 2\left(1-\frac{1}{\sigma^2}\right)\frac{m^2}{\sigma^2} + \frac{m^2}{\sigma^4}
\end{align*}

Checking special cases:
\begin{itemize}
\item $m=0,\sigma=1$ ($\rho=\pi$): $a=0, b=0$, $I=0$ \checkmark
\item $m=0,\sigma\neq 1$: $I = (1-\frac{1}{\sigma^2})^2\sigma^2 = (\sigma - \frac{1}{\sigma})^2$, which is zero only when $\sigma=1$
\item $\sigma=1, m\neq 0$: $a=0, b=m$, $I = m^2$; the further the deviation, the larger $I$
\end{itemize}

\section{Generative Models: Score, Velocity, and Diffusion}\label{app:genmodels}

This appendix elaborates on the score/velocity relationship mentioned in the main text and the mathematical details of major generative models.

\subsection{The Complete Relationship among Score, Velocity, and Drift}

Starting from the Fokker-Planck velocity field $v_t = -\nabla V - \nabla\log\rho_t$:

\begin{keypoint}
\textbf{Definitions of the three quantities:}
\begin{itemize}
    \item \textbf{Score} (score function): $s_t(x) := \nabla_x\log\rho_t(x)$, the gradient of the log-density. This is what diffusion models learn.
    \item \textbf{Velocity} (velocity field): $v_t(x)$, satisfying the continuity equation $\partial_t\rho_t + \nabla\cdot(\rho_t v_t) = 0$. This is what flow matching learns.
    \item \textbf{Drift}: $f(x,t)$, the deterministic part of the SDE $dX_t = f\,dt + g\,dB_t$. This is given (determined by the forward process).
\end{itemize}

For the SDE $dX_t = f(X_t,t)\,dt + g(t)\,dB_t$, the velocity field of the corresponding \textbf{probability flow ODE} is:
\begin{equation}
v_t(x) = f(x,t) - \frac{1}{2}g(t)^2\,\nabla\log\rho_t(x)
\end{equation}

\textbf{In one sentence: }
\[
\boxed{\text{velocity}_{t} = \text{drift}_{t} - \tfrac{1}{2}g_t^2\times\text{score}_{t}}
\]
\end{keypoint}

\begin{center}
\renewcommand{\arraystretch}{1.4}
\begin{tabular}{@{}lll@{}}
\toprule
\textbf{Framework} & \textbf{What is learned} & \textbf{Mathematical object} \\
\midrule
Score-based diffusion & score $s_\theta(x,t)\approx\nabla\log\rho_t(x)$ & Gradient of the log-density \\
Flow matching & velocity $v_\theta(x,t)\approx v_t(x)$ & Velocity field of the continuity equation \\
\bottomrule
\end{tabular}
\end{center}

Their interconversion:
\begin{itemize}
    \item Knowing score + drift $\Rightarrow$ velocity: $v_t = f - \frac{1}{2}g^2 s_t$
    \item Knowing velocity + drift $\Rightarrow$ score: $s_t = \frac{2(f - v_t)}{g^2}$
    \item Both can generate samples: score via Langevin dynamics or reverse SDE; velocity via ODE integration
\end{itemize}

\subsection{Geometric meaning of the score: the Wasserstein gradient of entropy}

The score function $s(x) = \nabla_x\log\rho(x)$ has an elegant geometric interpretation---it is precisely the Wasserstein gradient of the negative entropy.

\medskip
\textbf{Step 1: Define the negative entropy functional.}

Negentropy:
\[
\mathcal{H}(\rho) = \int_{\R^d}\rho(x)\log\rho(x)\,dx
\]
A larger $\mathcal{H}$ indicates a more concentrated distribution (low entropy); a smaller $\mathcal{H}$ indicates a more spread-out distribution (high entropy).

\medskip
\textbf{Step 2: Compute the functional derivative $\frac{\delta\mathcal{H}}{\delta\rho}$.}

\begin{align*}
\frac{d}{d\epsilon}\bigg|_{\epsilon=0}\mathcal{H}(\rho+\epsilon\,\delta\rho)
&= \int\left[\delta\rho\cdot\log\rho + \delta\rho\right]dx
= \int(\log\rho + 1)\,\delta\rho\,dx
\end{align*}
Therefore $\frac{\delta\mathcal{H}}{\delta\rho}(x) = \log\rho(x) + 1$.

\medskip
\textbf{Step 3: From functional derivative to Wasserstein gradient.}

The core formula of Otto calculus: the velocity field corresponding to the Wasserstein gradient is $v = \nabla\frac{\delta G}{\delta\rho}$. For $\mathcal{H}$:
\[
\mathrm{grad}_W\mathcal{H} = \nabla\frac{\delta\mathcal{H}}{\delta\rho} = \nabla(\log\rho + 1) = \nabla\log\rho
\]

\medskip
\textbf{Step 4: Recognize the score.}
\[
\boxed{s(x) = \nabla\log\rho = \mathrm{grad}_W\mathcal{H}(\rho)}
\]
That is: \textbf{Score $=$ Wasserstein gradient of the negative entropy.}

\begin{keypoint}
Learning the score $\nabla\log\rho_t$ $\equiv$ learning the ``gradient direction'' in probability space.
A diffusion model learning the score is essentially learning ``which direction to move in probability space so that the distribution converges to the data distribution as quickly as possible.''
\end{keypoint}

\subsection{A unified framework for mainstream models}

All diffusion/flow-based generative models follow the same framework:
\begin{enumerate}
\item \textbf{Forward process}: Design an SDE that gradually transforms the data distribution $\rho_{\text{data}}$ into a simple distribution (typically Gaussian $\mathcal{N}(0,I)$)
\item \textbf{Reverse process}: Learn the reverse dynamics to generate data from noise
\end{enumerate}

The unified forward SDE takes the form:
\[
dX_t = f(X_t, t)\,dt + g(t)\,dB_t, \qquad X_0\sim\rho_{\text{data}}
\]

Reverse SDE (Anderson, 1982):
\[
dX_t = \left[f(X_t,t) - g(t)^2\,\nabla\log\rho_t(X_t)\right]dt + g(t)\,d\bar{B}_t
\]

The corresponding probability flow ODE (noise-free version, with the same marginal distributions):
\[
\frac{dX_t}{dt} = f(X_t,t) - \frac{1}{2}g(t)^2\,\nabla\log\rho_t(X_t)
\]

\subsubsection{DDPM (Denoising Diffusion Probabilistic Models)}

Ho et al.\ (2020) proposed this discrete-time framework.

\textbf{Forward process}: $x_t = \sqrt{\alpha_t}\,x_{t-1} + \sqrt{1-\alpha_t}\,\epsilon_t$, with closed form $x_t = \sqrt{\bar\alpha_t}\,x_0 + \sqrt{1-\bar\alpha_t}\,\epsilon$.

\textbf{Training objective} (predicting noise):
\[
\mathcal{L}_{\text{DDPM}} = \mathbb{E}_{t,x_0,\epsilon}\left[\|\epsilon - \epsilon_\theta(\sqrt{\bar\alpha_t}\,x_0 + \sqrt{1-\bar\alpha_t}\,\epsilon,\; t)\|^2\right]
\]

\textbf{Relationship to the score}: $\nabla\log\rho_t(x_t) = -\frac{\epsilon_\theta(x_t, t)}{\sqrt{1-\bar\alpha_t}}$. DDPM is essentially learning the score.

\subsubsection{NCSN / SMLD}

Song \& Ermon (2019, 2020): Directly train $s_\theta(x,\sigma)\approx\nabla\log\rho_\sigma(x)$ and generate via annealed Langevin dynamics.
\[
\mathcal{L}_{\text{NCSN}} = \mathbb{E}_{\sigma,x,\tilde x}\left[\|s_\theta(\tilde x, \sigma) + \frac{\tilde x - x}{\sigma^2}\|^2\right]
\]

\subsubsection{VE-SDE (Variance Exploding)}

Song et al.\ (2021): $dX_t = \sqrt{\dot\sigma^2(t)}\,dB_t$ ($f=0$, pure diffusion). The variance ``explodes'' to $\sigma^2(T)\gg 1$.
The corresponding FP equation reduces to the heat equation $\partial_t\rho = \frac{1}{2}\dot\sigma^2\,\Delta\rho$.

\subsubsection{VP-SDE (Variance Preserving)}

Song et al.\ (2021): $dX_t = -\frac{\beta(t)}{2}X_t\,dt + \sqrt{\beta(t)}\,dB_t$ (OU process). The variance remains bounded $\to 1$.
The corresponding FP equation: $\partial_t\rho = \frac{\beta}{2}\nabla\cdot(\rho\,x) + \frac{\beta}{2}\Delta\rho$.

Probability flow ODE: $\frac{dx}{dt} = -\frac{\beta(t)}{2}[x + \nabla\log\rho_t(x)]$.

\subsubsection{Flow Matching}

Lipman et al.\ (2023), Liu et al.\ (2023): Instead of starting from an SDE, directly learn a deterministic velocity field that satisfies the continuity equation.

Conditional Flow Matching training objective:
\[
\mathcal{L}_{\text{FM}} = \mathbb{E}_{t,x_1,x}\left[\|v_\theta(x,t) - \frac{x_1 - x}{1-t}\|^2\right]
\]
where $x\sim\mathcal{N}(t\,x_1,(1-t)^2I)$. Generation: solve the ODE $\dot x = v_\theta(x,t)$.

Advantages: pure ODE (fast sampling), straight-line paths (easier to learn), and freely designable interpolation schemes.

\subsection{Unified comparison table}

\begin{center}
\renewcommand{\arraystretch}{1.5}
\begin{tabular}{@{}p{2.2cm}p{2.8cm}p{2.5cm}p{2.2cm}p{2.5cm}@{}}
\toprule
\textbf{Model} & \textbf{Forward process} & \textbf{What is learned} & \textbf{Generation} & \textbf{$f(x,t)$, $g(t)$} \\
\midrule
DDPM & Discrete noising & Noise $\epsilon$ & Stepwise denoising & $-\frac{\beta}{2}x$, $\sqrt\beta$ \\[3pt]
NCSN & Multi-scale noising & score $\nabla\log\rho$ & Annealed Langevin & $0$, $\sqrt{\dot\sigma^2}$ \\[3pt]
VE-SDE & Pure diffusion & score $\nabla\log\rho$ & reverse SDE/ODE & $0$, $\sqrt{\dot\sigma^2}$ \\[3pt]
VP-SDE & OU process & score $\nabla\log\rho$ & reverse SDE/ODE & $-\frac{\beta}{2}x$, $\sqrt\beta$ \\[3pt]
Flow Matching & Linear interpolation & velocity $v_t$ & ODE integration & N/A (no SDE) \\
\bottomrule
\end{tabular}
\end{center}

\subsection{A unified understanding from the Fokker-Planck/JKO perspective}

\begin{itemize}
\item \textbf{Score-based models}: The FP equation corresponding to the reverse SDE describes the density descending along the free energy $\mathcal{F}$. Each denoising step is essentially one JKO step.

\item \textbf{Flow Matching}: The velocity field $v_t$ satisfies the continuity equation. For optimal transport paths, $v_t$ is precisely the velocity along the $W_2$ geodesic---a direct application of the Benamou-Brenier formula.

\item \textbf{Common destination}: Regardless of the method, the generation process ``traces a path'' in probability space from $\mathcal{N}(0,I)$ to $\rho_{\text{data}}$. The only difference is whether the path is tortuous (diffusion) or a smooth geodesic (flow matching).
\end{itemize}

\section{Convex Analysis Quick Review}\label{app:convex}

This appendix provides a quick review of basic concepts from convex analysis, which appear repeatedly in the uniqueness and existence arguments for the JKO scheme.

\subsection{Convex functions}

\begin{definition}[Convex function]
A function $f:\R^d\to\R\cup\{+\infty\}$ is a \textbf{convex function} if for all $x,y\in\R^d$ and $t\in[0,1]$:
\[
f\bigl(tx + (1-t)y\bigr) \leq t\,f(x) + (1-t)\,f(y)
\]
\end{definition}

\textbf{Geometric Intuition: } The line segment between any two points on the graph lies above (or coincides with) the graph. In other words, the function ``curves upward.''

\textbf{Typical Examples: }
\begin{itemize}
    \item $f(x) = x^2$ (parabola), $f(x) = |x|$ (V-shape), $f(x) = e^x$
    \item In $\R^d$: $f(x)=\|x\|^2$, $f(x)=\max_i x_i$
    \item Non-convex examples: $f(x)=\sin x$, $f(x)=-x^2$
\end{itemize}

The key property of convex functions: \textbf{every local minimum is a global minimum}. This is the fundamental reason why convex optimization is much easier than general optimization.

\subsection{The Hessian matrix}

\begin{definition}[Hessian]
The \textbf{Hessian matrix} of a scalar function $f:\R^d\to\R$ is the $d\times d$ matrix of its second-order partial derivatives:
\[
\bigl[\nabla^2 f(x)\bigr]_{ij} = \frac{\partial^2 f}{\partial x_i\,\partial x_j}(x)
\]
\end{definition}

The Hessian describes the ``degree and direction of curvature'' of a function at a given point; it is the multidimensional generalization of the second derivative $f''(x)$ in one dimension.

\begin{itemize}
    \item One dimension: $f''(x)>0$ means the function curves upward at $x$ (convex)
    \item Multiple dimensions: the eigenvalues of $\nabla^2 f(x)$ describe the curvature in each direction
\end{itemize}

\subsection{Positive definiteness and positive semidefiniteness}

\begin{definition}[Positive definite / semidefinite]
For a symmetric matrix $A\in\R^{d\times d}$:
\begin{itemize}
    \item \textbf{Positive definite} (denoted $A\succ 0$): $v^\top A v > 0$ for all $v\neq 0$. Equivalently, all eigenvalues of $A$ are $>0$.
    \item \textbf{Positive semidefinite} (denoted $A\succeq 0$): $v^\top A v \geq 0$ for all $v$. Equivalently, all eigenvalues are $\geq 0$.
\end{itemize}
\end{definition}

\textbf{Connection to convexity:} For a twice-differentiable function $f$:
\[
f\text{ is convex} \iff \nabla^2 f(x) \succeq 0 \text{ for all } x
\]
A positive definite Hessian ($\nabla^2 f \succ 0$) means the function curves strictly upward in every direction.

\textbf{Intuition: } Think of $A\succ 0$ as ``uphill in every direction.'' For the quadratic $f(x)=\frac{1}{2}x^\top Ax$,
$A\succ 0$ means exactly that $f$ is ``bowl-shaped'' (with a unique minimum), and the eigenvalues of $A$ determine the steepness in each direction.

\subsection{Strong convexity and $\lambda$-convexity}

\begin{definition}[Strongly convex]
A function $f$ is \textbf{$\lambda$-strongly convex} ($\lambda>0$) if $f(x) - \frac{\lambda}{2}\|x\|^2$ is still convex. Equivalently:
\[
f\bigl(tx+(1-t)y\bigr) \leq t\,f(x) + (1-t)\,f(y) - \frac{\lambda}{2}\,t(1-t)\|x-y\|^2
\]
For twice-differentiable $f$, this is equivalent to $\nabla^2 f(x)\succeq\lambda I$ for all $x$.
\end{definition}

\textbf{Intuition: } Strong convexity means the function is not only convex but has a ``lower bound on its curvature''---it curves at least as much as $\frac{\lambda}{2}\|x\|^2$.
\begin{itemize}
    \item The larger $\lambda$ is, the more the function ``bends,'' and the faster convergence is near the minimum
    \item Convexity = the degenerate case $\lambda=0$
    \item Strong convexity guarantees that the minimum is \textbf{unique} and that gradient descent converges at a \textbf{linear rate}
\end{itemize}

\textbf{Appearance in this article:} The condition $\nabla^2 V\succeq\lambda I$ ($\lambda$-convexity) on the potential $V$ in the JKO scheme guarantees the uniqueness of the solution at each JKO step, as well as convergence of the discrete solutions to the continuous gradient flow.

\section{Supplementary Concepts}\label{app:supplement}

This appendix explains several terms that appear in the main text, providing a quick reference for readers who may be unfamiliar with these concepts.

\subsection{Langevin dynamics}

\textbf{Langevin dynamics} is a stochastic differential equation originally proposed by Paul Langevin in 1908 (Langevin, 1908) to describe the Brownian motion of particles in a fluid. In modern probability, statistical mechanics, and machine learning, the most common form is the \textbf{overdamped Langevin dynamics}:
\[
dX_t = -\nabla V(X_t)\,dt + \sqrt{2\varepsilon}\,dB_t
\]
where $V$ is the potential energy function, $\varepsilon>0$ is the temperature/noise strength, and $B_t$ is standard Brownian motion. In the main text we often set $\varepsilon=1$ for simplicity.

\textbf{Physical intuition: } A particle moves in the potential energy landscape $V$, subject to two effects:
\begin{itemize}
    \item $-\nabla V$: deterministic drift, pushing the particle toward regions of low potential energy.
    \item $\sqrt{2\varepsilon}\,dB_t$: random thermal noise, preventing all particles from collapsing to a single minimizer of $V$.
\end{itemize}

\textbf{Stationary distribution.} The probability density $\rho_t$ of $X_t$ satisfies the Fokker--Planck equation
\[
\partial_t\rho_t
=
\nabla\cdot(\rho_t\nabla V)
+
\varepsilon\Delta\rho_t.
\]
At stationarity, the deterministic force and the entropic force balance:
\[
-\nabla V-\varepsilon\nabla\log\rho_\infty=0.
\]
Therefore
\[
\nabla\log\rho_\infty=-\frac{1}{\varepsilon}\nabla V,
\qquad
\rho_\infty(x)\propto e^{-V(x)/\varepsilon}.
\]
Thus Langevin dynamics samples from the Gibbs/Boltzmann distribution with temperature $\varepsilon$.

\textbf{Free-energy interpretation.} The same equation is the Wasserstein gradient flow of
\[
\mathcal{F}_\varepsilon(\rho)
=
\int V\rho\,dx
+
\varepsilon\int\rho\log\rho\,dx.
\]
The first term pulls mass toward low potential energy. The second term is negative entropy; it creates the diffusion force $-\varepsilon\nabla\log\rho$. This is the precise sense in which Langevin dynamics minimizes free energy rather than potential energy alone.

\begin{keypoint}
Gradient descent on $V$ alone,
\[
\dot X_t=-\nabla V(X_t),
\]
finds low-energy points. Langevin dynamics,
\[
dX_t=-\nabla V(X_t)\,dt+\sqrt{2\varepsilon}\,dB_t,
\]
samples from the full distribution $\rho_\infty\propto e^{-V/\varepsilon}$.
The noise is not a nuisance: it is what realizes the entropy term in the free energy.
\end{keypoint}

\textbf{Discretization: the unadjusted Langevin algorithm.} Euler--Maruyama discretization gives
\[
X_{k+1}
=
X_k-\eta\nabla V(X_k)+\sqrt{2\varepsilon\eta}\,Z_k,
\qquad
Z_k\sim\mathcal{N}(0,I).
\]
This is called the \textbf{unadjusted Langevin algorithm} (ULA). If one adds a Metropolis--Hastings accept/reject correction, the method becomes MALA (Metropolis-adjusted Langevin algorithm), which removes discretization bias under suitable assumptions.

\textbf{In machine learning.}
\begin{itemize}
    \item \textbf{Energy-based models}: If a model defines $p_\theta(x)\propto e^{-V_\theta(x)}$, Langevin dynamics can sample from it using only $\nabla V_\theta$.
    \item \textbf{SGLD}: Stochastic Gradient Langevin Dynamics replaces $\nabla V$ by a minibatch stochastic gradient. Adding calibrated Gaussian noise turns optimization-like SGD into approximate posterior sampling.
    \item \textbf{Score-based models}: Annealed Langevin dynamics uses learned scores $\nabla\log p_\sigma(x)$ at multiple noise levels. Since $\nabla\log p(x)=-\nabla V(x)$ when $p\propto e^{-V}$, score-based sampling can be viewed as Langevin dynamics written in score language.
    \item \textbf{Diffusion models}: Reverse SDE sampling is a time-inhomogeneous, score-driven Langevin-type dynamics. Probability flow ODE sampling removes the explicit Brownian noise but keeps the same density evolution by absorbing the diffusion effect into a deterministic score velocity.
\end{itemize}

\textbf{Relation to this article.} Langevin dynamics is the particle-level stochastic process; the Fokker--Planck equation is its density-level PDE; the free energy is its Lyapunov function; and the JKO scheme is the implicit Euler discretization of this free-energy gradient flow in Wasserstein space.

\subsection{Kolmogorov forward equation}

The \textbf{Kolmogorov forward equation} (Kolmogorov, 1931) is another name for the Fokker-Planck equation. The two are identical.
For the SDE $dX_t = f(X_t,t)\,dt + g(t)\,dB_t$ (where the diffusion coefficient $g$ depends only on time, not on $X_t$---all models considered in this article are of this type):
\[
\partial_t\rho = -\nabla\cdot(f\rho) + \frac{g^2}{2}\Delta\rho
\]
This describes the time evolution of the probability density $\rho_t(x)$ of $X_t$. (For the more general case of state-dependent diffusion $G(X_t)$, the equation becomes $\partial_t\rho = -\nabla\cdot(f\rho) + \frac{1}{2}\sum_{ij}\partial_i\partial_j[(GG^\top)_{ij}\rho]$.)

\textbf{Why two names?} Historical reasons: Kolmogorov derived this equation in 1931 from a purely mathematical perspective, while Fokker and Planck independently discovered it earlier (1913--1917) from a physics perspective. Physicists call it the Fokker-Planck equation; probabilists call it the Kolmogorov forward equation.

There is also a corresponding \textbf{Kolmogorov backward equation}, which describes the probability evolution ``looking back from the endpoint to the starting point,''
corresponding to the PDE satisfied by expected values of the stochastic process.

\subsection{Ornstein-Uhlenbeck (OU) process}

The \textbf{OU process} (Uhlenbeck \& Ornstein, 1930) is one of the most important linear SDEs:
\[
dX_t = -\theta\,X_t\,dt + \sigma\,dB_t, \qquad \theta>0
\]

\textbf{Intuition: } Spring + noise. The particle is pulled back to the origin by a spring ($-\theta X_t$ is the restoring force), while simultaneously being subject to random perturbations.

\textbf{Key properties:}
\begin{itemize}
    \item The stationary distribution is Gaussian: $X_\infty\sim\mathcal{N}\bigl(0, \frac{\sigma^2}{2\theta}\bigr)$
    \item From any initial condition, $X_t$ converges to the stationary distribution at an exponential rate
    \item It is a ``compromise'' between Brownian motion ($\theta=0$) and pure deterministic decay ($\sigma=0$)
\end{itemize}

\textbf{In this article:} The forward process of VP-SDE, $dX_t = -\frac{\beta(t)}{2}X_t\,dt + \sqrt{\beta(t)}\,dB_t$, is an OU process with time-varying coefficients.
It gradually ``pulls'' the data distribution toward the standard Gaussian $\mathcal{N}(0,I)$.

\subsection{Otto Calculus}

\textbf{Otto calculus} is a formal computational framework introduced by Felix Otto in 2001, whose core idea is:

\begin{center}
\fbox{\parbox{0.8\textwidth}{\centering
Treat the space of probability measures $\Prob_2(\R^d)$ as an infinite-dimensional Riemannian manifold and do calculus on it.
}}
\end{center}

Specifically, Otto calculus consists of the following ``conventions'':
\begin{enumerate}
    \item \textbf{Tangent space}: The ``tangent vectors'' at a distribution $\rho$ are gradient fields $\nabla\psi$ (velocity fields satisfying the continuity equation $\dot\rho = -\nabla\cdot(\rho\nabla\psi)$)
    \item \textbf{Riemannian metric}: The inner product of two tangent vectors $\nabla\psi_1, \nabla\psi_2$ is defined as $\langle\nabla\psi_1,\nabla\psi_2\rangle_\rho = \int\nabla\psi_1\cdot\nabla\psi_2\,\rho\,dx$
    \item \textbf{Gradient}: The Wasserstein gradient of a functional $G(\rho)$ is the tangent vector (velocity field) $\mathrm{grad}_W G = \nabla\frac{\delta G}{\delta\rho}$. The induced direction of density evolution is $\partial_t\rho = -\nabla\cdot\bigl(\rho\,\mathrm{grad}_W G\bigr) = -\nabla\cdot\bigl(\rho\nabla\frac{\delta G}{\delta\rho}\bigr)$
\end{enumerate}

\textbf{Why is it called ``calculus''?} Because with the conventions above, many complex infinite-dimensional computations reduce to ``plugging into formulas''---just as ordinary calculus frees you from going back to the limit definition every time.

\textbf{Rigorous vs.\ formal:} Otto calculus is \emph{formal}---$\Prob_2$ is not strictly a smooth manifold (it is infinite-dimensional and has a boundary).
However, its computational results can be verified by rigorous theory (the Ambrosio-Gigli-Savar\'{e} framework). In practice, computing directly with Otto's formalism is usually more efficient.

\subsection{Geodesics}

\begin{definition}[Geodesic --- informal]
On a metric space or Riemannian manifold, a \textbf{geodesic} connecting two points $x$ and $y$ is the ``shortest path''---a curve whose length equals the distance between $x$ and $y$.
\end{definition}

\textbf{Examples in various spaces: }
\begin{itemize}
    \item Euclidean space $\R^d$: geodesics are straight line segments
    \item Sphere $S^2$: geodesics are great circle arcs (the shortest flight routes)
    \item Wasserstein space $\Prob_2(\R^d)$: geodesics are McCann displacement interpolations $\rho_t = \bigl((1-t)\id + t\,T\bigr)_\#\mu$,
    where $T$ is the optimal transport map from $\mu$ to $\nu$
\end{itemize}

\textbf{Connection to Benamou-Brenier:} $W_2(\mu,\nu)$ is precisely the length of the geodesic connecting $\mu$ and $\nu$,
and the Benamou-Brenier formula $W_2^2 = \inf\int_0^1\int|v_t|^2\rho_t\,dx\,dt$ seeks the shortest path among all paths satisfying the continuity equation.

\subsection{Riesz representation theorem}

The \textbf{Riesz representation theorem} has several versions; the one used in this article is the Hilbert space version:

\begin{theorem}[Riesz representation --- Hilbert space version]
Let $H$ be a Hilbert space and $\varphi: H\to\R$ a continuous linear functional on $H$. Then there exists a unique $f\in H$ such that:
\[
\varphi(h) = \langle f, h\rangle_H \quad \text{for all } h\in H
\]
\end{theorem}

\textbf{Intuition: } In a space equipped with an inner product, every ``linear measurement device'' (continuous linear functional) can be realized as ``taking the inner product with some fixed vector.''

\textbf{Role in this article:} The Wasserstein gradient $\grad_W\mathcal{F}$ is precisely the ``Riesz representative'' of the functional derivative $\frac{\delta\mathcal{F}}{\delta\rho}$---it
converts the functional derivative (a cotangent vector) under the $L^2$ inner product into a gradient (a tangent vector) under the Wasserstein metric.
This is exactly the infinite-dimensional manifestation of the principle from Appendix~\ref{app:diffgeom} that going from ``differential $\to$ gradient'' requires a metric.

\subsection{Lyapunov functions}

\begin{definition}[Lyapunov function --- informal]
For a dynamical system $\dot x = F(x)$, a \textbf{Lyapunov function} $V(x)$ satisfies:
\begin{enumerate}
    \item $V(x)\geq 0$, and $V(x^*)=0$ only at the equilibrium point $x^*$
    \item It is monotonically decreasing along trajectories of the system: $\frac{d}{dt}V(x(t))\leq 0$
\end{enumerate}
\end{definition}

\textbf{Intuition: } A Lyapunov function is like ``energy''---if you can find a quantity that only decreases during the system's evolution,
then the system must evolve toward the point where this quantity is zero (the equilibrium). It is the standard tool for proving stability.

\textbf{In this article:} The free energy $\mathcal{F}(\rho)=\KL(\rho\|\pi)$ is the Lyapunov function of the Fokker-Planck dynamical system---it
decreases monotonically along solutions of the FP equation ($\frac{d}{dt}\mathcal{F}(\rho_t)\leq 0$) and attains its minimum value of zero at the stationary distribution $\rho=\pi$.

\subsection{KKT conditions}

The \textbf{KKT conditions} (Karush-Kuhn-Tucker conditions) are necessary optimality conditions for constrained optimization problems, generalizing the ``gradient equals zero'' condition from unconstrained optimization.

Consider the problem: $\min f(x)$ subject to $g(x)=0$. The Lagrangian is $L(x,\lambda) = f(x) + \lambda\,g(x)$.

The KKT conditions require:
\[
\nabla_x L = \nabla f(x) + \lambda\,\nabla g(x) = 0, \qquad g(x) = 0
\]

\textbf{Intuition: } On the constraint surface, the gradient of the objective function must be parallel to the normal direction of the constraint surface (otherwise one could continue descending along the constraint surface).
The multiplier $\lambda$ measures ``how tight'' the constraint is.

\textbf{In this article:} The derivation of the optimality conditions for the JKO scheme (Section~6) is essentially the infinite-dimensional version of the KKT conditions---the
functional derivative equals zero, together with the Lagrange multiplier corresponding to the mass conservation constraint $\int\rho=1$.

\end{document}